\documentclass{article}
\usepackage{ijcai18}
 
\usepackage{times}
\usepackage{xcolor}
\usepackage{soul}
\usepackage[utf8]{inputenc}
\usepackage[small]{caption}
\usepackage{amssymb}
\usepackage{graphics}
\usepackage{enumitem}

\usepackage{xspace}

\usepackage{stmaryrd}
\usepackage{amsmath} 
\usepackage{amsthm}
\usepackage{amssymb}
\usepackage{latexsym}
\usepackage{graphicx} 
\usepackage{multirow} 
\usepackage{tikz} 
\usetikzlibrary{positioning} 
\usetikzlibrary{shapes, arrows} 
\usepackage{float} 
\tikzset{
	events/.style={circle, minimum size=1cm, draw, align=center},
}
\usepackage{cleveref}
\usepackage[english]{babel}
\usepackage{enumitem}

\usepackage{xspace}

\usepackage{tabularx}

\usepackage{stmaryrd}
\usepackage{amsmath} 
\usepackage{amsthm}
\usepackage{amssymb}
\usepackage{graphicx} 
\usepackage{multirow}
\usepackage{tikz}
\usetikzlibrary{positioning}
\usetikzlibrary{shapes, arrows}
\usepackage{float}
\tikzset{
	events/.style={circle, minimum size=1cm, draw, align=center},
}

\usepackage[english]{babel}

\definecolor {processblue}{cmyk}{0.96,0.96,0,0}
\definecolor {processred}{cmyk}{0,0.96,0.96,0}
\usepackage{framed}
\usepackage{multicol}
\usepackage{booktabs}

\usepackage{thmtools}
\usepackage{thm-restate}

\usepackage{cleveref}

\newcommand{\nin}{\ensuremath{o_{\rho}}}

\newcommand{\equals}{\ensuremath{\approx}\xspace}
\newcommand{\nequals}{\ensuremath{\not \approx}\xspace}
\newcommand{\SUtrig}{\ensuremath{\textsf{Su}}\xspace}
\newcommand{\PRtrig}{\ensuremath{\textsf{Pr}}\xspace}

\newcommand{\SUt}{\ensuremath{\textsf{Su}_t} \xspace}
\newcommand{\PRt}{\ensuremath{\textsf{Pr}_t} \xspace}
\newcommand{\NOMt}{\ensuremath{\textsf{Nom}_t} \xspace}
\newcommand{\QueryLHS}[1]{\ensuremath{\Gamma_{#1}}\xspace} 
\newcommand{\QueryRHS}[1]{\ensuremath{\Delta_{#1}}\xspace} 
\newcommand{\GroundS}[1]{\ensuremath{N_{#1}}\xspace}
\newcommand{\InputQueryLHS}{\ensuremath{\Gamma_Q}\xspace}
\newcommand{\InputQueryRHS}{\ensuremath{\Delta_Q}\xspace}
\newcommand{\Required}[1]{\ensuremath{\Gamma_{#1}}\xspace} 
\newcommand{\Forbidden}[1]{\ensuremath{\Delta_{#1}}\xspace} 
\newcommand{\Grounding}[1]{\ensuremath{\sigma_{t}}\xspace}

\newcommand{\ruleword}[1]{\textsf{#1}}
\newcommand{\Rsystem}[1]{\ensuremath{R_{#1}}}
\newcommand{\RSystem}[2]{\ensuremath{R^{#2}_{#1}}}
\newcommand{\RequiredN}{\ensuremath{\Gamma_{o}}\xspace} 
\newcommand{\ForbiddenN}{\ensuremath{\Delta_{o}}\xspace} 

\newcommand{\HU}{\textsf{HU}}
\newcommand{\body}[1]{\textsf{Body}(#1)}
\newcommand{\rootc}{\ensuremath{v_r}\xspace} 
\newcommand{\rootSUtrig}{\ensuremath{\textsf{Su}^{r}}\xspace}
\newcommand{\rootPRtrig}{\ensuremath{\textsf{Pr}^{r}}\xspace}

\newcommand{\RuleSymbol}{\ensuremath{\Rightarrow}}
\newcommand{\Rule}[3]{\ensuremath{#1 \RuleSymbol\hspace{-0.1em} #3 \in #2}}
\newcommand{\RwRelSymbol}[1]{\ensuremath{\mathrel{\rightarrow_{#1}}}}
\newcommand{\RwRel}[3]{\ensuremath{#1 \RwRelSymbol{#2}\hspace{-0.1em} #3}}
\newcommand{\RwRelRefTransSymbol}[1]{\ensuremath{\mathrel{\overset{*}{\rightarrow}_{#1}}}}
\newcommand{\RwRelRefTrans}[3]{\ensuremath{#1 \RwRelRefTransSymbol{#2}\hspace{-0.1em} #3}}
\newcommand{\CongruenceSymbol}[1]{\mathrel{\overset{*}{\leftrightarrow}_{#1}}}
\newcommand{\Congruence}[3]{\ensuremath{#1 \CongruenceSymbol{#2}\hspace{-0.1em} #3}}

\newcommand{\II}{\ensuremath{\mathcal{I}} \xspace}
\newcommand{\JJ}{\ensuremath{\mathcal{J}} \xspace}
\newcommand{\model}[2]{\ensuremath{{#1}^{*}_{#2}} \xspace}
\newcommand{\Rmodel}[1]{\ensuremath{R_{#1}^{*}} \xspace}
\newcommand{\RModel}[2]{\ensuremath{(R^{{#2}}_{{#1}})^{*}} \xspace}
\newcommand{\nmodels}{\ensuremath{\not \models}} 

\newcommand{\ContextStructure}{\ensuremath{\mathcal{D}}\xspace}
\newcommand{\Contexts}{\ensuremath{\mathcal{V}}\xspace}
\newcommand{\Edges}{\ensuremath{\mathcal{E}}\xspace}
\newcommand{\core}[1]{\ensuremath{\textsf{core}_{#1}}\xspace}
\newcommand{\SC}[1]{\ensuremath{\mathcal{S}_{#1}}\xspace}
\newcommand{\Strategy}{\textsf{strat} \xspace }

\newcommand{\Onto}{\ensuremath{\mathcal{O}}\xspace}

\newcommand{\Atom}{\ensuremath{A}\xspace}

\newcommand{\Functions}{\ensuremath{\Sigma_{f}}\xspace}
\newcommand{\FunctionsOnto}{\ensuremath{\Sigma_{f}^{\mathcal{O}}}\xspace}

\newcommand{\Unpreds}{\ensuremath{\Sigma_{B}}\xspace}
\newcommand{\UnpredsOnto}{\ensuremath{\Sigma_{B}^{\Onto}}\xspace}

\newcommand{\Binpreds}{\ensuremath{\Sigma_{S}}\xspace}
\newcommand{\BinpredsOnto}{\ensuremath{\Sigma_{S}^{\Onto}}\xspace}

\newcommand{\AllIndiv}{\ensuremath{\Sigma_o^{\mathcal{O}}} \xspace}
\newcommand{\TrueIndiv}{\ensuremath{\Sigma_o^{\mathcal{O}}} \xspace}

\newcommand{\ostar}{\star\kern-0.42em{\circ}}

\newtheorem{theorem}{Theorem}

\newtheorem{lemma}{Lemma}

\theoremstyle{definition}
\newtheorem{definition}{Definition}

\makeatletter
\providecommand{\bigsqcap}{%
  \mathop{%
    \mathpalette\@updown\bigsqcup
  }%
}

\newcommand*{\@updown}[2]{%
  \rotatebox[origin=c]{180}{$\m@th#1#2$}%
}
\newcommand{\sbin}{\mbox{\,\,}\hat{\in}\mbox{\,} }
\newcommand{\nsbin}{\mbox{\,\,}\hat{\not \in}\mbox{\,} }

\newcommand{\cmmnt}[1]{}
\newcommand{\napprox}{\not \approx}

\newcommand{\ab}{\textsf{a}}
\newcommand{\pred}{\textsf{p}}
\newcommand{\TOP}{\textsf{true}}

\newcommand{\SHIQ}{\ensuremath{\mathcal{SHIQ}} \xspace}

\newcommand{\ELHO}{\ensuremath{\mathcal{ELHO}} \xspace}
\newcommand{\logic}{\ensuremath{\mathcal{ALCHOIQ}} \xspace }

\newcommand{\clabel}[1]{\refstepcounter{clauses}\label{cl:#1}(\theclauses)}
\newcommand{\rlabel}[1]{\refstepcounter{rows}\label{row:#1}}

\newtheorem{corollary}{Corollary}

\crefname{equation}{}{}
\Crefname{equation}{}{}

\crefformat{condition}{condition~#2#1#3}
\crefrangeformat{condition}{conditions~#3#1#4 through~#5#2#6}
\crefmultiformat{condition}{conditions~#2#1#3}{ and~#2#1#3}{, #2#1#3}{ and~#2#1#3}

\newlist{enumerateConditions}{enumerate}{2}
\setlist[enumerateConditions]{label={\arabic*.},ref={(\arabic*)}}

\crefalias{enumerateConditionsi}{condition}
\crefalias{enumerateConditionsii}{condition}

\crefname{rows}{row}{rows}

\crefname{clauses}{clause}{clauses}
\crefname{enumerateConditions}{condition}{conditions}

\newcommand{\citeA}[1]{\citeauthor{#1} [\citeyear{#1}]}
 
 \newcommand{\newmarkedtheorem}[1]{%
  \newenvironment{#1}
    {\pushQED{\qed}\csname inner@#1\endcsname%
     \renewcommand{\qedsymbol}{$\lhd$}\normalfont}
    {\popQED\csname endinner@#1\endcsname}%
  \newtheorem{inner@#1}%
}
\newmarkedtheorem{example}[theorem]{Example}

\title{Consequence-based Reasoning for Description Logics with Disjunction, Inverse Roles, Number Restrictions, and Nominals}

\author{
David Tena Cucala, 
Bernardo Cuenca Grau, 
Ian Horrocks, 
\\ 
Department of Computer Science, University of Oxford  \\
first.last1[.last2]@cs.ox.ac.uk
}

\begin{document}

\maketitle

\begin{abstract}
  We present a consequence-based calculus for concept subsumption and classification in
  the description logic  $\logic$, which extends $\mathcal{ALC}$ 
  with role hierarchies, inverse roles, number restrictions, and nominals. By using standard
  transformations, our calculus extends to $\mathcal{SROIQ}$, which covers
  all of OWL 2 DL except for datatypes.
  A key feature of our calculus is its pay-as-you-go behaviour: unlike existing algorithms, our calculus is worst-case optimal
  for all the well-known proper fragments of $\logic$, albeit not for the full logic.
\end{abstract}

\section{Introduction}

Description logics (DLs) \cite{BCMNP03} are a family of knowledge representation formalisms 
which are widely used in applications.
Although the basic DL reasoning problems, such as concept satisfiability and subsumption,  
are of high worst-case complexity for expressive DLs,  different calculi
have been developed and implemented in practical systems.
 
Tableau and hyper-tableau calculi \cite{BaaderSattler-StudiaLogica,MoSH09a} are  a prominent reasoning technique
underpinning  many  DL reasoners \cite{DBLP:journals/ws/SirinPGKK07,fact++,DBLP:journals/semweb/HaarslevHMW12,ghmsw14HermiT,StLG14a}.
To check whether a concept subsumption relationship holds, 
(hyper-)tableau calculi attempt to construct a finite representation 
of an ontology model  disproving the given subsumption. The constructed models can, however, be large---a common source of performance issues;
this problem is exacerbated in classification tasks due to the large number of subsumptions to be tested. 

Another major category of DL reasoning calculi comprises methods based on first-order logic
 resolution \cite{BachmairGanzinger:HandbookAR:resolution:2001}. A common approach to
ensure both termination and worst-case optimal running time is to 
parametrise resolution to ensure that the calculus only derives a bounded number of clauses 
\cite{DeNivelleSchmidtHustadt00,DBLP:journals/jar/HustadtS02,DBLP:conf/birthday/SchmidtH13,hms08deciding,ganzinger99superposition,KaMo06a,HuMS04}. 
This technique has been implemented for instance, 
in the KAON2 reasoner for $\SHIQ$. Resolution can also be used to 
simulate model-building (hyper)tableau techniques \cite{HustadtSchmidt00a}, including blocking 
methods which ensure termination \cite{GeorgievaHustadtSchmidt03a}. 

 \emph{Consequence-based} (CB) calculi have emerged as a promising approach to DL reasoning combining 
  features of (hyper)tableau and resolution
 \cite{BaBL05,Kaza09a,KaKS12,DBLP:conf/kr/BateMGSH16}. On the one hand, similarly to resolution, they 
 derive formulae entailed by the ontology (thus avoiding the explicit construction of large models), and they 
 are typically worst-case optimal. 
 On the other hand, clauses are organised into \emph{contexts} arranged  as
 a graph structure reminiscent to that used for model construction in (hyper)tableau; this prevents
 CB calculi from drawing many 
 unnecessary inferences and yields a nice goal-oriented behaviour. Furthermore, in contrast to both resolution and (hyper)tableau,
 CB calculi  can verify a large number of subsumptions in a 
 single execution, allowing for one-pass  classification. Finally, CB calculi are very practical
 and systems based on them have shown outstanding performance.

CB calculi were first proposed for the $\mathcal{EL}$ family of 
DLs \cite{BaBL05,KaKS12},  and later extended to more expressive logics 
like Horn-$\mathcal{SHIQ}$ \cite{Kaza09a}, Horn-$\mathcal{SROIQ}$ \cite{OrRS10}, 
and $\mathcal{ALCH}$ \cite{SiKH11a}. A unifying framework for CB reasoning was 
developed in \cite{SiMH14a} for $\mathcal{ALCHI}$, introducing the notion of \emph{contexts} as a 
mechanism for constraining resolution inferences and making them goal-directed. 
The framework has been  extended to the DLs $\mathcal{ALCHIQ}$, which supports
number restrictions and inverse roles  \cite{DBLP:conf/kr/BateMGSH16}; $\mathcal{ALCHIO}$, 
which supports inverse roles and nominals \cite{DBLP:conf/dlog/CucalaGH17}, and
$\mathcal{ALCHOQ}$, supporting nominals and number restrictions \cite{DBLP:conf/dlog/KarahroodiH17}.
  
To the best of our knowledge, however, no CB calculus can handle 
DLs supporting simultaneously all Boolean connectives, inverse roles, number restrictions, and
nominals.  Such DLs, which underpin the standard ontology languages, 
pose significant challenges for consequence-based reasoning. Indeed, 
DLs lacking inverse roles, number restrictions, or
nominals enjoy a variant of the
\emph{forest model property}, which is exploited by reasoning algorithms.
However, no such property holds
when a DL simultaneously supports all the aforementioned features;
for non-Horn DLs,  this
results in a complexity jump from  \textsf{ExpTime} to \textsf{NExpTime},
and complicates the design of reasoning calculi \cite{HoSa05a}.

In this paper, we present the first consequence-based calculus for the DL 
$\logic$, which supports
 all Boolean connectives, role hierarchies, inverse roles, number restrictions, and nominals. By using well-known
transformations, our calculus extends to  $\mathcal{SROIQ}$, which covers
 OWL 2 DL except for datatypes \cite{HoKS06a}.
Following  \citeA{DBLP:conf/kr/BateMGSH16}, we encode consequences derived by the calculus as 
first-order clauses of a specific form and handle equality reasoning using a variant of
ordered paramodulation. To account for nominals, we allow for 
ground atoms in derived clauses and 
group consequences  about named individuals into a single
\emph{root context}. 
We have carefully crafted the rules of our calculus so that it exhibits
worst-case optimal performance for the known proper fragments of $\logic$.
In particular, our calculus works in deterministic exponential time for 
all of  $\mathcal{ALCHOI}$, $\mathcal{ALCHOQ}$, $\mathcal{ALCHIQ}$ and 
Horn-$\logic$. Furthermore, it works in polynomial time
for the lightweight DL $\mathcal{ELHO}$. Our calculus is, however,
not worst-case optimal for the full logic $\logic$, exhibiting similar worst-case running time as other well-known calculi for expressive DLs \cite{MoSH09a,KaMo06a}. 
The source of additional complexity is very localised and  
only manifests when disjunction, nominals, 
number restrictions and inverse roles interact simultaneously---a rare situation in practice. 
Although our results are theoretical, we believe that our calculus can be
seamlessly implemented
as an extension of the $\mathcal{SRIQ}$ reasoner Sequoia \cite{DBLP:conf/kr/BateMGSH16}.

\section{Preliminaries}

\noindent
\textbf{Many-sorted clausal equational logic. } We use
standard terminology for many-sorted first-order logic 
with equality ($\approx$) as the only predicate. 
This is w.lo.g.\  since predicates other than equality can be encoded by means of an entailment-preserving
transformation  \cite{NieuwenhuisRubio:HandbookAR:paramodulation:2001}. 

A many-sorted signature $\Sigma$ is a pair $\langle \Sigma^{S}, \Sigma^{F} \rangle$ with 
$\Sigma^{S}$ a non-empty set of sorts, and $\Sigma^{F}$ 
a countable set of function symbols. Each $f \in \Sigma^F$ is associated to a symbol type, which is an $n+1$-tuple $\langle \textsf{s}_1, \dots, \textsf{s}_{n+1} \rangle$
with each $\textsf{s}_i \in \Sigma^S$. The sort of $f$ is  
$\textsf{s}_{n+1}$ and its arity is $n$; if $n=0$, then $f$ is a constant.   
For each $\textsf{s} \in \Sigma^S$, let $X_{\textsf{s}}$ be a disjoint, countable 
set of variables. 
The set of terms
is the smallest set containing all variables in $X_{\textsf{s}}$ as terms of sort $\textsf{s}$, and  all expressions $f(t_1,\dots,t_n)$ as terms of sort $\textsf{s}_{n+1}$, where each $t_i$ is a term of sort $\textsf{s}_i$. 
A term is ground if it has no variables. We use the standard definition of a position $p$ of 
a term as an integer string identifying an occurrence of a subterm $t|_p$ and the 
standard notion of a substitution, represented as an expression $\{x_1 \mapsto t_1, \dots, x_n \mapsto t_n\}$ containing all non-identity mappings. 
We represent by $t[r]_p$ the result of replacing the subterm in position $p$ of $t$ by a term $r$ of the same sort. 

An equality is an expression  $s \approx t$ with $s$ and $t$  
terms of the same sort. An inequality is 
of the form $\neg (s \approx t)$, and is written
as $s \nequals t$. 
A literal is an equality or an inequality. A clause 
is a sentence  $\forall \vec{x}\,\, ( \Gamma \rightarrow \Delta)$, with the body
 $\Gamma$  a conjunction of equalities, the head $\Delta$ a disjunction of literals, and 
$\vec{x}$ the variables occurring in the clause. 
 The quantifier is often omitted,  conjunctions and disjunctions are treated as sets,
 and the empty conjunction (disjunction) is written as $\top$ ($\bot$). 
 
Let $\mbox{HU}^{\textsf{s}}$ be the set of ground 
terms of sort $\textsf{s}$. A Herbrand equality interpretation
 $\mathcal{I}$ is a set of ground equalities satisfying the usual properties of equality:
 (i) reflexivity, (ii) symmetry,  (iii) transitivity, and (iv) $t[s]_p \equals t \in \mathcal{I}$
 whenever $t|_p \equals s \in \mathcal{I}$, for each
 ground term $t$, position $p$, and ground term $s$.
  For any ground conjunction, ground disjunction, (not necessarily ground) clause, or a set thereof; an
  interpretation $\mathcal{I}$ satisfies it according to the usual criteria, with the difference that
  each quantified variable of sort $\textsf{s}$ ranges only over $\mbox{HU}^{\textsf{s}}$. We write $\mathcal{I} \models \alpha$
  if $\mathcal{I}$ satisfies $\alpha$, and say that $\mathcal{I}$ is a $\emph{model}$ of $\alpha$. Entailment is defined as usual. 

\noindent
\textbf{Orders} A strict (non-total) order $\succ$ on a non-empty set $S$ is a binary, irreflexive, and transitive
 relation between elements of $S$. A strict order induces a non-strict 
order $\succeq$ by taking the reflexive closure of $\succ$. A total order $>$ is a strict 
order such that for any $a,b \in S$, either $a > b$ or $b >a$. 
For any of these orders $\circ$, we write $a \circ N$, if $a \circ b$ for each $b \in N$, where $a\in S$, $N \subseteq S$.
The multiset extension of $\circ$ is defined as follows: 
for $S$-multisets $M$ and $N$, we have $M \circ N$ iff for each $a \in N \backslash M$, 
there is $b \in M \backslash N$ such that $b \circ a$, where $\backslash$ is the multiset difference operator. 
Order $\circ$ induces an order between literals by treating each equality $s \equals t$ 
as the multiset $\{s,t\}$, and each inequality $s \nequals t$ as the multiset $\{s,s,t,t\}$. 

\noindent
\textbf{Description Logics} 
DL expressions can be transformed into clauses of many-sorted equational logic in a way that preserves 
satisfiability and entailment \cite{DBLP:journals/tocl/SchmidtH07}. Following standard practice, 
we use a two-sorted signature with a sort $\ab$ representing standard FOL terms, and a sort $\pred$ for standard FOL atoms. 
The set of function symbols is the disjoint 
union of a set $\Unpreds$ of atomic concepts $B_i$ of type $\langle \ab, \pred \rangle$, 
a set $\Binpreds$ of atomic roles $S_i$ 
of type $\langle \ab, \ab, \pred \rangle$, a set $\Functions$ of functions
$f_j$ of type $\langle \ab, \ab \rangle$, and 
a set $\Sigma_o$ of named individuals $o$ of type $\langle \ab \rangle$. 
A term of the form $f_j(t)$ is an $f_j$ successor of $t$, and $t$ is its predecessor.
Our signature uses variables $\{x\} \cup \{z_i\}_{i\geq 1}$ of sort $\ab$, where $x$ is called a central
 variable, and each $z_i$ is a neighbour variable. 
A DL-\ab-term is a term of the form $z_i,x,f_i(x)$, or $o$. A 
DL-\pred-term is of the form $B_i(z_j)$, $B_i(x)$, $B_i(f_j(x))$, $B_i(o)$, $S_i(z_j,x)$, $S_i(x,z_j)$, $S_i(x,f_j(x))$, $S_i(f_j(x),x)$, $S_i(o,x)$ or $S_i(x,o)$. A DL-literal 
is either an equality of the form $A \equals \TOP$ (or just $A$) with $A$ a DL-\pred-term, 
or an (in)equality  between DL-\ab-terms. A DL-clause 
contains only body atoms of the form $B_i(x),S_i(z_j,x),$ or $S_i(x,z_j)$, and only DL-literals in the head.
To ensure completeness and termination of our calculus, we require 
that each $z_j$ in the head occurs also in the body and \textcolor{black}{that
if the body contains two or more neighbour variables $z_j$, then the clause is of the form DL4}. Note that clauses
in \cref{tab:ontology} satisfy these restrictions.
An \emph{ontology} is a finite set of DL-clauses. An ontology is $\logic$ if each DL-clause is of the form given in \cref{tab:ontology}. An $\logic$ ontology is 
$\mathcal{ALCHOI}$ if it does not contain axioms DL4 and all axioms DL2 satisfy $n = 1$;
it is $\mathcal{ALCHOQ}$ if it does not contain axioms DL6; it is 
$\mathcal{ALCHIQ}$ if it does not contain axioms  DL7-DL8. 
Furthermore, it is Horn if each axiom DL1 satisfies $n \leq m \leq n+1$ and 
each axiom DL4 satisfies $n =1$.
Finally, it is in \ELHO if it is Horn and contains only axioms DL1,  DL2 with $n = 1$, DL3, DL5, DL7, or DL8. 

\begin{table}[t]
\centering
\scalebox{0.74}{
\begin{tabular}{lrcl}
\rule{0pt}{3ex}  
DL1 & $ \displaystyle  \bigsqcap_{\substack{1 \leq i \leq n}}  B_i  \sqsubseteq \! \! \! \! \bigsqcup_{n+1 \leq i \leq m} \! \! \! \! B_i$  & $\rightsquigarrow$&$  \displaystyle \bigwedge_{1 \leq i \leq n } B_i(x) \rightarrow \displaystyle \! \! \! \! \bigvee_{n+1 \leq i \leq m} \! \! \! \! B_i(x) $  \\

  \multirow{2}{*}{DL2} & \multirow{2}{*}{$B_1 \sqsubseteq \geq n S.B_2$} & \multirow{2}{*}{$\rightsquigarrow$}  & $ B_1(x) \rightarrow B_2(f_i(x)),$~ $1 \leq i \leq n$   \\
 							 & & & $ B_1(x) \rightarrow S(x,f_i(x))$,~$1 \leq i \leq n$  \\
 							  & & & $ B_1(x) \rightarrow f_i(x) \nequals f_j(x)$, ~$ \! \!1 \leq i < j \leq n$  \\

DL3 & $\exists S. B_1 \sqsubseteq B_2$  & $\rightsquigarrow$  & $S(z,x) \wedge B_1(x) \rightarrow B_2(z) $ \\
 \multirow{2}{*}{DL4} & \multirow{2}{*}{$B_1 \sqsubseteq\, \leq \! n S.B_2$} & \multirow{2}{*}{$\rightsquigarrow$}  & $S(z_1,x) \wedge B_2(x) \rightarrow S_{B_2}(z_1,x)$  \\
 							 & & & $B_1(x) \wedge    \bigwedge_{\substack{1 \leq i \leq n+1}} S_{B_2}(x,z_i) \rightarrow$ \\
 							 & & & $\qquad \qquad \qquad  \bigvee_{1 \leq i < j \leq n+1} z_i \equals z_j$    \\
DL5 & $S_1 \sqsubseteq S_2$ & $ \rightsquigarrow$ & $S_1(z_1,x) \rightarrow S_2(z_1,x)$ \\ 
DL6 & $S_1 \sqsubseteq S^{-}_2$ & $\rightsquigarrow$ & $S_1(z_1,x) \rightarrow S_2(x,z_1)$ \\
DL7 & $\{o\} \sqsubseteq B_1$  & $ \rightsquigarrow$  & $\top  \rightarrow B_1(o)$ \\ 
DL8 & $B_1 \sqsubseteq \{o\}$  &$ \rightsquigarrow$ & $B_1(x) \rightarrow  x \equals o  $ 
\label{tab:ALCHOIQ} 
\end{tabular}}
\caption{DL axioms as clauses. Roles $S_{B_2}$ in DL4 are fresh. }
\label{tab:ontology}
\end{table}

\section{A CB Algorithm for  $\logic$}

Consequence-based reasoning combines features of both
(hyper-)tableau calculi and resolution. 
Although the presentation of CB calculi varies in the literature, all CB calculi 
that we know of share certain core characteristics.
First, they derive in a single run all consequences of a certain form
 (typically subsumptions between atomic concepts, $\top$ and $\bot$) and hence 
 they are not just refutationally complete. 
Second, like resolution, they compute a saturated set of clauses---represented either as DL-style axioms
\cite{BaBL05,Kaza09a,SiKH11a,SiMH14a} or using first-order notation \cite{DBLP:conf/kr/BateMGSH16,DBLP:conf/dlog/CucalaGH17,DBLP:conf/dlog/KarahroodiH17}---the shape of which is restricted to ensure termination.
Third, unlike resolution, 
where all  clauses are kept in a single set, CB calculi construct a graph-like \emph{context structure}
where clauses can only interact with other clauses in the same context or in neighbouring contexts, thus
guiding the reasoning process in a way that is reminiscent of
(hyper-)tableau calculi. Fourth, the expansion of the context structure during a run of the algorithm is determined by an \emph{expansion strategy}, which controls
 when and how to create or reuse contexts.

In the remainder of this section we define our CB calculus, and specify a reasoning algorithm based on it.
We then establish its key correctness and complexity properties.

\subsection{Definition of the Calculus}\label{sec:calculus}

Throughout this section we fix an arbitrary  ontology $\Onto$, and we let
$\FunctionsOnto$, $\UnpredsOnto$ and $\BinpredsOnto$ be the sets of functions, atomic concepts and atomic roles 
in $\Onto$, respectively; all our definitions and theorems are 
implicitly relative to $\Onto$.

The set of \emph{nominal labels} $\Pi$ for $\Onto$ is the smallest set containing the empty string and
every string $\rho$ of the form 
$S_1^{j_1} \cdot ... \cdot S_n^{j_n}$, with $j_k \in \mathbb{N}$, 
and $S_i \in \Binpreds$. The set of named individuals
$\AllIndiv$ for $\Onto$ is then defined as $\{o_{\rho} \;|\; o~ \text{individual in  } \Onto, \rho \in \Pi\}$.
Intuitively, the set $\AllIndiv$ consists of the individuals occurring explicitly in $\Onto$ plus a
set of \emph{additional nominals}, the introduction of which is reminiscent of 
existing (hyper-)tableau calculi \cite{HoSa05a,MoSH07a}.

Following \citeA{DBLP:conf/kr/BateMGSH16}, our calculus for $\logic$ represents all derived consequences in contexts 
as \emph{context clauses} in many-sorted equational logic, rather than DL-style axioms.
Context clauses use only variables $x$ and $y$, which carry
a special meaning. Intuitively, a context represents a set of similar elements in a model of the ontology;
when variable $x$ corresponds to such an element, $y$ corresponds to its predecessor, if it exists. This naming convention determines rule application in the calculus, and should
be distinguished from variables $x$ and $z_i$ in DL-clauses, where the latter can map, in a canonical model,
to both predecessors and successors of the elements assigned to $x$.
Context clauses are defined analogously to \cite{DBLP:conf/kr/BateMGSH16}, where the main difference
is that we allow context literals mentioning named individuals; furthermore, our calculus defines
a distinguished \emph{root context} where most inferences involving such literals take place. 
This context represents the non tree-like part of the model, and it exchanges information with other contexts using newly devised inference rules.
\begin{definition} \label{def:context-terms}
	A \textit{context \ab-term} is a term of sort $\ab$ which is either $x$, or $y$, or a named individual $o \in \AllIndiv$, 
	or of the form $f(x)$ for $f \in \FunctionsOnto$. A \textit{context \pred-term} is a term of sort $\pred$ 
	of the form $B(y)$, $B(x)$, $B(f(x))$, $B(o)$, $S(x,y)$, $S(y,x)$, $S(x,x)$, $S(x,f(x))$, $S(f(x),x)$, 
	$S(x,o)$, $S(o,x)$, $S(o,o')$, for $f \in \FunctionsOnto, o \mbox{ and } o' \in \AllIndiv$, $B \in \UnpredsOnto$, and $S \in \BinpredsOnto$. 
	A \textit{root context \ab-term} (\pred-term) is a term of sort $\ab$ (\pred) of the form $t \{x \mapsto o'\}$, with $t$ a context $\ab$-term ($\pred$-term) and $o' \in \AllIndiv$.
	A (root) \textit{context atom} is an equality of the form $\Atom \equals \TOP$, written simply as $\Atom$, with $\Atom$
	a context (root) $\pred$-term; a (root) \textit{context literal} is a (root) context atom, \textcolor{black}{an inequality $\TOP \nequals \TOP$}, or an equality or inequality between \ab-terms (\textcolor{black}{replacing $x$ by $o' \in \AllIndiv$}). 
	A (root) \textit{context clause} is a clause of (root) context atoms in the body and (root) context literals in the head.
	A \textit{query} clause has only atoms of the form $B(x)$, with $B \in \UnpredsOnto$.
\end{definition}
The kinds of information to be exchanged between adjacent contexts 
 is determined by
a set of \emph{triggers}, which are named after the rules that they activate.
\begin{definition}
The set of \textit{successor triggers} \SUtrig is the smallest set of atoms satisfying the following properties 
for each
clause $\Gamma \rightarrow \Delta$ in \Onto: \emph{(i)} $B(x) \in \Gamma$ implies $B(y) \in \SUtrig$;
\emph{(ii)} $S(x,z_i) \in \Gamma$ implies $S(x,y) \in \SUtrig$; and
\emph{(iii)}  $S(z_i,x) \in \Gamma$ implies $S(y,x) \in \SUtrig$.
The set of \textit{predecessor triggers} $\PRtrig$ is defined as the set of literals
   \begin{multline*}
         \{A\{x \mapsto y, \; y \mapsto x\} \mid A \in \SUtrig \} \cup \{B(y) \mid B \in \UnpredsOnto \} \cup \\
            \{ x \equals y \} \cup \{x \equals o \mid o \in \AllIndiv \} \cup   \{y \equals o \mid o \in \AllIndiv \}.
    \end{multline*}
The set of \emph{root successor triggers} $\rootSUtrig$ 
consists of all atoms $B(o)$, $S(y,o)$ and $S(o,y)$ with $B \in \UnpredsOnto$, $S \in \BinpredsOnto$ and $o \in \AllIndiv$. \textcolor{black}{The set of \emph{root predecessor triggers} $\rootPRtrig$ 
consists of $\rootSUtrig \cup \{B(y) \mid B \in \UnpredsOnto\} \cup  \{ y \equals o \mid o \in \AllIndiv\}$.}
\end{definition}
The definition of triggers extends that in \cite{DBLP:conf/kr/BateMGSH16} by considering 
equalities of a variable and an individual as information that should be propagated to predecessor contexts, and by
identifying a specific set of triggers for propagating information to and from the distinguished root context.

Same as in resolution calculi and other CB calculi, clauses are ordered 
using a term order $\succ$
based on a total order $\gtrdot$ on function symbols of sort \ab.
The order restricts the derived  clauses since only 
$\succ$-maximal literals can participate in inferences. The following definition specifies the conditions that $\succ$
must satisfy; although
each context can use a different $\succ$ order,  \ab-terms are compared in the same way
across all contexts since $\gtrdot$ is globally defined. 
In \cref{sec:appendix-order} we show how a context order  can be constructed once $\gtrdot$ is fixed.

\begin{restatable}{definition}{contextorder}\label{def:order}
Let $\gtrdot$ be a total order on symbols of $\FunctionsOnto$ and $\AllIndiv$ such that for every $\rho \in \Pi$, if $\rho = \rho' \cdot \rho''$, then $o_{\rho} \gtrdot o_{\rho'}$. A (root) \textit{context order} $\succ$ w.r.t. $\gtrdot$ is a strict order on (root) context atoms 
satisfying each of the following properties:
\begin{enumerateConditions}[wide, labelwidth=!, labelindent=0pt]
	\item \label{def:order:variable}$A \succ x \succ y \succ \TOP$ for each context \pred-term $A \neq \TOP$;
	\item \label{def:order:individual}$n \succ m$ for each pair $n,m \in \AllIndiv$ with $n \gtrdot m$;
	\item \label{def:order:function} $f(x) \succ g(x)$, for all $f,g \in \Sigma^{\mathcal{O}}_f$ with $f \gtrdot g$;
	\item \label{def:order:simplification}$t[s_1]_p \succ t[s_2]_p$ for any context term $t$, position $p$, and context terms $s_1,s_2$ such that $s_1 \succ s_2$;
	\item \label{def:order:subterm} $s \succ s|_p$ for each context term $s$ and proper position $p$ in $s$;
	\item \label{def:order:forbidden}$A \nsucc s$ for each atom $A \approx \textsf{true} \in \PRtrig$ \textcolor{black}{($\rootPRtrig$)} and context term $s \notin \{x,y,\textsf{true}\} \cup \AllIndiv$.
\end{enumerateConditions}
\end{restatable}
The main difference between Definition \ref{def:order} and the orderings used in prior work is the additional requirement that the global order $\gtrdot$
must satisfy on the set of nominals; this is necessary to ensure both completeness and termination.

We use a notion of
redundancy elimination analogous to that of prior work to significantly reduce the amount of clauses 
derived by the algorithm. 
\begin{definition}\label{def:redundancy-elim}
	A set of clauses $U$ contains a clause $\Gamma \rightarrow \Delta$ \textit{up to redundancy}, written
	$\Gamma \rightarrow \Delta \sbin U$ if
	\begin{enumerate}[wide, labelwidth=!, labelindent=0pt]
	\item \label{cond:redundancy:equality}  $t \approx t \in \Delta$ or $\{ t \equals s, t \nequals s \} \subseteq \Delta$  for some \ab-term $t,s$, or
	\item \label{cond:redundancy:subset} $\Gamma' \rightarrow \Delta' \in U $ for some $\Gamma' \subseteq \Gamma$ and $\Delta' \subseteq \Delta$. 
	\end{enumerate}

\end{definition}
The first condition in Definition \ref{def:redundancy-elim} captures tautological statements, whereas the second condition 
captures clause subsumption. Similarly to prior work, clauses $A \rightarrow A$ are not deemed tautological in our calculus since they imply that
atom $A$ may hold in a context.

We next define the notion of a context structure $\ContextStructure$ 
as a digraph. Each node $v$, labelled with 
a set of clauses $\SC{v}$, represents a set of ``similar'' terms in a model of $\Onto$; 
edges, labelled with a function symbol, represent connections between neighboring contexts. 
Each context $v$ is assigned
a core $ \core{v}$ specifying the atoms that must hold for all terms in the canonical model described by the context, and
 a term order $\succ_v$ that restricts the inferences applicable to $\SC{v}$; since  $\core{v}$ holds 
 implicitly in a context, the conjunction  $\core{v}$ is not included in the body of any clause in 
 $\SC{v}$.
\begin{definition}\label{def:context-structure}
A context structure for $\Onto$ is a tuple $\ContextStructure = \langle \Contexts, \Edges, \core{}, \SC{}, \gtrdot, \succ  \rangle$ 
where $\Contexts$ is a finite set of
\emph{contexts} containing the  \emph{root context} $\rootc$; 
$\Edges$ is a subset of $\Contexts \times \Contexts \times \FunctionsOnto$;
$\core{}$ is a function mapping
each context $v$ to a conjunction \core{v} of atoms of the form $B(x)$, $S(x,y)$, $S(y,x)$; 
$\SC{}$ is a function mapping each
non-root context $v$  to a set of context clauses  and
$\rootc$ to a set of root context clauses, $\gtrdot$ is a total order 
 and  $\succ$ is a function mapping each (root) context $v$ to a (root) context order $\succ_v$ w.r.t. $\gtrdot$ where
 all $\gtrdot$ and $\succ_v$ 
satisfy  Definition \ref{def:order}.
\end{definition}

We now define when a context structure is sound with respect to $\Onto$.
In prior work, soundness was defined by requiring that clauses derived by the calculus are
logical consequences of $\Onto$. 
Our calculus, however, introduces additional nominals, and clauses mentioning them are
not logical consequences of $\Onto$. 
Furthermore, not all the the additional nominals generated by our calculus will correspond to actual elements of a canonical model.
To address this difficulty, we introduce the following notions of $N$-reduction and $N$-compatibility.

\begin{definition}
Let $N$ be a (possibly empty) set of additional nominals in $\AllIndiv$. 
Interpretation $\II$ with domain $\Delta^{\II}$ is $N$-compatible if the following conditions hold for each
non-empty nominal label $\rho$, which we rewrite as $\rho = \rho' \cdot S$:
\begin{itemize}[wide, labelwidth=!, labelindent=0pt]
\item[-]  if $o_{\rho} \in N$,  $\II \nmodels S(o_{\rho'},u)$ for each $u \in \Delta^{\II}$,

\item[-]  if $o_{\rho} \not\in N$, $\II \models S(o_{\rho'},o_{\rho})$, \textcolor{black}{and there is $k_0 \in \mathbb{N}$ s.t. for each $u \in \Delta^{\II}$, $\II \models S(o_{\rho'},u)$ implies $\II \models u \equals o_{\rho' \cdot S^{k}}$ for some $k \leq k_0$, and for each $k>k_0,$ $\II \models o_{\rho' \cdot S^{k}} \equals o_{\rho' \cdot S^{1}}$'.}

\end{itemize}
The $N$-reduction $N(\Gamma \to \Delta)$ of a clause $\Gamma \to \Delta$ is empty
if an element of $N$ occurs in $\Gamma$ \textcolor{black}{or in an inequality in $\Delta$}, and the clause $\Gamma \to N(\Delta)$ otherwise, with $N(\Delta)$ 
obtained from $\Delta$ by removing all literals mentioning a term in $N$. 
\end{definition}
Intuitively, given $N$ and a model $\II$ of $\Onto$, the notion of $N$-reduction tests whether it is possible to map 
the additional nominals not in $N$ occurring in derived clauses  to actual domain elements of $\II$ without
invalidating the model.  This intuition leads to the following notion of soundness.

\begin{definition}\label{def:sound-context-structure}
A context structure $\ContextStructure \! = \! \langle \Contexts, \Edges,\! \core{}, \!\SC{}, \! \gtrdot,\! \succ \rangle$
is \newline \textit{sound} if, for every model $\II$ of $\Onto$, there exists a (possibly empty) set $N$ of additional nominals
and an $N$-compatible conservative extension $\JJ$ of $\II$ satisfying the following clauses:
\emph{(i)} $N(  \core{v} \wedge \Gamma  \rightarrow  \Delta ) $ for each $v \in \Contexts$ and 
each $\Gamma \rightarrow \Delta$ in $\SC{v}$;  and
\emph{(ii)}  $N(\core{u} \rightarrow \core{v} \{x \mapsto f(x), y \mapsto x\})$ for each  $\langle u, v, f \rangle \in \Edges$. 
\end{definition}

As in existing CB calculi, the rules of our calculus 
are parameterised by an \textit{expansion strategy}  used to decide 
whether to create new contexts or re-use already existing ones. 

\begin{definition}
An expansion strategy \Strategy is a polynomially computable function which takes as 
input a triple $(f,K_1,\ContextStructure)$, where $f \in \FunctionsOnto$, $K_1 \subseteq \SUtrig$, $\ContextStructure = \langle \Contexts, \Edges, \SC{}, \core{}, \gtrdot, \succ \rangle$ is a context structure, and returns a triple $(v,\core{},\succ)$ such that $\core{} \subseteq K_1$, $\succ$ is a context order w.r.t $\gtrdot$, and either $v \notin \mathcal{V}$ or otherwise $v \neq v_r$, $\core{}= \core{v}$ and $\succ\ =\ \succ_v$.
\end{definition}

Three expansion strategies are typically considered in practice (see \cite{SiKH11a} for details).
The \textit{trivial} strategy pushes all inferences to a single context $v_{\top}$ with empty core and always returns $(v_{\top},\top)$. 
The \textit{cautious} strategy only creates contexts for concept names in existential restrictions; it returns $(v_{\top},\top)$ 
unless $f$ occurs in \Onto in exactly 
 one atom $B(f(x))$ with $B \in \Sigma_B^{\mathcal{O}}$ and $B(x) \in K_1$, in which case
 it returns $(v_B,B(x))$. Finally, the \textit{eager} strategy creates a new context for each conjunction $K_1$ by returning $(v_{K_1},K_1)$.

\begin{table}
\centering
\scalebox{0.8}{
\begin{tabular}{c|l}
\hline \hline
\multirow{2}{*}{\rotatebox{90}{\textsf{Core}}}  
& If 1. $A \in \core{v}$ \\
  & then add $\top \rightarrow A$ to $\SC{v}$. \\ \hline 
 
\multirow{3}{*}{\rotatebox{90}{\textsf{Hyper}}} 
& If 1. $\bigwedge_{i=1}^n A_i \rightarrow \Delta \in \Onto$ and $\sigma(x)\!=\!x$ if $v \neq v_r$ or $\sigma(x)\! \in  \AllIndiv$ o.w.\ \\
 & $\phantom{\mbox{If}}$ 2. and $\Gamma_i \rightarrow \Delta_i \vee A_i \sigma \in \SC{v}$ with $\Delta_i  \nsucceq_v A_i \sigma$, for $1 \leq i \leq n$, \\
 & then add $\bigwedge_{i=1}^n \Gamma_i \rightarrow \bigvee_{i=1}^n \Delta_i \vee \Delta \sigma$ to $\SC{v}$.\\ \hline \hline
  
\multirow{5}{*}{\rotatebox{90}{\textsf{Eq}}}
 & If 1. $\Gamma_1 \rightarrow \Delta_1 \textcolor{black}{\vee} s_1 \approx t_1 \in \SC{v}$ with $t_1 \not \succeq_v s_1$ and $\Delta_1 \not \succeq_v s_1 \approx t_1$, \\
 & $\phantom{\mbox{If}}$ 2.  $\Gamma_2 \rightarrow \Delta_2 \textcolor{black}{\vee} s_2 \bowtie t_2 \in \SC{v}$ with $\bowtie\, \, \in \! \{\approx, \napprox\}$, \\
 & $\phantom{\mbox{If 3.}}$  $t_2 \not \succ_v s_2$, $\Delta_2 \nsucceq_v s_2 \bowtie t_2$, and $s_2|_{p}$ is not a variable, \\ 
 & \textcolor{black}{$\phantom{\mbox{If 3.}}$ and if $s_2|_{p} \in \AllIndiv$, then $s_2$ contains no function symbols.} \\
 & then add $\Gamma_1 \wedge \Gamma_2 \rightarrow \Delta_1 \vee \Delta_2 \vee s_2[t_1]_p \bowtie t_2$ to $\SC{v}$. \\ \hline

\multirow{2}{*}{\rotatebox{90}{\textsf{Ineq}}}
 & If 1. $\Gamma \rightarrow \Delta \vee t \napprox t \in \SC{v}$, \\
 & then  add $\Gamma \rightarrow \Delta$ to $\SC{v}$.\\ \hline

\multirow{3}{*}{\rotatebox{90}{\textsf{Fact}}}
 & If 1. $\Gamma \rightarrow \Delta \vee s \approx t \vee s \approx t' \in \SC{v}$, \\
 & $\phantom{\mbox{If}}$ 2. with $\Delta \cup \{s \approx t\} \nsucceq_v s \approx t'$ and $t' \not \succ_v s$,\\
 & then  add $\Gamma \rightarrow \Delta \vee t \napprox t' \vee s \approx t'$ to $\SC{v}$. \\ \hline 

\multirow{3}{*}{\rotatebox{90}{\textsf{Elim}}}  
 & If 1. $\Gamma \rightarrow \Delta \in \SC{v}$ \\
 &  $\phantom{\mbox{If}}$ 2. and $\Gamma \rightarrow \Delta \sbin \SC{v}\backslash (\Gamma \rightarrow \Delta)$ \\
 &  then remove  $\Gamma \rightarrow \Delta$ from $\SC{v}$. \\ \hline  \hline
 
 \multirow{7}{*}{\rotatebox{90}{\textsf{Pred}}}  
& If 1. $\bigwedge_{i=1}^{n} A_i \wedge  \bigwedge_{i=1}^{m} C_i   \rightarrow \bigvee_{i=1}^{k} L_i \in \SC{v}$ for $v \neq v_r$,\\
& $\phantom{\mbox{If 1. }}$ where each $C_i$ is ground, and each $A_i$ is nonground\\
& $\phantom{\mbox{If}}$ 2. $L_i \in \PRtrig$ for each nonground $L_i$,  \\
& $\phantom{\mbox{If}}$ 3. and there is $\langle u,v,f \rangle \in \mathcal{E}$ such that \\
& $\phantom{\mbox{If}}$ 4. for each $A_i$, there is $\Gamma_i \rightarrow \Delta_i \vee A_i \sigma \in \mathcal{S}_u$ with $\Delta_i  \nsucceq_u A_i \sigma$; \\
& $\phantom{\mbox{If}}$ 5. $\sigma\!\!=\!\! \{ y\!  \mapsto\!  x, \! x \mapsto \! f(x)\!  \}$ \textcolor{black}{if $\! u \! \neq\!  v_r$, $\sigma = \{ y \mapsto o, x\!  \mapsto\!  f(o) \}$ o.w.},  \\
& then add $\bigwedge_{i=1}^{n} \Gamma_i \wedge  \bigwedge_{i=1}^{m} C_i   \rightarrow \bigvee_{i=1}^n \Delta_i \, \vee  \bigvee_{i=1}^{k} L_i \sigma $ to $\SC{v}$, \\ \hline

\multirow{11}{*}{\rotatebox{90}{\textsf{Succ}}} 
 & If 1. $\Gamma \rightarrow \Delta \vee A \in \mathcal{S}_u$ where $\Delta  \nsucceq_u A$  \\
 & $\phantom{\mbox{If 1.}}$  and $A$ contains $f(x)$ if $u \neq v_r$ or $f(o)$ for some $o \in \AllIndiv$ o.w.,  \\
 & $\phantom{\mbox{If}}$ 2. there is no $\langle u, v, f\rangle \in \mathcal{E}$ s.t. $A' \rightarrow A' \sbin \SC{v}$ $\forall$  $A' \in K_2 \backslash \core{v}$ \\
 & then 1. let $\langle v, {\core{}}' , \succ' \rangle = \textsf{strat}(f,K_1,\mathcal{D})$ and if $v \notin \mathcal{V}$, then let \\
 & $\phantom{\mbox{then}}$ 2. $\mathcal{V} = \mathcal{V} \, \cup \{v\}$, and $\textsf{core}_v = \textsf{core}'$, $\succ_v \, = \, \succ'$, and $\SC{v} = \emptyset$. \\
 & $\phantom{\mbox{then}}$ 3. Add the edge $\langle u, v ,f \rangle$ to $\mathcal{E}$. \\
 & $\phantom{\mbox{then}}$ 4. Add $A' \rightarrow A'$ to $\SC{v}$ for each $A' \in K_2 \backslash \textsf{core}'_v$, where  \\ 
 & $\phantom{\mbox{then}}$ 5. $\sigma =\{ y \mapsto x, x \mapsto f(x)\}$ if $u \neq v_r$, or \\
 & $\phantom{\mbox{then}}$ 6. $\sigma = \{ y \mapsto o, x \mapsto f(o)\}$ if $u = v_r$, and \\  
 & $\phantom{\mbox{then}}$ 7.  $K_1 = \{A' \in \SUtrig\;|\; \top \rightarrow A' \sigma \in \mathcal{S}_u\}$, and \\
 & $\phantom{\mbox{then}}$ 8.  $K_2 = \{A' \in \SUtrig \;|\; \Gamma' \rightarrow \Delta' \vee A' \in \mathcal{S}_u \}$  and $\Delta' \nsucceq_u A' \sigma $.\\ \hline \hline

\end{tabular}}
\caption{Revised inference rules for the $\mathcal{ALCHIQ}$ calculus. }
\label{tab:oldrules}
\end{table}

\begin{table}
\centering
\scalebox{0.8}{
\begin{tabular}{c|ll}
\hline \hline

\multirow{5}{*}{\rotatebox{90}{\textsf{Join}}} 
& If 1. $A \wedge \Gamma \rightarrow \Delta \in \SC{v}$, with $A$ ground and $o$ occurring in $A$, and \\
& $\phantom{\mbox{If}}$ 2. $\Gamma' \rightarrow \Delta' \vee \Delta'' \vee A \in \SC{v}$, with $\Delta' \cup \Delta'' \nsucceq_v A$, or \\
& $\phantom{\mbox{If}}$ 3. $\Gamma' \rightarrow \Delta' \vee A' \in \SC{v}$, with $\Delta'  \nsucceq_v A'$, $A'\{x\mapsto o\}=A$, and \\
& $\phantom{\mbox{If 3.}}$ $\Gamma' \rightarrow \Delta'' \vee x \equals o \in \SC{v}$, $\Delta''  \nsucceq_v x \equals o$, $\Gamma'=\top$,\\
& then  add $\Gamma \wedge \Gamma' \rightarrow \Delta \vee \Delta' \vee \Delta''$ to $\SC{v}$.\\ \hline 

\multirow{6}{*}{\rotatebox{90}{$r$-\textsf{Succ}}} 
 & If 1. $\Gamma \rightarrow \Delta \vee A\sigma \in \mathcal{S}_u$ where $\Delta  \nsucceq_u A \sigma$, with $u \neq v_r$, \\
 & $\phantom{\mbox{If}}$ 2. $A  \in \rootSUtrig$, with $o \in \AllIndiv$ occurring in $A$, $\sigma =\{ y \mapsto x \}$, and \\

 & $\phantom{\mbox{If}}$ 3. there is no $\langle u, v_r, o \rangle \in \mathcal{E}$ s.t. $A \rightarrow A \sbin \SC{v_r}$, and \\

 & $\phantom{\mbox{If}}$ 4. (*) there is no $\Gamma'' \to \Delta'' \vee \bigvee_{i=1}^n L_i \in \SC{u}$ with $\Gamma''\subseteq \Gamma$,   \\ 
 & $\phantom{\mbox{If 4.}}$ $\Delta'' \subseteq \Delta$, \textcolor{black}{and $L_i$ of the form $x \equals o_i, y \equals o_i, x \equals y$}, \\ 
 & then add the edge $\langle u, v_r ,o \rangle$ to $\mathcal{E}$ and $A \rightarrow A $ to $\SC{v_r}$. \\ \hline

\multirow{6}{*}{\rotatebox{90}{$r$-\textsf{Pred}}} 
& If 1. $\bigwedge_{i=1}^{n} A_i \wedge \bigwedge_{i=1}^{m} C_i \rightarrow \bigvee_{i=1}^{k} L_i \in \SC{v_r}$, where \\
& $\phantom{\mbox{If 2.}}$ $L_i \in \textcolor{black}{\rootPRtrig}$ for each nonground $L_i$, each $C_i$ is ground,  \\ 
& $\phantom{\mbox{If 3.}}$ $A_i \in \rootSUtrig$, and $o_i$ is the named individual in $A_i$; and\\
& $\phantom{\mbox{If}}$ 2. there is $\langle u,v_r,o_i \rangle \in \mathcal{E}$ for each $o_i$ such that  \\
& $\phantom{\mbox{If 5.}}$ $\Gamma_i \rightarrow \Delta_i \vee A_i \sigma \in \mathcal{S}_u$ verifies (*), $\Delta_i  \nsucceq_u A_i \sigma $, $\sigma(y)=x$,\\
& then add $\bigwedge_{i=1}^{n} \Gamma_i \wedge \bigwedge_{i=1}^m C_i \rightarrow \bigvee_{i=1}^{n} \Delta_i \vee \bigvee_{i=1}^{k} L_i \sigma$ to $\SC{u}$. \\ \hline 

\multirow{5}{*}{\rotatebox{90}{\textsf{Nom}}}
 & If 1. $\bigwedge_{i=1}^n A_i \rightarrow \bigvee_{i=1}^{m} L_i \vee \bigvee_{i=m+1}^{k} L_i \in \Onto$, with $L_i$ $\ab$-equalities \\
 & $\phantom{\mbox{If}}$ 2. $\Gamma_i \rightarrow \Delta_i \vee A_i \sigma \in \SC{v_r}$, with $\Delta_i  \nsucceq_{v_r} A_i \sigma$ and $\sigma(x) = o$, and\\

 & $\phantom{\mbox{If}}$ 3. $L_i \sigma$ is of the form $y \equals y$ or $y \equals f_i(o_i)$ iff $m+1 \leq i \leq k$, \\
 & \textcolor{black}{then add $\Gamma \to \Delta \vee \bigvee_{i=1}^K y \equals o'_{\rho \cdot S^i}$ to $\SC{v_r}$, where  $\Gamma = \bigwedge_{i=1}^n \Gamma_i$ ,} \\
 & \textcolor{black}{$\phantom{\mbox{then 1.}}$ $\Delta = \bigvee_{i=1}^n \Delta_i \vee  \bigvee_{i=1}^{m} L_i\sigma $, and $K\!\!+\!\!1 = \max ( \textcolor{black}{i} \mid \{ z_i \mbox{ in } \Onto\})$ } \\ 

  \hline \hline
\end{tabular}}
\caption{Novel rules for reasoning with nominals in $\mathcal{ALCHOIQ}$. }
\label{tab:newrules}
\end{table} 

The inference rules of our calculus are specified in
 Tables $\ref{tab:oldrules}$ and $\ref{tab:newrules}$. As in prior work,
a rule is not triggered if the clauses that would be derived are already 
contained up to redundancy in the corresponding contexts.
 Rules in Table \ref{tab:oldrules} are a simple generalisation of those in
  \cite{DBLP:conf/kr/BateMGSH16} for $\mathcal{ALCHIQ}$ to take into account that certain rules
  can be applied to the distinguished root context and that
  clauses propagated to predecessor contexts
  may contain ground atoms.
  The \ruleword{Core} rule ensures that all atoms in a context's core hold. 
  The \ruleword{Hyper} rule performs hyperresolution
  between clauses in a context and ontology clauses; in prior work, variable $x$ had to map to
  itself in the $\sigma$ used in the rule, but now $x$ maps to an individual if the rule is applied on 
  the root context.
  The \ruleword{Eq}, \ruleword{Ineq} and \ruleword{Fact} rules implement equality
  reasoning, and the \ruleword{Elim} rule performs redundancy elimination as in prior work.
  The  \ruleword{Pred} rule performs hyperresolution between a context and a predecessor context. 
  The rule does not apply to the root context, but the calculus provides another rule for that. Ground and nonground body atoms are treated differently;
   the latter are simply copied to the body of the derived clause.
  Finally, the \ruleword{Succ} rule extends the context structure using the expansion strategy as in prior work.
  As in the \ruleword{Hyper} rule, variable $x$ maps to a named individual on the root context.

The rules in Table  \ref{tab:newrules}  handle reasoning with nominals. Rule \ruleword{Join} 
corresponds to a resolution step between two ground atoms of different clauses in the same context. 
Rule $r$-\ruleword{Succ}  
complements \ruleword{Succ} by dealing with information propagation from any non-root context to the root context;
in turn, $r$-\ruleword{Pred} complements \ruleword{Pred} in an analogous way. In contrast to  previous calculi, rules
 $r$-\ruleword{Succ} and $r$-\ruleword{Pred} can be used to exchange information between
 the root context and 
 any other context, not just a neighboring one; this is due to the fact that nominal reasoning
 is intrinsically non-local.
Finally, rule \ruleword{Nom} introduces additional 
nominals when an anonymous element of the canonical model may become arbitrarily interconnected. The \ruleword{Nom} rule does not apply if the input ontology 
lacks either inverse roles, or nominals, or number restrictions.

\subsection{The Reasoning Algorithm and its Properties}

We can obtain a sound and complete reasoning algorithm for $\logic$ by exhaustively applying
the inference rules in Tables \ref{tab:oldrules} and \ref{tab:newrules} on a suitably initialised context
structure. 
This follows from the calculus satisfying two properties analogous to those required by CB calculi in prior work.
The \emph{soundness property}
ensures that the application of an inference rule to a sound context structure yields another
sound context structure. The \emph{completeness property} ensures that any query clause entailed by $\Onto$ will be contained up to redundancy in 
a suitably initialised context of a saturated context structure. 

\begin{restatable}[Soundness]{theorem}{soundness}
Given a context structure \ContextStructure which is sound for \Onto, and an arbitrary expansion strategy, 
the application of a rule from Table \ref{tab:oldrules} or Table \ref{tab:newrules} to \ContextStructure with 
respect to \Onto yields a context structure which is sound for \Onto. 
\end{restatable}

\begin{restatable}[Completeness]{theorem}{completeness}
\label{theorem:completeness}
	Let \ContextStructure be a context structure
	which is sound for \Onto and such that no rule of  Table \ref{tab:oldrules} or Table \ref{tab:newrules} can be applied to it. 
	Then, for each query clause $ \InputQueryLHS \to \InputQueryRHS$
	and each context $q \in \Contexts$ such that all of the following conditions hold, we have that
	$\InputQueryLHS \rightarrow \InputQueryRHS \sbin \SC{q}$ also holds.
	\vspace*{-0.5em}
\begin{enumerateConditions}[label={C\arabic*.},ref={C\arabic*},leftmargin=2.5em]
\item \label{theorem:completeness:query} $ \Onto \models \InputQueryLHS \to \InputQueryRHS $ .
\item \label{theorem:completeness:RHS}  For each context atom $A \in \InputQueryRHS$ and each  $A'$ of the form $B(x)$ such that $A \succ_q A'$, we have $A' \in \InputQueryRHS$.
\item \label{theorem:completeness:LHS}  For each $A \in \InputQueryLHS$, we have $\InputQueryLHS \rightarrow A \sbin \SC{q}$.
\end{enumerateConditions} 
\vspace*{-0.5em}
\end{restatable}
	\vspace*{-0.15em}
To test whether $\Onto$ entails a query clause $ \InputQueryLHS \to \InputQueryRHS$, an algorithm can proceed as follows. 
In \emph{Step 1}, create
an empty context structure $ \ContextStructure $,
and fix an expansion strategy. 
In \emph{Step 2}, introduce a context $q$ into $\ContextStructure$, 
set its core to $\InputQueryLHS$, and initialise the order $\succ_q$ in a way that is consistent with Condition C2 in 
Theorem \ref{theorem:completeness}. Finally, in \emph{Step 3}, saturate $\ContextStructure$ over the inference rules of the calculus and check
whether  $ \InputQueryLHS \to \InputQueryRHS$ is contained up to redundancy in $\SC{q}$.
Such algorithm generalises to check in a single run a set of input query clauses by 
initialising in Step 2 a context $q$ for each query clause.

Our algorithm may not terminate if the expansion strategy introduces infinitely many contexts. 
Termination is, however, ensured for strategies introducing finitely many contexts,
such as those  discussed in section 3.1.

\begin{restatable}{proposition}{termination}\label{prop:termination-complexity}
    The algorithm consisting of Steps 1--3 terminates if the  expansion strategy
introduces finitely many contexts \textcolor{black}{and rule \ruleword{Join} is applied eagerly}. 
    If the 
    expansion strategy introduces at most exponentially many contexts, the algorithm runs in  triple exponential time in the size of $\Onto$.
\end{restatable}
Our algorithm is not worst-case optimal for $\logic$ (an $\textsc{NEXPTIME}$-complete logic), as 
it can generate a doubly exponential number of additional nominals and a number of clauses per context
that is exponential in the size of the relevant signature; thus, each context can contain a triple exponential number of clauses
in the size of $\Onto$.
We can show, however,
worst-case optimality for the well-known fragments of $\logic$, and thus
obtain pay-as-you-go
behaviour.

\begin{restatable}{proposition}{payasyougo}\label{prop:worst-case-el}
For any expansion strategy introducing at most exponentially many contexts, 
the algorithm consisting of Steps 1--3 runs in worst-case exponential time in the size of $\Onto$
if $\Onto$ is either $\mathcal{ALCHIQ}$, or $\mathcal{ALCHOQ}$, or $\mathcal{ALCHOI}$, or
if it is Horn.
Furthermore, for \ELHO ontologies, 
the algorithm runs in polynomial time in the size of $\Onto$ with
either the cautious or the eager strategy.
\end{restatable}

Note that the strategies discussed in section 3.1 introduce at most exponentially many contexts. We conclude this section with an example illustrating the 
application of our calculus.

\begin{example}
Let $\Onto_1$ contain the following clauses.
 
\newcounter{clauses}

\scalebox{0.98}{\parbox{\linewidth}{
\centering
		\begin{tabular}{llll}
		$A(x) \to R(x,f(x))$& \clabel{onto1} &  $A(x) \to B_1(f(x))$ & \clabel{onto2} \\
		$A(x) \to R(x,g(x))$& \clabel{onto4} &  $A(x) \to B_2(g(x))$ & \clabel{onto5} \\
		$B_1(x) \to S(o,x)$ & \clabel{onto3} & $B_2(x) \to S(o,x)$ & \clabel{onto6}  \\
		\multicolumn{3}{c}{\, \, \, \,    \, \, $S(x,z_1) \wedge S(x,z_2) \to z_1 \equals z_2$} &  \clabel{onto7} \\
		\multicolumn{3}{c}{$\! \! \! \! \! \! R(z_1,x)\wedge  B_1(x) \wedge B_2(x)   \to C(z_1)$} &  \clabel{onto8} \\
		\end{tabular}	
}}

We  check whether $\Onto_1 \models A(x) \to C(x)$ using the eager expansion strategy. 
Figure \ref{fig:example1} summarises the inferences relevant to deriving the query clause. \Cref{cl:a1,cl:a2,cl:a3,cl:a4,cl:a5,cl:a6,cl:a7,cl:a8,cl:a9,cl:a10,cl:a12} are attached to context $v_A$ having core $A(x)$, \cref{cl:1b1,cl:1b2,cl:1b3,cl:1b4}
 to $v_{B_1}$ with core $B_1(x) \wedge S(y,x)$,  \cref{cl:2b1,cl:2b2,cl:2b3,cl:2b4} to $v_{B_2}$ with core $B_2(x) \wedge S(y,x)$, and the remaining clauses
to the root context $v_r$ with empty core.

\begin{figure}
\scalebox{0.8}{\parbox{\linewidth}{
{{
\hspace*{-2.7em}
	\begin {tikzpicture}[-latex ,auto ,node distance = 2cm ,on grid ,
	semithick ,
	state/.style ={ circle ,top color =white , bottom color = processblue!20 ,
		draw,processblue , text=black , minimum width =1 cm},
	stat/.style ={ circle ,top color =white , bottom color = processred!40 ,
			draw,processred , text=black , minimum width =0.5 cm}]
	\node[state,label={[align=left] 90:  \begin{tabular}{llllll} 
			$\top \rightarrow A(x)$  & \clabel{a1} &   $\top \rightarrow B_1(f(x))$ & \clabel{a2} & $\top \rightarrow R(x,f(x))$ & \clabel{a3}\\
		    $\top \rightarrow B_2(g(x))$  & \clabel{a4} & $\top \rightarrow R(x,g(x))$  & \clabel{a5} &  $\top \to f(x) \equals o'$ & \clabel{a6} \\ 
			$\top \to B_1(o')$  & \clabel{a7} &      $\top \to R(x,o')$   & \clabel{a8} &  $\top \to g(x) \equals o'$ & \clabel{a9}   \\ 
		     $\top \to B_2(o')$ & \clabel{a10} & \multicolumn{4}{l}{\phantom{$ B_1(o') \wedge B_2(o') \to C(x)$} $ $ $\top \rightarrow C(x) \, \, \quad \quad \quad $  \clabel{a12}} \\
	   \end{tabular}  }](A) {$v_{A}$};
	
	\node[state,label={[label distance = -1.5cm]180: \begin{tabular}{lll} \\ \\ $\top \rightarrow R(y,x)$ & \clabel{1b1} & $\phantom{AAA}$ \\ $\top \rightarrow B_1(x)$ & \clabel{1b2}  \\ $\top \rightarrow S(o,x)$  & \clabel{1b3} \\ $\top \rightarrow x \equals o'$  & \clabel{1b4} \\  \end{tabular}  
		}] (B) [ left = 2cm of A] {$v_{B_1}$};
	 
	\node[state,label={[label distance = -1.5cm]0:\begin{tabular}{lll}\\ \\ $\phantom{AAA}$ & $\top \rightarrow R(y,x)$ & \clabel{2b1} \\ & $\top \rightarrow B_2(x)$ & \clabel{2b2}   \\ & $\top \rightarrow S(o,x)$  & \clabel{2b3}   \\ & $\top \rightarrow x \equals o'$  & \clabel{2b4}  \\ \end{tabular}
		} ] (C) [ right = 2cm of A] {$v_{B_2}$};
		
	{\node[stat,label={[label distance = -0.7cm] 270 :\begin{tabular}{llllllll}   
	 &	$S(o,y) \rightarrow S(o,y)$ & \clabel{z1}  & & $\phantom{AAAAa}$ & $\quad \quad \quad \, \, S(o,y) \rightarrow y \equals o'$   & \clabel{z2} & \\ 
	 & $R(y,o') \rightarrow R(y,o')$ & \clabel{z4}  &  & &$\quad \quad \quad \, \, B_1(o') \to B_1(o')$  & \clabel{z5}  \\	  
	 & $B_2(o') \to B_2(o')$ &  \clabel{z6} & \multicolumn{5}{c}{$ \,\, \,R(y,o') \wedge B_1(o') \wedge B_2(o') \to C(y)$ \quad \clabel{z7}} 	   \end{tabular}
			 }] (o) [ below  = 1.5cm  of A] {$v_{r}$};}

	\path (A) edge node[above] {\scriptsize $f$} (B);
	{\path (A) edge node[above] {\scriptsize $g$} (C);}
	{\draw [->] (B) to    node[above] {\scriptsize $o$}  (o);}
	{\draw [->] (C) to   node[above] {\scriptsize $o$} (o);}
		{\draw [->] (A) to   node[right] {\scriptsize $o'$} (o);}

\end{tikzpicture}
} 
}}}
\vspace*{-1em}
\caption{Calculus execution for Example 3. $o'$ stands for $o_{S^{1}}$.}
\label{fig:example1}
\end{figure}
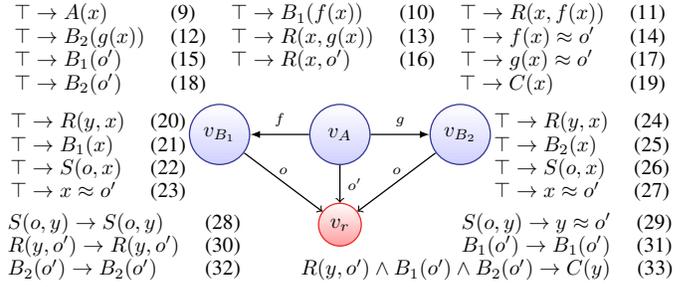

Consider the state of the context structure once \cref{cl:a1,cl:a2,cl:a3,cl:a4,cl:a5,cl:1b1,cl:1b2,cl:1b3} have 
been derived as in the calculus from \citeA{DBLP:conf/kr/BateMGSH16}. We  apply  $r$-\ruleword{Succ} to \cref{cl:1b3}, which 
creates an $o$-labelled edge to $v_{r}$ and adds \cref{cl:z1}; 
rule \ruleword{Nom}  derives \cref{cl:z2} in $v_{r}$, with
$o' = o_{S^{1}}$ an additional nominal. This clause   
can be back-propagated with $r$-\ruleword{Pred} to yield
\cref{cl:1b4} in $v_{B_1}$; in turn, this can be back-propagated with \ruleword{Pred} 
to yield \cref{cl:a6} in $v_A$. Two applications of \ruleword{Eq} yield \cref{cl:a7,cl:a8}. 
Proceeding analogously in context $v_{B_2}$, we derive \cref{cl:2b1,cl:2b2,cl:2b3,cl:2b4} and then \cref{cl:a9,cl:a10}. 
Next, we  apply $r$-\ruleword{Succ} to \cref{cl:a7,cl:a8,cl:a10} to derive \cref{cl:z4,cl:z5,cl:z6}. 
The \ruleword{Hyper} rule can  be applied to these clauses to derive \cref{cl:z7}. 
The head of this clause is in $\rootPRtrig$, so 
we can back-propagate using $r$-\ruleword{Pred} 
with the edge from  $v_{A}$ and \cref{cl:a7,cl:a8,cl:a10}
 to derive our target, \cref{cl:a12}.  
\end{example}

 \section{Conclusion and Future Work}

We have presented the first CB reasoning algorithm for 
a DL  featuring all Boolean operators, role hierarchies, inverse roles, nominals, and 
number restrictions. 
We see many challenges for future work. First, our algorithm runs in triple exponential time, when it should be possible 
to devise a doubly exponential time algorithm; we believe, however, that deriving such tighter upper bound
would require a significant modification of our approach.
Second, our algorithm should be extended with datatypes in order to cover all of OWL 2 DL.
Finally, we are implementing our algorithm as an extension of the Sequoia system \cite{DBLP:conf/kr/BateMGSH16}.
 We expect good performance from the resulting system, as our calculus only steps beyond pay-as-you-go behaviour in the rare situation where
 disjunctions, nominals, number restrictions and inverse roles interact simultaneously. 
 
 \vspace{-0.10in}
 \section*{Acknowledgments}

Research  supported by the SIRIUS  centre for Scalable Data Access,  and
the EPSRC projects DBOnto,
MaSI$^3$, and ED$^3$.

\newpage


\bibliographystyle{named}
\bibliography{consolidated}
 
\appendix

\onecolumn

\section{Context orders}
\label{sec:appendix-order}

\contextorder*

Given an order $\gtrdot$ verifying the properties in \cref{def:order}, 
we can obtain a context term order $\succ$ as follows.
Extend $\gtrdot$ on variables so that $x \gtrdot y$. Then,
extend it (arbitrarily) to symbols of sort $\pred$.
Next, we let $\succ$  be the \emph{lexicographic path order} (LPO)
\cite{BaN98} over context $\ab$- and $\pred$-terms induced by
$\gtrdot$. The well-known properties of LPOs ensure that $\succ$  is a
total simplification order on all context terms that satisfies
\cref{def:order:variable,def:order:individual,def:order:function,def:order:simplification,def:order:subterm}
in \cref{def:order}.

To also satisfy \cref{def:order:forbidden}, we relax $\succeq$ by 
dropping all ${A \succ s}$ where ${A \in
\PRtrig}$ and ${s \not\in \{ x,y,\TOP \} }$, as well as any of the form $y \succ o$, $o \succ y$, $x \succ o$, $o \succ x$.
\Cref{def:order:variable} is satisfied after this step because we have not removed
any of the orderings of that form, and all of them were satisfied before taking this step.
\Cref{def:order:individual,def:order:function} remain satisfied
because the relaxation step does not eliminate any ordering involving only \ab-terms except 
$y \succ o$, $o \succ y$, $x \succ o$, $o \succ x$, which are not
of the form given in this condition. 
For condition \cref{def:order:simplification}, we need to see if some of the eliminated orderings correspond to orderings of the form
$t[s_2]_p \succeq t[s_1]_p$ with $s_2 \succ s_1$. Observe that the only elements in $\PRtrig$ that could be of the form $t[s_1]_p$
for some $s_1$ such that $s_1 \succ s_2$ and $t[s_2]_p$ is a valid context term, in non-root contexts, are those when either $x$ or $y$ are replaced by 
a nominal (not those where $x$ is replaced by $y$). But even those orderings do not trigger this condition, because we also eliminate orderings of the form
$y \succ o$, $o \succ y$, $x \succ o$, $o \succ x$. Hence, none of the eliminated orderings would trigger the premise of this condition. 
In root contexts, elements in $\PRtrig$ of the form $A \equals \TOP$ can only be of the form $B(y)$, $S(y,y)$, or $S(y,o)$ or $S(o,y)$, so a similar argument applies
as every ordering of the form $y \succ o$ or $o \succ y$ is removed.
For condition \cref{def:order:subterm}, since we only eliminate orderings of the form $A \succ s$ where $A$ is in $\PRtrig$, $A$ contains no function symbols;
since in order to trigger this condition, $s$ must be a subterm of $A$, then it can only be $x$, $y$, or in $\AllIndiv$. However, it is precisely
these orderings that we do not remove, as we already discussed. 
\newpage
\section{Proof of soundness}

\soundness*

Remember that we say that a structure $\II$ satisfies a set of clauses $\mathcal{S}$ if and only if it satisfies every clause in the set. In particular, if $\SC{} = \emptyset$, then any $\II$ satisfies $\SC{}$. If $\SC{}$ is a set of clauses and $\tau$ a grounding on the universe of $\JJ$, then $\SC{} \tau$ is the set obtained by applying the homomorphism $\tau$ to every clause of $\SC{}$. Moreover, given a clause $C$, we use $\body{C}$ to represent the body of $C$.

\begin{proof}
Let $\Onto$ be an ontology, and let $\ContextStructure$ be a context structure $\langle \Contexts, \Edges, \SC{}, \core{}, \gtrdot, \succ \rangle$ be a context structure that is sound for $\Onto$. We next show that applying an inference rule from \cref{tab:oldrules} or \cref{tab:newrules} to $\ContextStructure$ using an arbitrary expansion strategy produces a sound context structure. In particular, we show that each derived clause is a context clause according to \cref{def:context-terms} and satisfies the conditions in \cref{def:sound-context-structure}. We make use of the following intermediate result:

\begin{lemma}\label{lem:sound-maximal-satisfaction}
Consider $n$ arbitrary clauses: 
\begin{align*} C_1 = \Gamma_1 \to \Delta_1 \vee L_1, \\ 
 C_2 = \Gamma_2 \to \Delta_2 \vee L_2, \\ 
\cdots \phantom{AAAAA}\\ 
C_n = \Gamma_n \to \Delta \vee L_n,
\end{align*}
where literals $L_i$ are atoms; consider the clause $$C=\bigwedge_{i=1}^n \Gamma_i \to \bigvee_{i=1}^n \Delta_i \vee \Delta.$$ For every set of additional nominals $N$, every $N$-compatible structure $\II$ such that \begin{itemize}
\item $N(C) \neq \emptyset$.
\item $\II \models N(C_i)$ for all $i$,
\end{itemize}
and every grounding $\tau$ such that \begin{itemize}
\item $\II \models \body{N(C) \tau}$,
\item $\II \nmodels N(\Delta_i)\tau$ for any $i$,
\end{itemize}
we have that $N(L_i)=L_i$ for every $i$ and $\II \models L_i \tau$ for each $i$.
\end{lemma}
Observe that the lemma implies that if there is at least one triple of $N$, $\II$, and $\tau$ verifying the conidtions of the lemma, then $N(L_i)=L_i$ for every $i$.

\begin{proof}

Since $N(C) \neq \emptyset$, we have that $N \left ( \bigwedge_{i=1}^n \Gamma_i \right ) = \bigwedge_{i=1}^n \Gamma_i$, and hence since $\II \models \body{N(C) \tau}$, we have $\II \models \left ( \bigwedge_{i=1}^n \Gamma_i \right ) \tau$. Thus, $\II \models \Gamma_i \tau$ for each $i$. Observe that $\Delta_i$ contains no inequality with elements of $N$, since otherwise $N(C)= \emptyset$ as that inequality would be in the head of $C$. Moreover, $L_i$ is not an inequality for any $i$, so we conclude $N(C_i) \neq \emptyset$ and hence $N(C_i)= \Gamma_i \to N(\Delta_i \vee L_i)$. But since $\II \models N(C_i)$ and $\II \models \Gamma_i \tau$ for each $i$, we then have $\II \models  N(\Delta_i \vee L_i) \tau$ for each $i$. However, we know $\II \nmodels N(\Delta_i) \tau$, and hence we conclude $\II \models N(L_i) \tau$, which implies $N(L_i) =L_i$. \end{proof}

Another helpful auxiliary tool in this proof will be an explicit list of every possible context and root-context literal, given next:
\newcounter{rows}

\begin{table}[H]

\begin{tabular}{c|c|c|c|c|c|c|c|c|c|c|c|c|c|}
\hline 
\rlabel{DL-bows}$\ab$-DL & \multicolumn{13}{c}{\multirow{2}{*}{$z_i \bowtie z_j$} $\quad$  \multirow{2}{*}{$z_i \bowtie f(x)$} $\quad$ \multirow{2}{*}{$z_i \bowtie x$} $\quad$ \multirow{2}{*}{$z_i \bowtie o$} $\quad$ \multirow{2}{*}{$x \bowtie f(x)$} $\quad$ \multirow{2}{*}{$f(x) \bowtie g(x)$} $\quad$ \multirow{2}{*}{$x \bowtie o$} $\quad$ \multirow{2}{*}{$o \bowtie o'$} $\quad$ \multirow{2}{*}{$f(x) \bowtie o$}} \\
literals \\ \hline
\rlabel{p-atoms}Context  & \multicolumn{13}{c}{\multirow{2}{*}{$B(y)$} $\,$ \multirow{2}{*}{$B(x)$} $\,$  \multirow{2}{*}{$B(f(x))$} $\,$  \multirow{2}{*}{$B(o)$} $\,$  \multirow{2}{*}{$S(x,y)$} $\,$  \multirow{2}{*}{$S(y,x)$} $\,$  \multirow{2}{*}{$S(x,f(x))$} $\,$  \multirow{2}{*}{$S(f(x),x)$} $\,$  \multirow{2}{*}{$S(x,x)$} $\,$ \multirow{2}{*}{$S(x,o)$} $\,$  \multirow{2}{*}{$S(o,x)$} $\,$  \multirow{2}{*}{$S(o,o')$}} \\
\pred-terms \\ \hline
\rlabel{p-bows}Other & \multicolumn{13}{c}{\multirow{2}{*}{$f(x) \bowtie g(x)$} $\quad$  \multirow{2}{*}{$f(x) \bowtie x$} $\quad$ \multirow{2}{*}{$f(x) \bowtie y$} $\quad$ \multirow{2}{*}{$f(x) \bowtie o$} $\quad$  \multirow{2}{*}{$x \bowtie y$} $\quad$  \multirow{2}{*}{$x \bowtie o$} $\quad$  \multirow{2}{*}{$y \bowtie o$}} \\
literals \\ \hline
\rlabel{r-atoms}$r$-Context  & \multicolumn{13}{c}{\multirow{2}{*}{$B(y)$} $\quad$ \multirow{2}{*}{$B(o)$} $\quad$ \multirow{2}{*}{$B(f(o))$} $\quad$ \multirow{2}{*}{$S(o,y)$} $\quad$ \multirow{2}{*}{$S(y,o)$} $\quad$  \multirow{2}{*}{$S(y,y)$} $\quad$ \multirow{2}{*}{$S(o,f(o))$} $\quad$ \multirow{2}{*}{$S(f(o),o)$} $\quad$ \multirow{2}{*}{$S(o',o)$} } \\
\pred-terms \\ \hline
\rlabel{r-bows}Other & \multicolumn{13}{c}{\multirow{2}{*}{$f(o) \bowtie g(o)$} $\quad$ \multirow{2}{*}{$f(o) \bowtie o$} $\quad$ \multirow{2}{*}{$f(o) \bowtie y$} $\quad$  \multirow{2}{*}{$f(o) \bowtie o'$} $\quad$ \multirow{2}{*}{$o \bowtie y$} $\quad$  \multirow{2}{*}{$o \bowtie o'$} }  \\
literals \\ \hline
\end{tabular}
\caption{Listing of context literals and root context literals.}
\label{tab:sound:context-literals}
\end{table}

Consider all inference rules from \cref{tab:oldrules,tab:newrules}. We assume the ontology is satisfiable, as otherwise the result follows trivially. We also let $\II$ be an arbitrary model of $\Onto$ and $(N,\JJ)$ the corresponding set of individuals and conservative extension of $\II$ verifying the conditions of \cref{def:sound-context-structure}  for $\ContextStructure$. For each possible application of a rule, we show how to generate $(N',\JJ')$ for $\ContextStructure'$ verifying the same conditions; in fact, we always choose $N' = N$ and $\JJ' = \JJ$. 

\medskip

(\ruleword{Core}) For each ${A \in \core{v}}$, $N(\core{v} \rightarrow A) = \core{v} \rightarrow A$ since cores do not contain named individuals or inequalities. Since $A \in \core{v}$, the clause is trivially satisfied by every structure, and in particular by $\JJ'$.  

(\ruleword{Hyper}) Observe that $ \JJ  \models \bigwedge_{i=1}^n A_i \rightarrow \Delta $ since $\JJ$ is a conservative extension of $\II$, and $\II$ is a model of $\Onto$ and the clause is an axiom of $\Onto$. Moreover, by the soundness of $\ContextStructure$, we have that for each $1 \leq i \leq n$, $\JJ \models N(\core{v} \wedge \Gamma_i \rightarrow \Delta_i \vee A_i \sigma)$. We will show that $\JJ$ satisfies 
\begin{equation}\label{eq:sound:hyper}
N(C) = N(\core{v} \wedge \bigwedge_{i=1}^n \Gamma_i \rightarrow \Delta\sigma \vee \bigvee_{i=1}^n \Delta_i).
\end{equation}
Suppose \cref{eq:sound:hyper} $\neq \emptyset$, as otherwise the result is trivial. We need to show that if $\JJ$ satisfies the body of \cref{eq:sound:hyper} for some grounding $\tau$, then it also satisfies the grounding of its head by $\tau$. So we assume $\JJ \models \body{N(C) \tau}$,  and we also assume $\JJ \not \models N(\Delta_i) \tau$ for every $i$, as otherwise the result follows immediately. We are therefore in the conditions of lemma \ref{lem:sound-maximal-satisfaction}, as $A_i \sigma \tau$ are always atoms. Thus, we have that  $\JJ \models A_i \sigma \tau$. But then, since $\JJ \models \bigwedge_{i=1}^n A_i \rightarrow \Delta $, we have $\JJ \models \Delta \sigma \tau$. The lemma also gives us $N(A_i \sigma)=A_i \sigma$; from this, we have that $\sigma$ does subsitute any variable of $\Delta$ (every variable in $\Delta$ must appear in some $A_i$ by definition of DL-clauses) by an element of $N$, so there are no occurrences of elements from $N$ in $\Delta \sigma$, so $N(\Delta \sigma)= \Delta \sigma$, and hence $\JJ' \models N(\Delta \sigma)$, which is what we wanted to show. 

To see that the clause generated is a context clause, consider the possible forms of body DL-atoms: $B_i(x)$, $S(z_i,x)$, or $S(x,z_i)$. For non-root contexts, since $\sigma(x)=x$, $z_i$ can be mapped only to $y$, $x$, $f(x)$, or $o$, so $\sigma(z_i)$ can only take values among these terms. By looking at the forms of DL-literals, and taking into account the possible values of $\sigma(z_i)$, we have that generated DL-literals can be of the form: $B(y)$, $B(f(x))$, $B(o)$, $B(x)$, $S(y,x)$, $S(f(x),x)$, $S(o,x)$, $S(x,y)$, $S(x,f(x))$, $S(x,o)$, $S(x,f(x))$,$S(f(x),x)$, or an equality or inequality between context \ab-terms. All these correspond to context literals. For the root context, we have $\sigma(x)=o$ for some $o \in \AllIndiv$, and $z_i$ can only be mapped to $y,f(o),o'$ for some $o' \in \AllIndiv$ (possibly equal to $o$), so the generated literals can be of the form $B(y)$, $B(f(o))$, $B(o)$, $S(y,o)$, $S(f(o),o)$, $S(o',o)$, $S(o,y)$, $S(o,f(o))$, $S(o,o')$, which are root context literals, or an \ab-equality obtained by substituting $x$ by $o$ and $z_i$ by $y$,$f(o)$, or $o'$, in \cref{row:DL-bows} of \cref{tab:sound:context-literals}, and it is fairly easy to check that all these are elements of \cref{row:r-bows} of \cref{tab:sound:context-literals}. 

(\ruleword{Eq}) We will show that $\JJ$ satisfies 
\begin{equation}\label{eq:sound:eq}
N(C) = N(\core{v} \wedge \Gamma_1 \wedge \Gamma_2 \to \Delta_1 \vee \Delta_2  \bigvee_{i=1}^n s_2[t_1]_p \bowtie s_1 ).
\end{equation}
In particular, we need to show that if $\JJ$ satisfies the body of $N(C)$ for some grounding $\tau$, then it also satisfies the grounding of its head by $\tau$. Suppose \cref{eq:sound:hyper} $\neq \emptyset$, as otherwise the result is trivial, so $N(\Gamma_1 \wedge \Gamma_2) = \Gamma_1 \wedge \Gamma_2$. So we assume $\JJ \models (\Gamma_1 \wedge \Gamma_2) \tau$,  and we also assume $\JJ \not \models N(\Delta_1 \vee \Delta_2) \tau$, as otherwise the result follows immediately.

By the soundness of $\ContextStructure$, we have that $\JJ \models N(\core{v} \wedge \Gamma_i \rightarrow \Delta_i \vee L_i)$ for $i \in \{1,2\}$; for $\Gamma_1 \to \Delta_1 \to s_1 \equals t_1$, $N(C) \neq \emptyset$ implies that $\Gamma_1$ contains no element of $N$ and $\Delta_1$ contains no inequalities with elements of $N$; since $L_1$ is an equality, we can safely conclude $N(C_1) \neq \emptyset$, and since $\JJ \models \Gamma_1 \tau$, we have $\JJ \models N(\Delta_1) \tau \vee (s_1 \equals t_1) \tau$. By our assumption that  $\JJ \not \models N(\Delta_1) \tau$, we obtain $\JJ \models (s_1 \equals t_1) \tau$ and that $N(s_1 \equals t_1) = (s_1 \equals t_1)$ so neither $s_1$ nor $t_1$ contain elements of $N$. 

For $\Gamma_2 \to \Delta_2 \to s_2 \bowtie t_2$, using an argument similar to the one above, and noting that if $s_2 \bowtie t_2$ is an inequality that contains some element of $N$, given that $s_1 \equals t_1$ contains no such elements, we have that $s_2[t_1]_p \bowtie t_2$ still contains an element of $N$, which contradicts that $N(C) \neq \emptyset$, we conclude that  $N(s_2[t_1]_p \bowtie t_2)$ $=$ $s_2[t_1]_p \bowtie t_2$, and $\JJ \models s_2[t_1]_p \bowtie t_2 \tau$ because $s_2[t_1]_p \bowtie t_2 \tau$ is a logical consequence of $s_2 \bowtie s_1 \tau $ and $t_2 \bowtie t_1 \tau$, which are both satisfied by $\JJ$. 

Now, for non-root contexts, one can see that applying this rule with $s_2 \bowtie t_2$ as any element of \cref{row:p-atoms} or \cref{row:p-bows} of \cref{tab:sound:context-literals} and $s_1 \bowtie t_1$, as any element \cref{row:p-bows} of \cref{tab:sound:context-literals} yields another element of \cref{row:p-atoms} or \cref{row:p-bows} of \cref{tab:sound:context-literals}. An analogous argument can be made for root contexts, using \cref{row:r-atoms} and \cref{row:r-bows} of \cref{tab:sound:context-literals}, and noting that rule \ruleword{Eq} never replaces an individual in an atom with function symbols. 

(\ruleword{Ineq}) Since we eliminate a literal which could not have been satisfied by any model, the result follows trivially from the soundness of $\ContextStructure$ for $\Onto$.

(\ruleword{Factor}) We will show that $\JJ$ satisfies 
\begin{equation}\label{eq:sound:factor}
N(C) = N(\core{v} \wedge \Gamma \to \Delta \vee t \nequals t' \vee s \equals t').
\end{equation}
In particular, we need to show that if $\JJ$ satisfies the body of $N(C)$ for some grounding $\tau$, then it also satisfies the grounding of its head by $\tau$. We suppose \cref{eq:sound:hyper} $\neq \emptyset$, as otherwise the result is trivial, so $N(\Gamma) = \Gamma$. Hence, assuming $\JJ \models \body{N(C)} \tau$ means $\JJ \models \Gamma \tau$. 

By the soundness of $\ContextStructure$, we have that $\JJ \models \core{v} \wedge \Gamma \to N(\Delta \vee s\equals t \vee s \equals t')$, and since $\JJ \models (\core{v} \wedge \Gamma) \tau$, we have $\JJ \models N(\Delta \vee s\equals t \vee s \equals t') \tau$. We assume $\JJ \nmodels N(\Delta \vee s \equals t') \tau$, as otherwise the result follows immediately. Thus, $\JJ \models (s \equals t) \tau$.

Observe that since $\JJ \nmodels (s \nequals t') \tau$, we have $\JJ \models (t \nequals t') \tau$. Furthermore, $N(t \nequals t')= t \nequals t'$, since we assumed $N(C) \neq \emptyset$. Thus, we verify $\JJ \models  N(t \nequals t') \tau$, and hence the result follows. 

Now, for non-root contexts, one can see that applying this rule with $s \equals t$ and $s \equals t'$ any two elements of this form in \cref{row:p-atoms} or \cref{row:p-bows} of \cref{tab:sound:context-literals} yields a $t \nequals t'$ which is either of the form $\TOP \nequals \TOP$, which is a context literal, or some element of row \cref{row:r-bows} of \cref{tab:sound:context-literals}. The argument for root atoms is analogous. 

(\ruleword{Elim}) The resulting context structure contains a subset of the clauses from $\ContextStructure$, so it is clearly sound for $\Onto$.

(\ruleword{Pred}) We have to prove:
\begin{equation}\label{eq:sound:pred}
\JJ \models N(C) \quad \mbox{ with } \quad N(C) = N(\core{u} \wedge \bigwedge_{i=1}^{n} \Gamma_i \wedge  \bigwedge_{i=1}^{m} C_i   \rightarrow \bigvee_{i=1}^n \Delta_i \, \vee  \bigvee_{i=1}^{k} L_i \sigma ).
\end{equation}
Once again, we assume $N(C) \neq \emptyset$, as otherwise this is trivial, so $N(\Gamma_i) = \Gamma_i$, and $N(C_i) = C_i$. Suppose $\JJ \models \body{C} \tau$ for some $\tau$. By soundness of $\ContextStructure$, we have $\JJ \models \core{u} \wedge \Gamma_i \to N(\Delta_i \vee A_i \sigma)$ for each $i$, so we have $\JJ \models N(\Delta_i \vee A_i \sigma) \tau$. Suppose $\JJ \nmodels N(\Delta_i) \tau$, as otherwise the result is trivial. Then $N(A_i \sigma) = A_i \sigma$, and also $\JJ \models A_i \sigma \tau$.

By soundness of $\ContextStructure$, we also have $\JJ \models \core{u} \models \core{v} \sigma$, so because $\JJ \models \core{u} \tau$, then $\JJ \models \core{v} \sigma \tau$. But then, again by soundness of $\ContextStructure$, we have $\JJ \models N(\core{v} \wedge A_i \wedge  \bigwedge_{i=1}^{m} C_i   \rightarrow \bigvee_{i=1}^{k} L_i)$, and we have that since $\core{v}$ does not contain elements of $N$, nor does $A_i$ since $N(A_i \sigma) = A_i \sigma$, and $N(C_i) = C_i$, and $L_i$ contains no inequality with elements of $N$, as otherwise we contradict $N(C) \neq \emptyset$, then $\JJ \models  \core{v} \wedge A_i \wedge  \bigwedge_{i=1}^{m} C_i   \rightarrow N(\bigvee_{i=1}^{k} L_i)$, so because $\JJ \models \core{v} \sigma \tau$, and $\JJ \models A_i \sigma \tau$ and $\JJ \models C_i \tau$ (because $C_i \sigma = C_i$, as $C_i$ is ground, then we have $\JJ \models N(\bigvee_{i=1}^{k} L_i) \sigma \tau$. We have already argued that $\sigma$ does not contain elements from $N$, and hence $N(\bigvee_{i=1}^{k} L_i) \sigma = \bigvee_{i=1}^{k} L_i \sigma$, which concludes the first part of our proof. 

For the second part, observe that by soundness of $\ContextStructure$, elements in $\Gamma_i$ are valid context atoms for a clause body, and the same happens for $C_i$, since they were already so in $v$, which is not a root context. A similar argument applies to elements in $\Delta_i$; for elements in $L_i$, observe that if they are ground, then $L_i \sigma$ are valid context literals in the head of a clause, since they already were so in $v$, which is not a root context; if $L_i$ are not ground, then they must belong to \PRtrig, and we can easily check that applying $\sigma$ to any element in \PRtrig yields a context literal.  

(\ruleword{Succ}) Every tautology added by this rule is trivially satisfied, so we only need to prove $$\JJ \models \core{u} \to \core{v} \sigma$$.
Observe that for each element $L_i$ in $\core{v}$, there is a clause $\top \to L_i \sigma$ in $\SC{u}$, with $\sigma$ defined as in the rule. Thus, by soundness of $\ContextStructure$ and the fact that cores do not contain elements in $N$, we directly obtain $\JJ \models \core{u} \to \core{v} \sigma$. For the second clause, observe that $A' \sigma$ does not contain elements of $N$. Since the added clauses contain exclusively elements in \SUtrig, and these are valid context atoms, the generated clauses are trivially context clauses. 

(\ruleword{Join}) We have to prove that 
\begin{equation}\label{eq:sound:join}
\JJ \models N(C) \quad \mbox{ with } \quad N(C) = N(\core{v} \wedge \Gamma \wedge \Gamma' \to \Delta \vee \Delta' \vee \Delta'' ).
\end{equation}

We assume $N(C) \neq \emptyset$, otherwise the result is trivially achieved. Observe that by soundness of $\ContextStructure$, we have $$\JJ \models \core{v} \wedge A \wedge \Gamma \to N(\Delta),$$ and $\JJ \models \core{v} \Gamma' \to N(\Delta' \vee \Delta'' \vee A)$; so if for some $\tau$ we have $\JJ \models \Gamma' \tau \wedge \Gamma \tau \wedge \core{v} \tau$, and we assume $\JJ \nmodels N(\Delta' \vee \Delta'') \tau$, as otherwise the result is trivially satisfied, we have $\JJ \models A \tau$, but then, we obtain $\JJ \models N(\Delta) \tau$, which is what we wanted to prove.

For the second type of the application of this rule, we have that  by soundness of $\ContextStructure$, we have $\JJ \models \core{v} \to N(\Delta' \vee A)$ as well as $\JJ \models \core{v}  \to N(\Delta'' \vee x \equals o)$; so if for some $\tau$ we have $\JJ \models \Gamma \tau \wedge \core{v} \tau$, and we assume $\JJ \nmodels N(\Delta' \vee \Delta'') \tau$, as otherwise the result follows trivially, then we have $\JJ \models N(A') \tau$ and $\JJ \models N(x \equals o) \tau$; we thus conclude that $N(A')= A'$, and therefore $A'$ does not contain any occurrences of elements in $N$; similarly, $N(x \equals o) = x \equals o$, so we have $\JJ \models A' \{x \mapsto o\} \tau$, as this is a logical conclusion of the previous two literals; but $A' \{ x \mapsto o\} = A$, so $\JJ \models A \tau$. Thus, since by soundness of $\ContextStructure$, we have $$\JJ \models \core{v} \wedge A \wedge \Gamma \to N(\Delta),$$ and we are assuming $\JJ \models \Gamma \tau \wedge \core{v} \tau$, we conclude $\JJ \models N(\Delta) \tau$, as we wanted to prove.

The literals in the added clause exist in the same positions in other clauses in $\ContextStructure$, so they are clearly context literals, and the added clause is therefore a context clause.

($r$-\ruleword{Succ}) The first part of this rule is trivially satisfied, as we only add tautologies to the root context. Moreover, these tautologies contain elements in $\rootSUtrig$, which are all valid context atoms, and therefore the resulting tautology is a context clause. Moreover, since the core of the root context is empty, $\core{u} \to \core{v_r} \sigma$ is trivially satisfied (remember that here this represents an implication from a conjunction to an (empty) conjunction).

($r$-\ruleword{Pred}) We have to prove:
\begin{equation}\label{eq:sound:rpred}
\JJ \models N(C) \quad \mbox{ with } \quad N(C) = N(\core{u} \wedge \bigwedge_{i=1}^{n} \Gamma_i \wedge  \bigwedge_{i=1}^{m} C_i   \rightarrow \bigvee_{i=1}^n \Delta_i \, \vee  \bigvee_{i=1}^{k} L_i \sigma ).
\end{equation}
Once again, we assume $N(C) \neq \emptyset$, as otherwise this is trivial, so $N(\Gamma_i) = \Gamma_i$, and $N(C_i) = C_i$. Suppose $\JJ \models \body{C} \tau$ for some $\tau$. By soundness of $\ContextStructure$, we have $\JJ \models \core{u} \wedge \Gamma_i \to N(\Delta_i \vee A_i \sigma)$ for each $i$, so we have $\JJ \models N(\Delta_i \vee A_i \sigma) \tau$. Suppose $\JJ \nmodels N(\Delta_i) \tau$, as otherwise the result is trivial. Then $N(A_i \sigma) = A_i \sigma$, and also $\JJ \models A_i \sigma \tau$.

By soundness of $\ContextStructure$, we also have $\JJ \models \core{u} \models \core{v} \sigma$, so because $\JJ \models \core{u} \tau$, then $\JJ \models \core{v} \sigma \tau$. But then, again by soundness of $\ContextStructure$, we have $\JJ \models N(\core{v} \wedge A_i \wedge  \bigwedge_{i=1}^{m} C_i   \rightarrow \bigvee_{i=1}^{k} L_i)$, and we have that since $\core{v}$ does not contain elements of $N$, nor does $A_i$ since $N(A_i \sigma) = A_i \sigma$, and $N(C_i) = C_i$, and $L_i$ contains no inequality with elements of $N$, as otherwise we contradict $N(C) \neq \emptyset$, then $\JJ \models  \core{v} \wedge A_i \wedge  \bigwedge_{i=1}^{m} C_i   \rightarrow N(\bigvee_{i=1}^{k} L_i)$, so because $\JJ \models \core{v} \sigma \tau$, and $\JJ \models A_i \sigma \tau$ and $\JJ \models C_i \tau$ (because $C_i \sigma = C_i$, as $C_i$ is ground, then we have $\JJ \models N(\bigvee_{i=1}^{k} L_i) \sigma \tau$. We have already argued that $\sigma$ does not contain elements from $N$, and hence $N(\bigvee_{i=1}^{k} L_i) \sigma = \bigvee_{i=1}^{k} L_i \sigma$, which concludes the first part of our proof. 

For the second part, observe that by soundness of $\ContextStructure$, elements in $\Gamma_i$ are valid context atoms for a clause body, and the same happens for $C_i$, since they were already so in $v$, which is not a root context. A similar argument applies to elements in $\Delta_i$; for elements in $L_i$, observe that if they are ground, then $L_i \sigma$ are valid root context literals in the head of a clause belonging to $\rootPRtrig$, and by the form of the elements in tis set, it is easy to check that applying $\sigma$ to any element in $\rootPRtrig$ yields a context literal.

(\ruleword{Nom}) We have $\JJ \models C$, where $C=\bigwedge_{i=1}^n A_i \rightarrow \bigvee_{i=1}^{m} L_i \vee \bigvee_{i=m+1}^{k} L_i$, and as this clause is in $\Onto$, $C=N(C)$. By soundness of $\ContextStructure$, we have that $\JJ \models N(\Gamma_i \to \Delta_i \vee A_i \sigma)$. Let each of these clauses be called $C_i$.
Suppose $N(C_i) = \emptyset$. Then, the result follows trivially because $A_i \sigma$ cannot be an inequality, as $A_i$ appears in the body of an axiom. Thus, assume $N(C_i) \neq  \emptyset$ and suppose that for some $\tau$, $\JJ \models \Gamma_i \tau$ for each of these, and it does not model $N(\Delta_i) \tau$. We then have $\JJ \models N(A_i \sigma) \tau$; observe that $A_i \sigma$ cannot contain elements from $N$, so $N(A_i \sigma) = A_i \sigma$, and $\JJ \models A_i \sigma \tau$. Then, since $\JJ$ models the axiom given above, we have $\JJ \models \bigvee_{i=1}^k L_i \sigma \tau$. As the substitution does not contain elements from $N$, we can see that $N(L_i \sigma) = L_i \sigma$ for each $i$, so we assume $\JJ \not \models N(L_i \sigma) \tau$ for $1 \leq i \leq m$, as otherwise the result follows trivially.  that $N(L_i \sigma)$ does not contain elements from $N$. 

Observe that $S(o'_{\rho},y) = A_i$, for some $i$, so $\JJ \models S(o'_{\rho}, y \tau)$. Now, suppose that some new individual $o'_{\rho \cdot S^i}$ introduced by this rule is in $N$. Then, we have that there is no element $u$ in the domain of $\JJ$ such that $\JJ \models S(o'_{\rho},u)$. But this is contradicted by the fact that $\JJ \models S(o'_{\rho}, y \tau)$. So every such individual is in $N$. Then, for every such element $o'_{\rho \cdot S^j}$, we have that $\JJ \models S^i(o'_{\rho},o'_{\rho \cdot S^j})$. Moreover, we assume that they are different, since if two of these are equal, then the result is trivially verified due to the conditions for $\JJ$ and $N$. Observe that this assumption guarantees that the result is true for every added clause of the form $\top \rightarrow o'_{\rho \cdot S^{\ell_2}} \nequals o'_{\rho \cdot S^{\ell_1}} \vee o'_{\rho \cdot S^{\ell}} \equals o'_{\rho \cdot S^1}$ and the corresponding $\Gamma \to \Delta \vee o'_{\rho \cdot S^{\ell_2}} \nequals o'_{\rho \cdot S^{\ell_1}} \vee \bigvee_{j=1}^{\ell_1} y \equals o'_{\rho \cdot S^j}$. 

Consider $\Delta'$ the part of $\bigvee L_i$ which does not contain elements of $y$, and let $\Delta$ be $\bigvee L_i \backslash \Delta'$. We have that $\JJ$ does not model any such element. Hence, we have that $\JJ$ verifies $\bigwedge A_i \{x \mapsto o\} \sigma \tau \to \Delta \{x \mapsto o\}$, and every element in $\Delta$ is of the form $z_i \equals z_j$ or $z_i \equals n_j$ for a constant $n_j$. It is easy to see that this restriction enforces that there can only be $K'$ different elements $u$ in $\JJ$ such that $\JJ \models S(o'_{\rho},u)$, where $K' = \max (K, K'')$, with $K$ defined as in the rule \ruleword{Nom} and $K''$ the number of distinct terms of the form $f_i(x)$ or $o_i$ in $\Delta$. Hence, since we have at least $K'$ different nominals of the form $o'_{\rho \cdot S^j}$, we have that $\JJ \models y \tau \equals o'_{\rho \cdot S^j}$ for some $j$, and therefore $\JJ \models N(\bigvee_{i=1}^K y \equals o'_{\rho \cdot S^j})$, which is what we wanted to prove. 

The fact that the added clause is a valid context clause follows from the fact that for any element $A$ in $\rootPRtrig$ we have $A \sigma$ is a valid context literal, as it can readily be checked. 

\end{proof}
\newpage
\section{Proof of completeness}

\completeness*

\subsection{Rewrite systems}

In the proof of \cref{theorem:completeness} we construct a model of an
ontology, which, as is common in equational theorem proving, we represent using
a ground \emph{rewrite system}. To make our proof self-contained, we next
recapitulate the definitions of rewrite systems, following the presentation and
the terminology introduced by
\citeA{BaN98}. For simplicity,
we adapt all standard definitions to ground rewrite systems only.

A (ground) \emph{rewrite system} $R$ is a binary relation on the Herbrand
universe $\HU$. Each pair ${(s,t) \in {R}}$ is called a \emph{rewrite rule} and
is commonly written as ${s \RuleSymbol t}$. The \emph{rewrite relation}
$\RwRelSymbol{R}$ for $R$ is the smallest binary relation on $\HU$ such that,
for all terms ${s_1, s_2, t \in \HU}$ and each (not necessarily proper)
position $p$ in $t$, if $\Rule{s_1}{R}{s_2}$, then
$\RwRel{t[s_1]_p}{R}{t[s_2]_p}$. Moreover,
$\RwRelRefTransSymbol{R}$ is the reflexive--transitive closure of
$\RwRelSymbol{R}$, and $\CongruenceSymbol{R}$ is the
reflexive--symmetric--transitive closure of $\RwRelSymbol{R}$. A term $s$ is
\emph{irreducible by} $R$ if no term $t$ exists such that $\RwRel{s}{R}{t}$;
and a literal, clause, or substitution $\alpha$ is \emph{irreducible by} $R$ if
each term occurring in $\alpha$ is irreducible by $R$. Moreover, term $t$ is a
\emph{normal form} of $s$ w.r.t.\ $R$ if ${\Congruence{s}{R}{t}}$ and $t$ is
irreducible by $R$. We consider the following properties of rewrite systems.
\begin{itemize}
    \item $R$ is \emph{terminating} if no infinite sequence ${s_1, s_2, \dots}$
    of terms exists such that, for each~$i$, we have
    ${\RwRel{s_i}{R}{s_{i+1}}}$.

    \item $R$ is \emph{left-reduced} if, for each ${\Rule{s}{R}{t}}$, the term
    $s$ is irreducible by $R \backslash \{s \RuleSymbol t \}$.

    \item $R$ is \emph{Church-Rosser} if, for all terms $t_1$ and $t_2$ such
    that ${\Congruence{t_1}{R}{t_2}}$, a term $z$ exists such that
    ${\RwRelRefTrans{t_1}{R}{z}}$ and ${\RwRelRefTrans{t_2}{R}{z}}$.
\end{itemize}
If rewrite system $R$ is terminating and left-reduced, then $R$ is
Church-Rosser \cite[Theorem~2.1.5 and Exercise 6.7]{BaN98}. If $R$ is
Church-Rosser, then each term $s$ has a unique normal form $t$ such that ${s
\RwRelRefTransSymbol{R} t}$ holds. The \emph{Herbrand equality interpretation
induced by} a rewrite system $R$ is the set $R^*$ such that, for all ${s,t
\in \HU}$, we have $s \equals t \in R^*$ iff $\Congruence{s}{R}{t}$.

Term orders can be used to prove termination of rewrite systems. A term order
$\succ$ on ground terms (i.e., on $\HU$) is a \emph{simplification order} if
the following conditions hold:
\begin{enumerate}
    \item for all ground terms $s_1$, $s_2$, and $t$, and each position $p$ in
    $t$, we have that ${s_1 \succ s_2}$ implies $t[s_1]_p \succ
    t[s_2]_p$; and

    \item for each term $s$ and each proper position $p$ in $s$, we have $s
    \succ s|_p$.
\end{enumerate}
Given a rewrite system $R$, if a simplification order $\succ$ exists such that
${s \RuleSymbol t \in R}$ implies ${s \succ t}$, then $R$ is terminating
\cite[Theorems~5.2.3 and~5.4.8]{BaN98}, and, for all ground terms $s$
and $t$, we have that ${s \RwRelSymbol{R} t}$ implies ${s \succ t}$.

\begin{lemma}
\label{lemma:rewrite:bigger-literal}
Let $R$ be a Church-Rosser rewrite system, $>$ an simplification order over a set of terms, and $l, r$ terms from the set with $l>r$ such that $\model{R}{} \nmodels l \equals r$. Consider a rewrite system $R \cup \{ k \Rightarrow s \}$, which is also Church-Rosser and where $k >l$. Then $\model{(R \cup \{ k \Rightarrow s \})}{} \nmodels l \equals r$.
\end{lemma}

\begin{proof}
Since $R$ is Church-Rosser, let $l'$ and $r'$ be the unique (and distinct) normal forms of $l$ and $r$ with respect to $R$. We have $\model{R} \nmodels l' \equals r'$. Now, since $k >l$ and $l>r$, we have $k > l'$ and $k>r'$, and therefore $k$ is neither a subterm of $l'$ nor $r'$. Suppose $\model{(R \cup k \Rightarrow s)}{} \models l' \equals r'$. Since $k$ is neither a subterm of $l'$ or $r'$, we have that $l'$ and $r'$ are irreducible by $R \cup k \Rightarrow s$.  Since the system $R$ is left-reduced, it can be represented by a set of trees, where $l'$ and $r'$ are roots. But since $\model{(R \cup k \Rightarrow s)}{} \models l' \equals r'$, the representation of $R \cup k \Rightarrow s$ must include a connection between trees rooted in $l'$ and $r'$, and this implies that $l'$ and $r'$ cannot both be roots, which contradicts the claim that $l'$ and $r'$ are still normal forms of $l$ and $r$ in $R$. Hence, $\model{R \cup k \Rightarrow s}{} \nmodels l' \equals r'$, so $\model{R}{} \nmodels l \equals r$.
\end{proof}

\subsection{Completeness conditions}

We fix an ontology $\Onto$, a saturated context structure $\ContextStructure = \langle \Contexts, \Edges, \SC{}, \core{}, \gtrdot, \succ \rangle$, a context $q \in \Contexts$, and a query clause $\InputQueryLHS \rightarrow \InputQueryRHS$ where $\core{q} = \InputQueryLHS$. The proof strategy is as follows: we assume conditions \ref{theorem:completeness:RHS} and \ref{theorem:completeness:LHS} and the negation of the conclusion i.e. $\InputQueryLHS \rightarrow \InputQueryRHS \nsbin \SC{q}$. With this, we show that $\Onto \nmodels \InputQueryLHS \rightarrow \InputQueryRHS$. This proves that if $\Onto \models \InputQueryLHS \rightarrow \InputQueryRHS$, which corresponds to \cref{theorem:completeness:query}, one of the three assumptions must be false. This implies that when \cref{theorem:completeness:RHS,theorem:completeness:LHS} are also verified, $\InputQueryLHS \rightarrow \InputQueryRHS \nsbin \SC{q}$ must be false, and hence the theorem is verified.

We define a new constant $c$. For each term $t$, if $t$ is of the form $f(s)$ for some term $s$ and $f \in \Sigma_f$, then $s$ is the \emph{predecessor} of $t$. Moreover, for each $f \in \Sigma_f$, we have that $f(t)$ is a \emph{successor} of $t$. We also define the \ab-neighbourhood of each term $t$ as the set which contains each successor of $t$, the predecessor of $t'$ of $t$, if it exists. Then, the \pred-neighbourhood of each $t$ is the set of terms that can be obtained by grounding context atoms using the substitution $\{x \mapsto t, y \mapsto t'\}$. this substitution is denominated $\Grounding{t}$. If $t$ is a constant, we define $\Grounding{t}$ as $\{x \mapsto t\}$. Finally, we define also:
$$ \textsf{Su}_t = \{ A \Grounding{t} \;|\; A \in \textsf{Su}(\mathcal{O}) \mbox{ and } A \Grounding{t} \mbox{ is ground } \} $$
$$ \textsf{Pr}_t = \{ A \Grounding{t} \;|\; A \in \textsf{Pr}(\mathcal{O}) \mbox{ and } A \Grounding{t} \mbox{ is ground }\} $$
$$ \textsf{Ref}_t = \{ S(t,t) \;|\; S \mbox{ is a binary atom } \} $$
$$ \textsf{Nom}_t = \{ S(t,o) \;|\; S \mbox{ is a binary atom, } o \in AllIndiv \} \cup \{ S(o,t)  \;|\; S \mbox{ is a binary atom, } o \in AllIndiv \}  $$

We define $\Omega$ as the set of all ground atoms with ground terms exclusively in $\AllIndiv$. We define $\Gamma_o = \Omega \cap \model{R}{c}$ and $\Delta_o = \Omega \backslash \model{R}{c} $

\subsection{Construction of a model fragment}

Suppose we have a fixed term $t$, a context $v$ such that if $t \neq c$, $v \neq v_r$, a conjunction of atoms $\Required{t}$, and a disjunction $\Forbidden{t}$ of literals, whose members are always in the  neighbourhood of $t$. We define the substitution $\sigma_t$ as $\{x \mapsto t, y \mapsto t'\}$, if $t'$ exists, and $\{x \mapsto t\}$ otherwise. Consider the following set:

$$N_t = \{ \Gamma \sigma_t \to \Delta \sigma_t \mid \Gamma \to \Delta \in \SC{v}, \mbox{ both } \Gamma \sigma_t \mbox{ and } \Delta \sigma_t \mbox{ are ground, and } \Gamma \sigma_t \subseteq \Gamma_t \}$$

We will now construct a rewrite system $\Rsystem{t}$ inducing a model fragment $\Rmodel{t}{}$ which we will include in the final, whole model. 

\subsubsection{Induction conditions}

Since the construction of the whole model model inductive, we assume that the following properties hold for the parameters introduced at the beginning of this section. Later we will prove inductively that these conditions always hold whenever we need to build the model fragment $\Rmodel{t}{}$
\begin{enumerateConditions}[label={L\arabic*.},ref={L\arabic*},leftmargin=2.5em]
\label{cond:fragment:induction}
	\item \label{cond:fragment:induction:compatibility} ${\Required{t} \rightarrow \Forbidden{t} \nsbin \GroundS{t}}$.
	\item \label{cond:fragment:induction:nocollapse} $t \equals o \in \Forbidden{t}$ for every $o \in \AllIndiv$ with $t >_t o$
	. Moreover, if $t'$ exists, then
	 $ \left ( \{t' \equals o \;|\; o \in \AllIndiv \} \right  \backslash \Gamma_o ) \subseteq \Forbidden{t}$, and $t \approx t'\in \Forbidden{t} $.
	\item \label{cond:fragment:induction:required} For each $A \in  \Required{t}$, we have $\Required{t} \rightarrow A \sbin \GroundS{t}$.
	\item \label{cond:fragment:induction:individuals} If $t >_t c$, then $\Required{t} \cap \Omega = \Gamma_o$, and $\Forbidden{t} \cap \Omega = \Delta_o$.
	\item \label{cond:fragment:induction:consistency} For every $A \in \Required{t}$, and every $p$ with $A|_p = o \in \AllIndiv$, and there is $o \equals o' \in \Gamma_o$ with $o > o'$, then $A[o']_p \in \Required{t}$.
	\item \label{cond:fragment:induction:all-lemas} Every lemma proved in this section is verified for $R_c$ if $t \neq c$, and for $R_{t'}$ if $t'$ exists.  
	\item \label{cond:fragment:induction:deltac}$\Delta_c = \Delta_Q \sigma_c$.
\end{enumerateConditions}

\subsubsection{Grounding of the context order}

The construction of the model fragment requires an ordering between terms obtained by grounding clauses in the context structure by $\Grounding{t}$. We consider a strict, simplification order $>_t$ between terms such that the following properties are verified:
\begin{enumerateConditions}[label={O\arabic*.},ref={O\arabic*},leftmargin=2.5em]
	\item \label{cond:fragment:groundorder:compatibility} $s_1 \succ_v s_2$ implies $s_1 \Grounding{t} >_t s_2 \Grounding{t}$.
	\item \label{cond:fragment:groundorder:forbidden} $s_1 \Grounding{t} \approx \textsf{true} \in \Delta_t$ and $s_2 \Grounding{t} \notin  \{t,t',\textsf{true}\} \cup \AllIndiv$ and $s_2 \Grounding{t} \approx \textsf{true} \notin \Delta_t$ imply $s_2 \Grounding{t} >_t s_1 \Grounding{t}$.
\end{enumerateConditions}

To see that order $\succ_{v}$ on (nonground) context terms can be grounded in a way that is compatible with these definitions, note that any strict, non-total order $\succ$ over a set $S$ can be used to generate a total order $>$ on $S$ in a way which guarantees that for any $a,b \in S$ such that $a \succ b$, we have $a > b$. Since each context $\ab$-term with a variable can only be mapped to a single ground $\ab$-term in the neighbourhood of $t$, we can define a strict, non-total order $\succ_t$ between ground terms so that $s_1 \Grounding{t} \succ s_2 \Grounding{t}$ if and only if $s_1 \succ_v s_2$, and then totalise it as described to ensure \cref{cond:fragment:groundorder:compatibility}; due to \cref{def:order:simplification,def:order:subterm} of \cref{def:order}, the properties of context orders, this also guarantees that it is a simplification order. However, in order to ensure that this order also satisfies \cref{cond:fragment:groundorder:forbidden}, it suffices to ensure that ground terms $s$ with $s \equals \TOP \in \Forbidden{t}$ are the smallest in $>_{t}$ after $\{t,t',\TOP\} \cup \AllIndiv$. But because of the way $>_{t}$ is obtained, this is only possible if for every $s_1$ with $s_1 \equals \TOP \in \Forbidden{t}$, we have that $s_1 \not \succ_t s_2$ for any ground term $s_2 \notin \{t,t',\TOP\}$ such that $s_2 \equals \TOP \notin \Forbidden{t}$. By definition of $\succ_t$, this will be the case if and only if for every context term $s_1$ with $s_1 \Grounding{t} \equals \TOP \in \Forbidden{t}$, we have $s_1 \not \succ_{v} s_2$ for any context term $s_2$ with $s_2 \Grounding{t} \notin \{t,t',\TOP\} \cup \AllIndiv$ and $s_2 \Grounding{t} \equals \TOP \notin \Forbidden{t}$. We therefore require that for any such $s_1$, we must have $s_1 \not \succ_{v} s_2$ for any context term $s_2 \notin \{x,y,\TOP\}\cup \AllIndiv$ with $s_2 \Grounding{t} \equals \TOP \notin \Forbidden{t}$. 

Now for ${t = c}$ with $c$ the distinguished constant introduced at the
beginning of the section, we have ${v = q}$ so by \cref{cond:fragment:induction:deltac}, if $s_1$ is a context term such that $s_1 \Grounding{c} \equals \TOP \in \Forbidden{c}$, we have $s_1 \equals \TOP \in \InputQueryRHS$, and therefore for any $s_2 \notin \{x,y,\TOP\} \cup \AllIndiv$ such that $s_2 \equals \TOP \notin \InputQueryRHS$, we have $s_1 \not \succ_q s_2$ by \cref{theorem:completeness:RHS} of
\cref{theorem:completeness}.
For ${t \neq c}$, observe that by definition of $\Forbidden{t}$, if $s_1$ is a context term such that $s_1 \Grounding{t} \equals \TOP \in \Forbidden{t}$, we have $s_1 \in \PRtrig$. But then, \cref{def:order:forbidden} of \cref{def:order} ensures that for any $s_2 \notin \{x,y,\TOP\} \cup \AllIndiv$ we have $s_1 \not \succ_{v} s_2$.

\subsubsection{Construction of the rewrite system $R_t$}

Let us write the clauses in $N_t$ as $\{C^1,\dots,C^n\}$. Observe that since the body of each such clause is in $\Required{t}$, the head of any such clause cannot be empty, as that would violate \cref{cond:fragment:induction:compatibility}. We represent each such clause as $\Gamma^i \rightarrow \Delta^i \vee L^i$, where $L^i >_t \Delta^i$ (clause heads have no duplicate literals, as they are sets). We also assume that the sequence is ordered in such a way that if $j>i$ then $C^j > C^i$. With this, we define a sequence of monotonically growing rewrite systems $\{R^0_t,\dots,R^n_t\}$, defined inductively as follows:
\begin{itemize}
\item $R^0_t := \emptyset$
\item $R^i_t = R_t^{i-1} \cup \{ l^i \Rightarrow r^i\}$ if $L^i$ is of the form $l^i \approx r^i$ such that:
	\begin{enumerateConditions}[label={R\arabic*.},ref={R\arabic*},leftmargin=2.5em]
		\item \label{cond:fragment:rsystem:notmodels} $(R_t^{i-1})^* \not \models \Delta^i \vee L^i$,
		\item \label{cond:fragment:rsystem:order} $l^i >_t r^i$,
		\item \label{cond:fragment:rsystem:irreducible} $l^i$ is irreducible by $R^{i-1}$, and
		\item \label{cond:fragment:rsystem:noduplicates} $(R^{i-1})^* \not \models s \approx r^i $ for each $l^i \approx s \in \Delta^i$.
	\end{enumerateConditions}
\item $R^i = R^{i-1}_t$ in all other cases.
\end{itemize}

Let $R_t = R_t^n$. If a clause verifies these \cref{cond:fragment:rsystem:notmodels,cond:fragment:rsystem:order,cond:fragment:rsystem:irreducible,cond:fragment:rsystem:noduplicates}, we call it a $C^i$ a \textit{generative} clause, and $\{ l^i \Rightarrow r^i\}$ is the \textit{generated} rule in $R_t$. Before moving on, we present and prove some properties of rewrite systems built in this way.

\begin{lemma}
\label{lemma:fragment:operational:smallerclause}
Let $C$ be a clause $\Gamma \to \Delta$ such that $C \sbin N_t$. If there is some $1 \leq i \leq n$ such that $C^i >_t C$ and $\Rmodel{t} \models \Delta^j \vee L^j$ for each $1 \leq j \leq i-1$, then $ \Rmodel{t} \models \Delta$. 
\end{lemma}
\begin{proof}
Suppose, to show a contradiction, that $ \Rmodel{t} \nmodels \Delta$. This means that even though $C \sbin N_t$, \cref{cond:redundancy:equality} cannot be verified, so there is $1 \leq j \leq n$ such that $\Gamma^j \subseteq \Gamma$ and $\Delta^j \vee L^j \subseteq \Delta$. Observe that this means that $C \geq_t C^j$, so if $j \geq i$, we obtain $C \geq_t C^j \geq_t C^i >_t C$, which is a contradiction. Thus, $j < i$ and hence $\Rmodel{t} \models \Delta^j \vee L^j$. Since $\Delta^j \vee L^j \subseteq \Delta$, this implies $ \Rmodel{t} \models \Delta$.
\end{proof}

\begin{lemma}
\label{lemma:fragment:operational:redundancy}
For any clause $\Gamma \rightarrow \Delta \sbin \SC{v}$ such that $\Gamma \sigma_t \subseteq \Required{t}$ and $\Gamma \sigma_t \rightarrow \Delta \sigma_t$ is ground, we have that $\Gamma \sigma_t \rightarrow \Delta \sigma_t \sbin N_t$. 
\end{lemma}

\begin{proof}
If clause $\Gamma \rightarrow \Delta \sbin \SC{v}$ then one of the following three cases occurs:\begin{itemize}
\item There is an equality $l \equals l \in \Delta$, so there is a literal $l \sigma_t  \approx l\ \sigma_t$ in $\Delta \sigma_t$, and hence by definition of redundancy, $\Gamma \sigma_t \rightarrow \Delta \sigma_t \sbin N_t$.
\item The literals $l \equals r$ and $l \nequals r$ are in $\Delta$, so we have that literals  $l \sigma_t \equals r \sigma_t$ and $l \sigma_t \nequals r \sigma_t$ are in $\Delta \sigma_t$, so by definition of redundancy, $\Gamma \sigma_t \rightarrow \Delta \sigma_t \sbin N_t$.
\item There exist $\Gamma' \subseteq \Gamma$ and $\Delta' \subseteq \Delta$ such that $\Gamma' \rightarrow \Delta' \in U$. Since $\Gamma' \sigma_t \rightarrow \Delta' \sigma_t$ is ground and $\Gamma' \sigma_t \subseteq \Required{t}$, by definition of $N_t$ we have that We write $\Gamma' \sigma_t \rightarrow \Delta' \sigma_t \in N_t$ so $\Gamma \sigma_t \rightarrow \Delta \sigma_t \sbin N_t$. 
\end{itemize}
\end{proof}

\begin{lemma}
\label{lemma:fragment:operational:lneql}
If there is a clause $\Gamma' \rightarrow \Delta' \vee l' \nequals l' \in N_t$, then clause $\Gamma' \rightarrow \Delta' \sbin N_t$.
\end{lemma}
\begin{proof}
Suppose there is a clause $\Gamma' \rightarrow \Delta' \vee l' \nequals l' \in N_t$. By definition of $N_t$ we have that there is a clause $\Gamma \rightarrow \Delta \vee l_1 \nequals l_2 \in \SC{v}$ such that $$\Gamma \sigma_t = \Gamma' \subseteq \Gamma_t, \quad \quad \Delta \sigma_t = \Delta', \quad \quad l_1 \sigma_t = l' \quad \quad l_2 \sigma_t = l'.$$
We consider two options:
\begin{itemize}
\item $l_1 \neq l_2$. Then, since $l_1 \sigma_t = l'$ and $l_2 \sigma_t = l'$, by definition of $\sigma_t$ the only possibility is $l_1 = x$ and $l_2 = \nin$, since literals of the form $y \equals \nin$ and $y \nequals \nin$ are forbidden. However, if $x \sigma_t = \nin$, this means $t=o$, and hence $\core{v}=(x \equals \nin)$, and since rule \ruleword{Core} is not applicable, $\top \rightarrow x \equals \nin \in \SC{v}$. But then, since rule \ruleword{Eq} is not applicable, clause $\Gamma \rightarrow \Delta \vee o \nequals o \sbin \SC{v}$, and because rule \ruleword{Ineq} is not applicable, then $\Gamma \rightarrow \Delta \sbin \SC{v}$. By \cref{lemma:fragment:operational:redundancy}, this implies $\Gamma' \rightarrow \Delta' \sbin N_t$.
\item $l_1 = l_2$. Then, since rule \ruleword{Ineq} is not applicable, $\Gamma \rightarrow \Delta \sbin \SC{v}$, which by \cref{lemma:fragment:operational:redundancy}, implies $\Gamma' \rightarrow \Delta' \sbin N_t$.  
\end{itemize}
\end{proof}

\subsubsection{The properties of the model fragment}

We now proceed to prove five main properties of each model fragment, namely: \begin{enumerate}
\item \textbf{Admissibility}: The rewrite system is Church-Rosser.
\item \textbf{Canonicity}: every rewrite rule in the system is necessary and maximal.
\item \textbf{Non-triviality}: the central term $t$ and its predecessor $t'$ are irreducible in the model fragment.
\item \textbf{Satisfaction}: The model fragment satisfies the relevant derived clauses. 
\item \textbf{Compatibility}: The model fragment verifies the restrictions that make it compatible with other model fragments.
\end{enumerate}
These properties guarantee that model fragments can be successfully combined into a model. 

{\textsc{Admissibility}}
$\phantom{AAA}$ \\
\begin{lemma}
\label{lemma:fragment:properties:church-rosser}
The rewrite system $R_t$ is Church-Rosser
\end{lemma}
\begin{proof} First we prove that $\Rsystem{t}$ is left-reduced. If $l^i \Rightarrow r^i$ is generated by $C^i$, observe that by condition \ref{cond:fragment:rsystem:irreducible} $l^i$ is irreducible by $\RSystem{t}{i-1}$ so there must be some $j > i$ such that $l^i$ is reduced by the rule $l^j \Rightarrow r^j$ generated by $C^j$. However, since $j > i$, we have $l^j \approx r^j >_t l^i \approx r^i$, so by the definition of the order extension to sets of literals, we have $l^j \geq_t l^i$. But if $l^j = l^i$, then the addition of $\{ l^j \Rightarrow r^j\}$ violates condition \ref{cond:fragment:rsystem:irreducible}, and if $l^j >_t l^i$, then $l^j$ cannot reduce $l^i$ because then it would have to be a subterm of $l^i$, but by definition of the ordering between ground terms, any subterm $s$ of $l^i$ verifies $l^i >_t s$. 

To conclude the proof, observe that $\Rsystem{t}$ is terminating because by property \ref{cond:fragment:rsystem:order} all rules in $\Rsystem{t}$ are of the form $l \Rightarrow r$ with $l >_t r$, and $>_t$ is a simplification order.
\end{proof}

{\textsc{Canonicity}}
$\phantom{AAA}$ \\
The main result of this sub-section is \cref{corollary:fragment:properties:monotonicity}, and it shows that whenever an equality generates a rewrite rule, no other literal in the head of the same clause is satisfied by the model fragment. Thus, every rule is \emph{necessary} for satisfying the fragment, in the sense that the elimination of any of the rewrite rules would generate a Herbrand equality interpretation which is not a model of the relevant set of clauses. Moreover, every rule is \emph{maximal} in the sense that it is not possible to eliminate it from the model and replace it by any number of larger rules in order to recover a model. There is clearly a unique model fragment which verifies these properties; this is the sense in what the model constructed in this section is \emph{canonical}.

\begin{lemma}
\label{lemma:fragment:properties:monotonicity}
For each $1 \leq i \leq n$, we have that for each literal $L$ which is either (i) an inequality $l \nequals r$ contained in $\Delta^i \vee L^i$ or (ii) an equality $l \equals r$ contained in $\Delta^i$ for $C^i$ a generative clause, then $\RModel{t}{i-1} \models L$ if and only if $\Rmodel{t}  \models L$.
\end{lemma}

\begin{proof}
Consider first the case where $L = l \nequals r$. Consider some $1 \leq i \leq n$; we first prove that if $\Rmodel{t} \models l \nequals s$, then $\RModel{t}{i-1} \models l \nequals s$ by showing its contrapositive: if $\RModel{t}{i-1} \models l \equals r$ then $\Rmodel{t} \models l \equals r$. But this is true by virtue of the fact that $\RSystem{t}{i-1} \subseteq \Rsystem{t}$. We complete the proof of this case by proving the contrapositive of the reverse implication: if $\RModel{t}{i-1} \nmodels l \equals r$ then $\Rmodel{t} \nmodels l \equals r$. We prove this by induction: consider an arbitrary $j$ with $i \leq j \leq n$; suppose  $\RModel{t}{j-1} \nmodels l \equals r$ and let us prove  $\RModel{t}{j} \nmodels l \equals r$. We assume that $C^j$ is generative, as otherwise the result is trivial. Let $L^j = l^j \equals r^j$. We have $l^j > l$; indeed, if $l \nequals r = L^i$, then $j > i$ implies $L^j \geq_t L^i$, but since $L^j$ is an equality and $L^i$ is an inequality, we must have $l^j >_t l$. If, instead, $l \nequals r \in \Delta^i$, we have that $L_i >_t \Delta^i$ implies $L^j>_t l \nequals r$ and the same argument applies. Thus, since both $\RSystem{t}{j-1}$ and $\RSystem{t}{j}$ are Church-Rosser, and $l^j > l$, the result follows by Lemma \ref{lemma:rewrite:bigger-literal}.

Consider now the case $ L=  l \equals r$, with $L \in \Delta^i$ and $C^i$ generative. The implication that if $\RSystem{t}{i-1} \models L$ implies $\RModel{t}{}$ follows again from the fact that $\RSystem{t}{i-1} \subseteq \Rsystem{t}$. In order to prove the reverse direction, we prove the contrapositive: if $\RModel{t}{i-1} \nmodels l \equals r$, then $\Rmodel{t}  \nmodels l \equals r$. Observe that we cannot directly re-use the proof in the previous paragraph, since that proof uses that $l \nequals r \in C^i$, which is not true in this case.

Consider an arbitrary $j$ with $i \leq j \leq n$; suppose  $\RModel{t}{j-1} \nmodels l \equals r$ and let us prove  $\RModel{t}{j} \nmodels l \equals r$. If $l^j > l$, we proceed as in the previous paragraph. If $l^j =l$, we have two possible cases: 
\begin{itemize}
\item Case $j=i$. We have $\RModel{t}{i} \models l \equals r^i$, since $l=l^j=l^i$. By \cref{cond:fragment:rsystem:noduplicates} we have $\RModel{t}{i-1} \nmodels r \equals r^i$. Therefore, by the argument used in the previous case, $\RModel{t}{i} \nmodels r \equals r^i$. But then, if $\RModel{t}{i} \models l \equals r$, the fact that $\RModel{t}{i} \models l \equals r^i$ implies $\RModel{t}{i}  \models r \equals r^i$, and hence we reach a contradiction. Thus, $\RModel{t}{i}  \nmodels l \equals r$.
\item Case $j > i$. Then, inequalities $l^j \geq_t l^i$ and $l^i \geq_t l$ imply $l=l^i$. But since both $C^i$ and $C^j$ generate rules of the form $l^i \Rightarrow \dots$, the system $\RSystem{t}{j}$ is not Church-Rosser, which contradicts \cref{lemma:fragment:properties:church-rosser}.
\end{itemize}
\end{proof}

\begin{corollary}
\label{corollary:fragment:properties:monotonicity}
For any $1 \leq i \leq n$, if $C^i$ is generative, we have $\model{R}{t} \nmodels \Delta^i$. 
\end{corollary}

{  \textsc{Non-triviality}}
$\phantom{AAA}$ \\
\begin{lemma}
\label{lemma:fragment:properties:irreducibility}
Both $t$ and $t'$ (if it exists) are irreducible by $R_t$. 
\end{lemma}

\begin{proof}

Observe that if $t=c$, then $t'$ does not exist, $c$ cannot occur on the left-hand side of a rewrite rule since it is the smallest term. Thus, in the remainder of this proof, we assume $t \neq c$, and proceed by contradiction. 

Suppose that there is a generative clause $C^i$ where $L^i$ is of the form $t' \Rightarrow s$ for some term $s$ with $t' >_t s$. Hence, because terms in rewrite systems are in the neighbourhood of $t$ or in $\AllIndiv$, we have that $s \in \AllIndiv$. If $t' \in \AllIndiv$ itself, then the literal $t' \equals s$ is in $\Omega$, so by \cref{cond:fragment:induction:individuals}, we have that either $t' \equals s \in \Gamma_o$ or $t' \equals s \in \Delta_o$. Suppose $t' \equals s \in \Delta_o$. Then, by the choice of order, every literal $L$ in $\Delta^i$ must be in $\Omega$, so it must belong to either $\Gamma_o$ or $\Delta_o$. Observe that any such literal in $\Delta_o$, by \cref{cond:fragment:induction:individuals}, must be in $\Forbidden{t}$.

If we have some $L \in \Delta^i$ such that $L \in \Gamma_o$, then by \cref{cond:fragment:induction:individuals,cond:fragment:induction:required}, we have that clause $\Gamma_t \to L \sbin N_t$, and in order not to violate \cref{cond:fragment:induction:compatibility}, we have $\Gamma_t \to L \in N_t$. We cannot have that this clause is generative, since otherwise $\RModel{t}{i-1} \models L$, which contradicts that $C^i$ is generative, since this clause is smaller than $C^i$. Thus, one of \cref{cond:fragment:rsystem:notmodels,cond:fragment:rsystem:order,cond:fragment:rsystem:irreducible,cond:fragment:rsystem:noduplicates} is satisfied. Clearly, it cannot be\cref{cond:fragment:rsystem:notmodels} as this would still imply $\RModel{t}{i-1} \models L$ because this clause is smaller than $C^i$. Similarly, it cannot be \cref{cond:fragment:rsystem:order}, as then we would have $L =  s' \equals s'$ for some $s' \in \AllIndiv$, and we would trivially have  $\RModel{t}{i-1} \models L$. Furthermore, it cannot be \cref{cond:fragment:rsystem:noduplicates}, as there is one literal only in the head of the clause. Thus, we have that there must be some smaller generative clause generating rule $o_1 \equals o_3$, where $L = o_1 \bowtie o_2$ and $o_1 >_t o_2$. Observe that $o_1 \equals o_3 \in \Gamma_o$ as otherwise we contradict \cref{cond:fragment:rsystem:notmodels}. But then, by \cref{cond:fragment:induction:consistency}, we have that $o_2 \bowtie o_3 \in \Gamma_o$ again. We can repeat the procedure described in this paragraph to generate an infinite sequence of nominals in $\Onto$, which is clearly a contradiction. Thus, we have that $L$ cannot be an inequality. If $L$ is an equality, at some point, we find some $o_{n-1} \equals o_n \in \Gamma_o$ such that $o_{n-1}$ is irreducible w.r.t the system we have just before considering $\Gamma_t \to o_{n-1} \equals o_n$ and we conclude $\RModel{t}{i-1} \models o_{n-1} \equals o_n$, but we also have that $\RModel{t}{i-1} \models o_{n-2} \equals o_n$, due to the way these individuals have been created, so we conclude $\RModel{t}{i-1} \models o_{n-2} \equals o_{n-1}$; repeating this procedure eventually yields $\RModel{t}{i-1} \models o_{1} \equals o_2$ i.e. $\RModel{t}{i-1} \models L$, so we again produce the same contradiction. 

If $t' \notin \AllIndiv$, then we have that according to \cref{cond:fragment:induction:nocollapse}, such literal must be in $\Forbidden{t}$, or be in $\Gamma_o$. However, if the latter were to be the case, then we have that by \cref{lemma:fragment:property:compatibility} applied to $\Gamma_{t'} \to t' \equals o \sbin N_{t'}$, we have that $t'$ is reduced by $\Rmodel{t'}$, and this contradicts this lemma applied to the model fragment for $t'$, which is verified due to \cref{cond:fragment:induction:all-lemas}. The rest of the head must be in $\Forbidden{t}$, as already argued, so we have that the entire head is in $\Forbidden{t}$; this contradicts \cref{cond:fragment:rsystem:notmodels}. We therefore conclude that $t'$ is irreducible. 

To prove $t$ is also irreducible, suppose that there is a generative clause $C^i$ where $L^i$ is of the form $t \Rightarrow s$ for some term $s$. If $s \neq t'$, we have $s \in \AllIndiv$, and hence $t \Rightarrow s \in \Delta_t$. But then, the entire head of the clause is $\Omega$, which leads to a contradiction, according to the argument formulated in the previous paragraph. Thus, suppose $ s = t'$, and observe that $t \equals t' \in \Forbidden{t}$; also we have that $\Delta^i$ can only contain literals of the form $t' \nequals t'$ or $t' \equals t'$, or  $t' \equals o$ with $t' \notin \AllIndiv$, or in $\Omega$. If it contains $t' \nequals t'$, then we contradict \cref{cond:fragment:induction:compatibility}, since \cref{cond:redundancy:equality} of \cref{def:redundancy-elim} of redundancy is satisfied; if it contains $t' \equals t'$, then \cref{cond:fragment:rsystem:notmodels} is not satisfied, so the clause cannot be generative. If it contains some $t' \equals o$ with $t' \notin \AllIndiv$, then, according to \cref{cond:fragment:induction:nocollapse}, such literal must be in $\Forbidden{t}$, or be in $\Gamma_o$. However, if the latter were to be the case, then we have that by \cref{lemma:fragment:property:compatibility} applied to $\Gamma_{t'} \to t' \equals o \sbin N_{t'}$, we have that $t'$ is reduced by $\Rmodel{t'}$, and this contradicts this lemma applied to the model fragment for $t'$, which is verified due to \cref{cond:fragment:induction:all-lemas}. Finally, any atom of the head in $\Omega$ must be in $\Forbidden{t}$ as it has already been discussed. Thus, the entire head is in $\Forbidden{t}$; this contradicts \cref{cond:fragment:rsystem:notmodels}, and hence $t$ is irreductible by $\Rsystem{t}$. 

\end{proof}

{  \textsc{Satisfaction}}
$\phantom{AAA}$ \\

\begin{lemma}
\label{lemma:fragment:properties:satisfaction}
For each $1 \leq i \leq n$, we have $\Rmodel{t} \models C^i $.
\end{lemma}

\begin{proof}
In order to prove the lemma, we prove a stronger result: for each $1 \leq i \leq n$, we have $\Rmodel{t} \models \Delta^i \vee L^i$.  We proceed using proof by contradiction: suppose the lemma is false, and let $i$ be the smallest number between $1$ and $n$ such that the lemma is not verified. i.e.  $\Rmodel{t}  \nmodels \Delta^i \vee L^i$, but $\Rmodel{t}  \models \Delta^j \vee L^j$ for $j < i$. We assume, without loss of generality, that $C^i$ is not generative, as otherwise a contradiction is immediately generated.

By definition of $N_t$, there is a clause $\Gamma^I \rightarrow \Delta^I \vee L^I \in \SC{v}$ such that $$\Gamma^I \Grounding{t} = \Gamma^i \subseteq \Gamma_t \quad \quad \Delta^I \Grounding{t} = \Delta^i \quad \quad L^I \Grounding{t} = L^i.$$ Observe that $\Delta^I \not \succeq_v L^I$ since if there were some literal $L \in \Delta^I$ with $L \succeq_v L^I$, we would have have $L \Grounding{t} >_t L^I \Grounding{t} = L^i$ by  \cref{cond:fragment:groundorder:compatibility}, which contradicts $L^i >_t \Delta^i$, since $L \sigma_t \in \Delta^i$.  Similarly, let $L^I$ be of the form $l_I \equals r_I$, with $l_I \Grounding{t} = l^i$ and $r_I \Grounding{t} = r^i$, and observe that we have $r_I \not \succeq_v l_I$, since if $r_I \succeq_v l_I$, then again by   \cref{cond:fragment:groundorder:compatibility}, $r^i >_t l^i$, which contradicts the assumption that $l^i \geq_t r^i$.

We will prove the lemma by considering all possible forms of $L^i$:
\begin{itemize}
\item Case $L^i = l^i \equals l^i$. Then we have that $\Rmodel{t} \models l^i \equals l^i$ trivially, which contradicts the main hypothesis.
\item Case $L^i = l^i \equals r^i$, with $l^i >_t r^i$. We reach a contradiction by showing that \cref{cond:fragment:rsystem:notmodels,cond:fragment:rsystem:order,cond:fragment:rsystem:irreducible,cond:fragment:rsystem:noduplicates} are verified, and therefore $C^i$ should be generative, which is a contradiction. We assume that the head of the clause is not entirely in $\Omega$, as otherwise, as we have already shown, it is generative or it violates \cref{cond:fragment:induction:compatibility}.
		\begin{itemize}
		\item Condition \ref{cond:fragment:rsystem:notmodels}. Suppose this condition is not verified. Then, there is some literal $K \in \Delta^i \vee L^i$ such that $\RModel{t}{i-1} \models K$. If $K$ is an equality, then $\Rmodel{t} \models K$ since $\RModel{t}{i-1} \subseteq \Rmodel{t}$; if $K$ is an inequality, lemma \ref{lemma:fragment:properties:monotonicity} guarantees $\Rmodel{t} \models K$. Thus, we reach a contradiction with the main hypothesis $\Rmodel{t} \nmodels \Delta^i \vee L^i$.  
		\item Condition \ref{cond:fragment:rsystem:order} \emph{must} be verified by our assumption that $l^i >_t r^i$.
		\item Condition \ref{cond:fragment:rsystem:irreducible}. Suppose $l^i$ can be reduced by $\RModel{t}{i-1}$. Let $l^j \Rightarrow r^j$ be one of the rules in $\RModel{t}{i-1}$ which reduces $l^i$, and let $p$ be a position at which $l^j$ reduces $l^i$. By definition of $N_t$, there is a clause $\Gamma^J \rightarrow \Delta^J \vee L^J \in \SC{v}$ such that $$\Gamma^J \Grounding{t} = \Gamma^j \subseteq \Gamma_t \quad \quad \Delta^J \Grounding{t} = \Delta^j \quad \quad L^J \Grounding{t} = L^j,$$ with $\Delta^J \not \succeq_v L^J$ and $L^J$ is of the form $l^J \equals r^J$, with $l^J \not \succeq_v r^J$, $l^J \Grounding{t} = l^j$ and $r^J \Grounding{t} = r^j$. Observe that by \cref{lemma:fragment:properties:irreducibility}, $l^J$ cannot be $y$ or $x$. 		
		We are therefore in the conditions of the rule \ruleword{Eq}, but since this rule is not applicable, we have that $$\Gamma^I \wedge \Gamma^J \rightarrow \Delta^I \vee \Delta^J \vee l_I[r_J]_p \approx r_I \sbin \SC{v}.$$ Moreover, since $\Gamma^I \Grounding{t} \subseteq \Required{t}$ and $\Gamma^J \Grounding{t} \subseteq \Required{t}$, and $\Delta^I \Grounding{t}$, $\Delta^J \Grounding{t}$, and $l_I[r_J]_p \approx r_I \Grounding{t}$ are ground, we have that by  \cref{lemma:fragment:operational:redundancy}, $$\Gamma^i \wedge \Gamma^j \rightarrow \Delta^i \vee \Delta^j \vee l^i[r^j]_p \approx r^i \sbin N_t.$$ Finally, observe that $\Rmodel{t} \nmodels \Delta^i \vee \Delta^j \vee l^i[r^j]_p \approx r^i$: indeed, (i) $\Rmodel{t} \nmodels \Delta^i$ is true by hypothesis, (ii) $\Rmodel{t} \nmodels \Delta^j$ follows from the fact that $j$ is generative and corollary \ref{corollary:fragment:properties:monotonicity}, and (iii) $\Rmodel{t} \nmodels l^i[r^j]_p \approx r^i$ because $\Rmodel{t}$ is a congruence, so $\Rmodel{t} \models l^i[r^j]_p \approx r^i$ would imply $\Rmodel{t} \models l^i \approx r^i$, and that would contradict the main hypothesis. Nevertheless, we have that $L^i \vee \Delta^i >_t \Delta^i \vee \Delta^j \vee l^i[r^j]_p \approx r^i$, since $L^i >_t \Delta^i$, and also $L^i \geq_t L^j >_t \Delta^j$ since $j < i$, and also $L^i >_t l^i[r^j]_p \approx r^i$ since $l^j >_t r^j$ (because $C^j$ is generative). By \cref{lemma:fragment:operational:smallerclause}, $\Rmodel{t} \models \Delta^i \vee \Delta^j \vee l^i[r^j]_p \approx r^i$, which contradicts our previous result. 
		\item Condition \ref{cond:fragment:rsystem:noduplicates}. Suppose that there is a term $s$ such that $l^i \equals s \in \Delta^i$ and $\RModel{t}{i-1} \models r^i \equals s$, and hence $\Rmodel{t} \models r^i \equals s$. By definition of $N_t$, similarly to the previous case, we have that this clause has a non-ground form in $\SC{v}$, which we write as $$\Gamma^I \rightarrow \bar{\Delta}^I \vee l^I \equals s^I \vee l^I \equals r^I.$$ Indeed, observe that the non-ground form of $l^i$ is uniquely determined, since the only possibility where this would not be the case \textit{a priori} is if $l^i \in \AllIndiv$, as then it may have been generated by $y$ or some $o \in \AllIndiv$ in the non-ground clause. But then, the entire head of the clause would be in $\Omega$, which contradicts our assumption.

		We are in the conditions of the \ruleword{Factor} rule, and since this rule is not applicable, we have that $$\Gamma^I \rightarrow \bar{\Delta}^I \vee (r^I \nequals s^I) \vee (l^I \equals r^I) \sbin \SC{v},$$ so $\Gamma^i \to \bar{\Delta}^i \vee r^i \nequals s \vee l \equals r^i \sbin N_t$ again by \cref{lemma:fragment:operational:redundancy}. Since $\Rmodel{t}$ is a congruence, $\Rmodel{t} \nmodels \Gamma \to \bar{\Delta}^i \vee r^i \nequals s \vee l \equals r^i$. However, since $l^i >_t r^i$ we have that $l^i \equals s >_t r^i \nequals s$, so $L^i \vee \Delta^i >_t \bar{\Delta}^i \vee r^i \nequals s \vee l \equals r^i$, and by \cref{lemma:fragment:operational:smallerclause}, $\Rmodel{t} \models \bar{\Delta}^i \vee r^i \nequals s \vee l \equals r^i$, which contradicts our previous claim.

		\end{itemize}

\item  Case $L^i = l^i \nequals l^i$. Then we have that $L^I$ is of the form $l^I \equals l^I$, but then \cref{lemma:fragment:operational:lneql} implies $\Gamma^i \rightarrow \Delta^i \sbin N_t$, and since $L^i >_t \Delta^i$, we have $\Rmodel{t} \models \Delta^i$ by \cref{lemma:fragment:operational:smallerclause}, which contradicts the main hypothesis. 

\item Case $L^i = l^i \nequals r^i$ with $l^i >_t r^i$. By \cref{lemma:fragment:properties:monotonicity}, we have that $\RModel{t}{i-1} \nmodels l^i \nequals r^i$, which means that $\RModel{t}{i-1} \models l^i \equals r^i$, and hence $l^i$ is reducible by $\RSystem{t}{i-1}$. The contradiction is then generated analogously to the case above, where we show that if $L^i = l^i \equals r^i$ with $l^i >_t r^i$ and condition \ref{cond:fragment:rsystem:irreducible} is not verified, there is a smaller, equivalent clause $C \sbin  N_t$, and this generates a contradiction because equivalence entails $\Rmodel{t} \nmodels C$, but induction hypothesis entails $\Rmodel{t} \models C$. 
\end{itemize}
\end{proof}

\begin{corollary}
\label{corollary:fragment:property:satisfaction}
For any clause $\Gamma \rightarrow \Delta \sbin N_t$, we have $\Rmodel{t} \models \Delta$. 
\end{corollary}
\begin{proof}
The result is trivial under \cref{cond:redundancy:equality} of \cref{def:redundancy-elim}; for condition \ref{cond:redundancy:subset}, the result is a direct consequence of  \cref{lemma:fragment:properties:satisfaction,lemma:fragment:operational:smallerclause}.
\end{proof}

{  \textsc{Compatibility}}
$\phantom{AAA}$ \\

\begin{lemma}
\label{lemma:fragment:property:compatibility}
We have that $\Rmodel{t} \models \Required{t}$ and $\Rmodel{t} \not \models \Forbidden{t}$.
\end{lemma}

\begin{proof}
In order to prove that $\Rmodel{t} \models \Required{t}$, observe that condition \ref{cond:fragment:induction:required} implies that for each $A \in \Gamma_t$ we have $\Gamma_t \rightarrow A \sbin N_t$, so the result follows from Corollary \ref{corollary:fragment:property:satisfaction}.

Now, we prove $\Rmodel{t} \nmodels \Forbidden{t}$, and we do it by induction and contradiction. Thus, suppose we have a literal $K \in \Forbidden{t}$ such that $\Rmodel{t} \models K$ and $K$ is the smallest element in $\Forbidden{t}$ for which this is true. This means that there exists a position $p$ such that $K|_p= l^i$ for some rule $l^i \Rightarrow r^i \in \Rsystem{t}$. We will complete the proof by showing that $\Delta^i \vee l^i \equals r^i \subseteq \Forbidden{t}$, which then violates condition \ref{cond:fragment:induction:compatibility}.

Since $K \geq_t l^i \equals r^i$, we have that $K \geq_t L$ for any literal $L \in \Delta^i \vee l^i \equals r^i$. Let us represent $L$ as $l \bowtie r$ with $l \geq_t r$, and consider the possible forms of $L$, which are limited as a consequence of $K \geq_t L$ and condition \ref{cond:fragment:groundorder:forbidden}:

\begin{itemize}

\item $l \equals r$ with $r = \TOP$. Then, since $K \geq_t L$ and $K \in \Forbidden{t}$, by condition \ref{cond:fragment:groundorder:forbidden}, we have that $L \in \Forbidden{t}$.
\item $l \equals l$. But then, by condition \ref{cond:fragment:rsystem:notmodels}, $C^i$ is not generative, so $L$ cannot be of this form.
\item $l \nequals l$. But then, by \cref{lemma:fragment:operational:lneql}, clause $\Gamma^i \rightarrow (\Delta^i \vee L^i) \backslash l \nequals l \sbin N_t$, and by  \cref{corollary:fragment:property:satisfaction}, $\Rmodel{t} \models (\Delta^i \vee L^i) \backslash l \nequals l $. Observe that \cref{cond:redundancy:equality} of \cref{def:redundancy-elim} cannot be verified as otherwise $C^i$ would not be generative, so let $C^k$ be the corresponding clause of $N_t$ subsuming $\Gamma^i \rightarrow (\Delta^i \vee L^i) \backslash l \nequals l$. By \cref{corollary:fragment:properties:monotonicity} $\Rmodel{t} \not \models \Delta^i$, and since $\Rmodel{t} \models \Delta^k \vee L^k$, we have $L^i \in \Delta^k \vee L^k$ since otherwise clause $C^k$ violates \cref{corollary:fragment:property:satisfaction}.

Observe that $k < i$, for clause $C^k$ subsumes $C^i$. Moreover, $L^i = L^k$, since $L^i \in (\Delta^i \vee L^i) \backslash l \nequals l$, for $C^i$ is generative and $L^i$ cannot be an inequality; but if $L^i \in \Delta^k$, then $L^k >_t L^i$ and therefore it cannot be the case that $k <i$. Now, we have: (i) $\RModel{t}{k-1} \nmodels \Delta^k$, since otherwise we would have $\RModel{t}{i-1} \models \Delta^i$, and $C^i$ would not be generative; (ii) $l^i >_t r^i$; (iii) $l^i$ is irreducible by $\RModel{t}{k-1}$, as otherwise it would not be irreducible by $\RModel{t}{i-1}$, and (iv) there is no $l^i \equals s$ in $\Delta^k$ such that $\RModel{t}{k-1} \models l^i \equals s$, since if this were the case, the same equality would be in $\Delta^i$ and  $\RModel{t}{i-1} \models l^i \equals s$, so $C^i$ would not be generative. Thus, we have that $C^k$ is generative, and it generates $l^i \Rightarrow r^i$. But again, this is a contradiction, since then $C^i$ cannot be generative. Hence, $L$ cannot be of the form $l \nequals l$.
\item $t \equals t'$. But then, by condition \ref{cond:fragment:induction:nocollapse}, we have $L \in \Forbidden{t}$.
\item $t \nequals t'$. By \cref{lemma:fragment:properties:irreducibility}, both $t$ and $t'$ are irreducible by $\Rmodel{t}$, which means $\Rmodel{t} \models t \nequals t'$. Then, by \cref{lemma:fragment:properties:monotonicity}, $\RModel{t}{i-1} \models t \nequals t'$, and hence $C^i$ is not generative, which contradicts our assumption, so $L$ cannot be of this form. 
\item $o_1 \equals o_2$, with $o_1,o_2 \in \AllIndiv$ and $o_1 >_t o_2$. But by \cref{cond:fragment:induction:individuals}, either $L \in \Required{t}$ or $L \in \Forbidden{t}$. However, in the former case, we have $\Rmodel{t} \models L$, and since $C^i$ is generative, $L=L^i$. But then, by \cref{cond:fragment:induction:consistency}, we have $K' = K[o_2]_p$ is in $\Forbidden{t}$, and since $\Rmodel{t}$ is a congruence, $\Rmodel{t} \models K'$. This contradicts that $K$ is the minimal element of $\Forbidden{t}$ satisfied by $\Rmodel{t}$; indeed, $o_1 >_t o_2$ implies $K >_t K'$. Thus, if $L$ is of this form, $L \in \Forbidden{t}$.
\item $o_1 \nequals o_2$, with $o_1,o_2 \in \AllIndiv$ and $o_1 >_t o_2$. But by \cref{cond:fragment:induction:individuals}, either $L \in \Required{t}$ or $L \in \Forbidden{t}$. However, in the former case, we have that $\Rmodel{t} \models L$, and since $L$ is an inequality, by \cref{lemma:fragment:properties:monotonicity}, $\RModel{i-1}{t} \models L$ and $C^i$ is not generative, which contradicts our assumption. Again, if $L$ is of this form, $L \in \Forbidden{t}$.
\item $t' \nequals o$ with $o \in \AllIndiv$, $t' \notin \AllIndiv$. Since $L$ is an inequality, $L \neq L^i$, and since $C^i$ is generative $\RModel{t}{i-1} \nmodels t' \nequals o$, so by \cref{lemma:fragment:properties:monotonicity}, $\Rmodel{t} \nmodels t' \nequals o$, so $\Rmodel{t} \models t' \equals o$; since $t' \notin \AllIndiv$, then $t' >_t o$, and since system is Church-Rosser by \cref{lemma:fragment:properties:church-rosser}, we have that there must be a rewrite rule of the form $t' \Rightarrow s'$ for some term $s'$. But this contradicts \cref{lemma:fragment:properties:irreducibility}. Hence, $L$ cannot be of this form. 
\item $t' \equals o$ with $o \in \AllIndiv$, $t' \notin \AllIndiv$; but then, $L \in \Forbidden{t}$. 
\item $t \equals o$, with $o \in \AllIndiv, t \notin \AllIndiv$; but then, $L \in \Forbidden{t}$. 
\item $t \nequals o$, with $o \in \AllIndiv, t \notin \AllIndiv$. We have that $L$ cannot then be $L^i$, as $L$ is an inequality and $C^i$ is generative. Then, by condition \ref{cond:fragment:rsystem:notmodels}, we have that $\RModel{t}{i-1} \nmodels t \nequals o$. By \cref{lemma:fragment:properties:monotonicity}, $\Rmodel{t} \nmodels t \nequals o$, and since $t \notin \AllIndiv$, $t >_t o$, and the system is Church-Rosser, but this means that $t$ is reducible by $R_t$, but this contradicts lemma \ref{lemma:fragment:properties:irreducibility}. Hence, $L$ cannot be of this form. 
\end{itemize}
Observe that either we reach a contradiction or the literal $L$ is in $\Forbidden{t}$. Thus, $\Delta^i \vee L^i \subseteq \Forbidden{t}$, and this contradicts condition \ref{cond:fragment:induction:compatibility}. This proves the lemma by contradiction.
\end{proof}

\subsection{Combining the models}

We use partial induction over the \ab-terms of the Herbrand Universe. We define a function $X$ that maps a term $t$ to a context $X_t \in \mathcal{V}$, and functions $\Gamma$ and $\Delta$ which map each $t$, respectively, to conjunction $\Required{t}$ and disjunction $\Forbidden{t}$. Finally, function $R$ maps each $t$ to the model fragment $R_t$ for $\langle t,\Required{t},\Forbidden{t}, X_t \rangle$.

\subsubsection{Unfolding strategy}
\label{sec:completeness:unfolding-strategy}

\begin{itemize}
	\item For $c$, we write $X_c=q$, $\Gamma_c = \Gamma_q \sigma_c$ and $\Delta_c = \Delta_q \sigma_c$.	
	
	\item For each $t \in \AllIndiv$, unless $t \equals o$ for some $o \in \AllIndiv$ with $t >_t o$, we write $X_t=v_r$, $\Gamma_t= \Gamma_o$, and $\Delta_t = \Delta_o$. Otherwise, we add no fragment for $t$.
	\item For any other $t$, which must be of the form $f(t')$ for some $f \in \Sigma^f$, we distinguish several cases:\begin{itemize}
		\item If $f(t')$ is irreducible by $\Rmodel{t'}$, and $t$ appears in $\Rmodel{t'}$, then $f(t')$ appears in some ground clause $C^i \in N_{t'}$ which we write as $\Gamma^i \rightarrow \Delta^i \vee L^i$. Observe that $f(t') \in L^i$ because $L^i$ generates the rule where $f(t')$ appears. Let $u = X_{t'}$; we have that there is in $\SC{u}$ a clause $C = \Gamma \rightarrow \Delta \vee l \equals r$ such that $\Gamma \sigma_{t'} = \Gamma^i$, $\Delta \sigma_{t'} = \Delta^i$, $l \Grounding{t'} = l^i$, and $ r \Grounding{t'}= r^i$. Since rule \ruleword{Succ} is not applicable to $C$, there must be an $f$-descendant $v$ of context $u$. We choose $X_t= v$, and then $\Required{t} = \Rmodel{t} \cap \SUt \cap \Omega$, and $\Forbidden{t} = (\Omega \cup \PRt) \backslash \Rmodel{t'}$.
		\item If $f(t')$ is irreducible by $\Rmodel{t'}$ but $t$ does not appear in $\Rmodel{t'}$, we define $\Rmodel{t}$  as $t \Rightarrow c$. 
		\item If $f(t')$ is not irreducible by $\Rmodel{t'}$, we define no fragment for $t$. Observe that as a consequence of this, we define no fragment either for any successor of $t$. 
		\end{itemize}
\end{itemize}

We now prove that this unfolding strategy verifieis a series of properties:

\begin{lemma}
\label{lemma:unfolding:consistency}
If $A \in \Gamma_o$, then there is either a (not necessarily generative) clause $\top \to \Delta \vee A \in N_c$ with $A >_c \Delta$ and $\Rmodel{c} \nmodels \Delta$, or clauses $\top \to \Delta_1 \vee A(x)$ and $\top \to \Delta_2 \vee x \equals o$, where $A(x) \{x \mapsto o\} = A$. 
\end{lemma}
\begin{proof}
Let $A'$ be the normal form of $A$ (this includes the case where $A$ is its own normal form) w.r.t \Rsystem{c}. Since $A \in \Omega$, there is a sequence of atoms $A_0, \dots, A_k$ and a sequence of rewrite rules in $\Rsystem{c}$ $l_1 \Rightarrow r_1, \dots,l_k \Rightarrow r_k$ and positions $p_1,\dots,p_k$ such that $A_0 = A'$, $A_k = A$ and for any $1 \leq s \leq k$, $A_{s}[r_s]_{p_s} = A_{s-1}$. We complete the proof using induction. 

The base case is verified because $A'$ is irreducible, so there is a generative clause $\Gamma_0 \to \Delta_0 \vee A'$, and by \cref{corollary:fragment:properties:monotonicity}, we have $\Rmodel{c} \nmodels \Delta_0$.
For the induction step, suppose there is a clause $\Gamma_s \to \Delta_s \vee A_s \in N_c$ such that $\Rmodel{c} \nmodels \Delta_s$. Consider rewrite rule $l_s \Rightarrow r_s$, which must have been generated by a generative clause $\Gamma' \rightarrow \Delta' \vee l_s \equals r_s$; note that again by \cref{corollary:fragment:properties:monotonicity}, $\Rmodel{c} \nmodels \Delta'$. If we consider the non-ground version of these clauses, we have that $A_s$ cannot contain function symbols and contains $l_s$, which is in $\AllIndiv$. If the corresponding $r_s$ is an individual, we are in the conditions of the \ruleword{Eq} rule, and since this rule is not applicable, we have that clause $\Gamma' \wedge \Gamma_s \to \Delta' \vee \Delta_s \vee A_{s+1} \sbin N_t$. Observe that this clause is smaller than $\Gamma_s \to \Delta_s \vee A_s$, as this clause must be greater than $\Gamma' \rightarrow \Delta' \vee l_s \equals r_s$ since $l_s$ is a subterm of $A_s$. Thus, by \cref{lemma:fragment:operational:smallerclause} we have $\Rmodel{c} \models  \Delta' \vee \Delta_s \vee A_{s+1}$. Note also that \cref{cond:redundancy:equality} cannot be true for $\Gamma' \wedge \Gamma_s \to \Delta' \vee \Delta_s \vee A_{s+1} \sbin N_c$, as two clauses used to generate this were generative, and therefore we have that there is a clause $C$ in $N_c$ which subsumes this one. However, $A_{s+1}$ must be in the head of $C$, since otherwise, due to $\Rmodel{c} \nmodels \Delta_s$ and $\Rmodel{c} \nmodels \Delta'$, we would violate \cref{corollary:fragment:property:satisfaction}; hence, $C$ is the clause whose existence is postulated by the lemma for $s+1$. By induction, this verifies the lemma. By contrast, it $l_s$ is not an individual, it must be the case that $r_s = c$; even though we are not in the conditions of rule \ruleword{Eq}, the non-ground form of clauses $\Gamma' \rightarrow \Delta' \vee l_s \equals r_s$ and $\Gamma_s \to \Delta_s \vee A_s $ are precisely the clauses whose existence is postulated by the lemma.  
\end{proof}

\begin{lemma}
\label{lemma:unfolding:noncontradiction}
$\RequiredN \to \ForbiddenN \nsbin N_t$. 
\end{lemma}
\begin{proof}
First, observe that by definition of $\Delta_o$, we have that if $L \in \Delta_o$, then $\Rmodel{c} \nmodels L$, so $L$ cannot be an equality; nor can it be the case that we have $L_1$ and $L_2$ in $\Delta_o$ of the form $L_1 = l \equals s$ and $L_2 = l \nequals s$, since $\Rmodel{c} \nmodels L_1$ iff $\Rmodel{C} \models L_2$. Hence, \cref{cond:redundancy:equality} of \cref{def:redundancy-elim} cannot be satisfied.

For $t= c$, the result is easy to see: suppose $\Gamma^i \to \Delta^i \vee L^i \in N_c$; since $\Rmodel{c} \nmodels \Delta^i \vee L^i$ by definition of $\Delta_o$, we have that \cref{corollary:fragment:property:satisfaction} is violated for $\Rmodel{c}$, and therefore we reach a contradiction.

Now, we prove this for $t \in \AllIndiv$, using proof by contradiction. Suppose that there is a clause $C^i \in N_t$ of the form $\Gamma^i \rightarrow \Delta^i \vee L^i$, with $\Gamma^i \subseteq \RequiredN$ with $\Delta^i \vee L^i \subseteq \ForbiddenN$. Observe that we have $\Gamma^i \to \Delta^i \vee L^i \in \SC{v_r}$. By $r$-\ruleword{Pred} with $n=0$, we have $\Gamma^i \to \Delta^i \vee L^i \in \SC{v_q}$. We now have two options: either $\Gamma^i =\top$, or there exists some $A \in \Gamma^i$, with $A \in \Omega$. In the first case, we have $\Gamma^i \subseteq \Gamma^c = \top$, and we can apply the same argument as above. In the second case, by \cref{lemma:unfolding:consistency}, there must be either a clause $\top \to \Delta_1 \vee A_1$ with $A_1 >_c \Delta_1$ in $N_c$ or clauses $\top \to \Delta_1 \vee A(x)$ and $\Gamma_2 \to \Delta_2 \vee x \equals o$, where $A(x) \{x \mapsto o\} = A$. In both circumstances, since rule \ruleword{Join} is not applicable, we have that $(\Gamma^i \backslash A) \to \Delta_1 \vee \Delta_2 \vee \Delta^i \vee L^i \sbin \SC{v_r}$, with $\Rmodel{c} \nmodels \Delta_1 \vee \Delta_2 \vee \Delta_i$. Again we have two options, either $(\Gamma^i \backslash A) = \top$, or there exists some $A_2 \in (\Gamma^i \backslash A)$, with $A_2 \in \Omega$. Then, we can repeat the same argument, applying again \cref{lemma:unfolding:consistency} until we eliminate every atom from $\Gamma^i$; but we already have shown that if $\Gamma^i = \top$, we reach a contradiction. This concludes the proof for $t \in \AllIndiv$. 

Now, we prove this for $t$ of the form $f(t')$, also using proof by contradiction. Suppose that there is a clause $C^i \in N_t$ of the form $\Gamma^i \rightarrow \Delta^i \vee L^i$, with $\Gamma^i \subseteq \RequiredN$ with $\Delta^i \vee L^i \subseteq \ForbiddenN$. We shall prove that $\Gamma^i \to \Delta^i \vee L^i \sbin \SC{w}$, where $w$ is either $v_q$ or $v_r$; we use structural induction. The base case is trivial: if $t=c$ or $t \in \AllIndiv$, then $X_t=v_q$ or $X_t=v_r$, respectively, and the result is satisfied by the proofs given in the previous paragraphs. Now, suppose that if $\Gamma^i \to \Delta^i \vee L^i \sbin \SC{u}$ for $u = X_{t'}$, then $\Gamma^i \to \Delta^i \vee L^i \sbin \SC{v_q}$. Observe that since the clause is in $\Omega$, if  $\Gamma^i \to \Delta^i \vee L^i \in N_t$, then $\Gamma^i \to \Delta^i \vee L^i \in \SC{v}$, and since rule \ruleword{Pred} is not applicable, we have that $\Gamma^i \to \Delta^i \vee L^i \in \SC{u}$; this completes the proof by induction and verifies that $\Gamma^i \to \Delta^i \vee L^i \sbin \SC{v_q}$. But we have shown above that this is impossible, and therefore we reach a contradiction.

\end{proof}

\begin{lemma}
\label{lemma:unfolding:induction}
The model fragments defined in these sections satisfy \cref{cond:fragment:induction:compatibility,cond:fragment:induction:nocollapse,cond:fragment:induction:required,cond:fragment:induction:individuals,cond:fragment:induction:consistency,cond:fragment:induction:all-lemas,cond:fragment:induction:deltac}.
\end{lemma}

\begin{proof}
The proof proceeds by structural induction. For $t=c$:
\begin{itemize}
\item \Cref{cond:fragment:induction:compatibility} is proved by contradiction. Suppose $\Gamma_c \to \Delta_c \sbin N_c$. Since $c$ can only appear in $N_c$ through substitution $\sigma_c = \{ x \mapsto c\}$, we have that $\Gamma_c \to \Delta_c \sbin N_c$ $\InputQueryLHS \rightarrow \InputQueryRHS \sbin \SC{q}$, where we have used the definitions of $\Gamma_c$ and $\Delta_c$, but the latter claim violates the main hypothesis of this completeness proof.
\item \Cref{cond:fragment:induction:nocollapse} is trivially satisfied since $c \leq_c o$ for each $o \in \AllIndiv$.
\item \Cref{cond:fragment:induction:required} is satisfied by the assumption that \cref{theorem:completeness:LHS} is satisfied.
\item \Cref{cond:fragment:induction:individuals} is satisfied vacuously since there exists no $t$ with $t >_c c$.
\item \Cref{cond:fragment:induction:consistency} is satisfied because there is no position $p$ with $A|_p \in \AllIndiv$.
\item \Cref{cond:fragment:induction:all-lemas} is verified trivially due to the fact that $c$ has no ancestor.
\item \Cref{cond:fragment:induction:deltac} is satisfied directly by definition of \Required{c}.
\end{itemize}

Now, consider $t \in \AllIndiv$.

\begin{itemize}
\item \Cref{cond:fragment:induction:compatibility} can be proved by contradiction. Observe that $\Gamma_t = \Gamma_o$ and $\Delta_t= \Gamma_o$, so if we had $\Required{t} \to \Forbidden{t} \sbin N_t$, this would violate \cref{lemma:unfolding:noncontradiction}. 
\item \Cref{cond:fragment:induction:nocollapse} must be verified, because $\Delta_t= \Delta_o$, so otherwise we would have $t \equals o \in \Gamma_o$, and hence $\Rmodel{c} \models t \equals o$, which contradicts the fact that we are building a model for $t$ according to the unfolding strategy described in \cref{sec:completeness:unfolding-strategy}.
\item \Cref{cond:fragment:induction:required} is verified since we have $A \to A \sbin v_q$ for all relevant ground atoms, so if $A \in \Gamma_t = \Gamma_o$, then by \cref{lemma:fragment:operational:redundancy}, we have $A \to A \sbin N_t$.  
\item \Cref{cond:fragment:induction:individuals} is trivially satisfied by definition of $\Gamma_o$ and $\Delta_o$. 
\item \Cref{cond:fragment:induction:consistency} is satisfied by definition of $\Gamma_t$ as $\Gamma_o$: we have that both $A$ and $A[o']_p$ are in $\Omega$, and $o \equals o'$ is also in $\Omega$; moreover, we have $A \in \Gamma_o$ and $o \equals o' \in \Gamma_o$. Thus, the fact that $\Gamma_o = \Omega \cap \Rmodel{c}$, together with the fact that $\Rmodel{c}$ is a congruence, implies $A[o']_p \in \Rmodel{c}$; but since $A[o']_p \in \Omega$ too, then $A[o']_p \in \Gamma_o$. 
\item \Cref{cond:fragment:induction:all-lemas} is satisfied directly by the use of structural induction.
\item \Cref{cond:fragment:induction:deltac} is satisfied by definition of $\Delta_c$. 
\end{itemize}

For the remainder of this section, we consider a term $t$ of the form $t=f(t')$. Note that terms $t$ and $t'$ are irreducible by $R_{t'}$
due to \ref{lemma:fragment:properties:irreducibility} due to the order followed in structural induction. 

\begin{itemize}

\item \Cref{cond:fragment:induction:compatibility}: in order to prove this condition, observe that $\SUt$ contains atoms of the form $B(t)$, $S(t,t')$, and $S(t',t)$, and so each atom in $\SUt$ is irreducible by $R_{t'}$. Hence, the atoms in $\Required{t}$ are irreducible by $\Rsystem{t'}$, so if we let $\Gamma_t = \{ A_1, \dots, A_n\}$, these atoms must be generated by generative clauses of the form \cref{eq:pred:side:ground}, where $\Gamma_i \subseteq \Gamma_{t'}$ for every $i$:
\begin{align}
\Gamma_i \rightarrow \Delta_i \vee A_i      & \in N_{t'}  && \text{with} \quad \Delta_i \not \geq_t A_i  \label{eq:pred:side:ground}
\end{align}
But by definition of $N_{t'}$, for each such clause, there must be a clause in $\SC{X_{t'}}$ which satisfies
\begin{align}
\Gamma_i' \rightarrow \Delta_i' \vee A_i'   & \in \SC{u}         && \Gamma_i = \Gamma_i'\sigma_{t'}, \quad \Delta_i = \Delta_i'\sigma_{t'}, \quad A_i = A_i'\sigma_{t'}, \quad \text{and} \quad \Delta_i' \not \succeq_{u} A_i' \label{eq:pred:side:nonground}
\end{align}
Thus, assume for the sake of a contradiction that ${\QueryLHS{t} \rightarrow \QueryRHS{t} \sbin N_{t}}$ holds. Since ${\Forbidden{t} \subseteq \PRt}$ holds due to definition of $\Forbidden{t}$, we consider only the possibility where this clause is contained up to redundancy according to condition \ref{cond:redundancy:subset}. Hence, the set $N_t$ contains a clause 
\begin{align}
\bigwedge_{i=1}^m A_i \rightarrow \bigvee_{i=m+1}^{m+n} L_i     & \quad \begin{array}{l@{\;}l}
\text{with} & \{ A_i \;|\; 1 \leq i \leq m \} \subseteq \Required{t} \subseteq (R_{t'})^* \cap \SUt \\
\text{and}  & \{ L_i \;|\; m+1 \leq i \leq m+n\} \subseteq \Forbidden{t} \subseteq \PRt;
\end{array} \label{eq:pred:main:ground}
\end{align}
to simplify indexing, we assume w.l.o.g.\ that ${A_1, \dots, A_m}$ are the first $m$ atoms from $\Required{t}$. By the definition of $N_{t}$, set $\SC{v}$ contains a clause
\begin{align}
\bigwedge_{i=1}^m A_i' \rightarrow \bigvee_{i=m+1}^{m+n} L_i'   & \quad \begin{array}{l@{\;}l}
\text{with} & A_i = A_i' \sigma_t \text{ for } 1 \leq i \leq m \\
\text{and}  & L_i = L_i' \sigma_t \text{ and } L_i' \in \PRtrig \text{ for } m+1 \leq i \leq m+n.
\end{array} \label{eq:pred:main:nonground}
\end{align}
Now each $A_i$ with ${1 \leq i \leq m}$ is generated by a ground clause of the form \cref{eq:pred:side:ground}, which in turn is obtained from the corresponding non-ground clause \cref{eq:pred:side:nonground}. The \ruleword{Pred} rule is not applicable to \cref{eq:pred:main:nonground} or \cref{eq:pred:side:nonground} so \cref{eq:pred:conclusion:nonground} holds. Since $\Gamma'_i \sigma{t'} = \Gamma_i $ and $\Gamma_i \subseteq \Gamma_t$, \cref{lemma:fragment:operational:redundancy} implies \ref{eq:pred:conclusion:ground}.
\begin{align}
\bigwedge_{i=1}^m \Gamma_i' \rightarrow \bigvee_{i=1}^{m}\Delta_i' \vee \bigvee_{i=m+1}^{m+n} L_i'\sigma & \sbin \SC{u} \quad \text{for } \sigma = \{x \mapsto f(x), \; y \mapsto x \} \label{eq:pred:conclusion:nonground} \\
\bigwedge_{i=1}^m \Gamma_i \rightarrow \bigvee_{i=1}^{m}\Delta_i \vee \bigvee_{i=m+1}^{m+n} L_i          & \sbin N_{t'} \label{eq:pred:conclusion:ground}
\end{align}
Now, \cref{corollary:fragment:properties:monotonicity} applied to \cref{eq:pred:side:ground} implies $R^*_{t'} \not \models \Delta_i$ for each ${1 \leq i \leq m}$; moreover, the definition of $\Delta_t$ ensures $R^*_{t'} \not \models \Forbidden{t}$, so in particular $R^*_{t'} \not \models L_i$ foreach ${m+1 \leq i \leq m+n}$. But then, none of the literals in the head of \cref{eq:pred:conclusion:ground} is satisfied, which contradicts \cref{corollary:fragment:property:satisfaction}. We have reached a contradiction, so we conclude ${\QueryLHS{t} \rightarrow \QueryRHS{t} \nsbin N_{t}}$
\item \Cref{cond:fragment:induction:nocollapse}: we have $\Forbidden{t} = \PRt \backslash R^*_{t'}$, and hence $\Forbidden{t} \subseteq \PRt$. Observe that $\{ t \equals t' \} \in \PRt$ by definition of $\PRtrig$. Furthermore, $ \{t \equals t' \} \not \in R^*_{t'}$ holds since $t$ is irreducible by $\Rsystem{t'}$; consequently, we have $ \{ t \equals t' \} \subseteq \Forbidden{t} $, as required. Similarly, if $t' \in \AllIndiv$, then $\{ t' \equals o\;|\; o \in \AllIndiv \}$ is a subset of $\Omega$, so if we have some $L \in  \{ t' \equals o\;|\; o \in \AllIndiv \}$ with $L \notin \Gamma_o$, then $L \in \Delta_o$; moreover, since $t' \in \AllIndiv$, $\Delta_o = \Delta_t$, which is precisely what we were trying to prove. However, if $t' \notin \AllIndiv$, then an analgous argument to the case $t' \equals t$ applies: now $\{ t' \equals o \;|\; o \in \AllIndiv \} \cap \Gamma_o = \emptyset$, so we need to show $\{ t' \equals o \;|\; o \in \AllIndiv \} \subseteq \Delta_t$. But this follows from the fact that $t'$ is irreducible by $\Rsystem{t'}$, so we have $t' \equals o \in \Delta_t$ by definition of $\Delta_t$. Finally, for $t \equals o$, we must have $t \equals o \in \Delta_t$ since otherwise by the unfolding strategy in \cref{sec:completeness:unfolding-strategy} we would not be building a fragment for $t$. 
\item \Cref{cond:fragment:induction:required}: to show that this condition holds, consider an arbitrary atom ${A_i \in \Required{t}}$, let \cref{eq:pred:side:ground} be the clause that generates $A_i$ in $R^*_{t'}$, and let \cref{eq:pred:side:nonground} be the corresponding non-ground clause. If ${A_i \in \SUt}$, atom $A_i'$ is of the form $A_i''\sigma$, where $\sigma$ is the substitution from the \ruleword{Succ} rule; but then, ${A_i'' \in K_2}$, where $K_2$ is as specified in the \ruleword{Succ} rule. In the unfolding strategy in \cref{sec:completeness:unfolding-strategy}, we chose $X_t=v$ so that the conditions of the \ruleword{Succ} rule are satisfied, and therefore ${A_i'' \rightarrow A_i'' \sbin \SC{v}}$; but then, since ${A_i''\sigma_{t} = A_i}$, we have ${A_i \rightarrow A_i \sbin N_{t}}$, as required for \cref{cond:fragment:induction:required}. If $A_i \notin \SUt$, then $A_i$ must be ground, and we already have $A_i \to A_i \sbin N_t$ by the $r$-\ruleword{Pred} rule with $n=0$ and the generation of atoms of the form $A \to A$ in the root context for every ground atom $A$ in the signature. 
\item \Cref{cond:fragment:induction:individuals} We prove first the double implication for $\Gamma_o$. Suppose that for some $A \in \Omega$ we have $A \in \Gamma_t$; by definition of $\Gamma_t$, this happens if and only if $\Rmodel{t'} \models A$, and since this condition is satisfied for the model fragment $\Rmodel{t'}$, this happens if and only if $A \in \Gamma_o$. Similarly, suppose that for some $A \in Omega$, we have that $A \in \Delta_t$; this happens iff $\Rmodel{t'} \nmodels A$; but since this condition is satisfied for $\Rmodel{t'}$, then we have that $\Rmodel{t'} \nmodels A$ iff $A \in \Delta_o$. This concludes the proof.  
\item \Cref{cond:fragment:induction:consistency} Let $A \in \Gamma_t$ with $A|_p = o$, and $o \equals o' \in \Gamma_o$ with $o >_t o'$. By definition of $\Gamma_t$, we have $\Rmodel{t'} \models A$, and since $o \equals o' \in \Gamma_o$ and \cref{cond:fragment:induction:individuals} is satisfied for $\Rmodel{t'}$, we also have $\Rmodel{t'} \models o \equals o'$. But since $\Rmodel{t'}$ is a congruence, then $\Rmodel{t'} \models A[o']_p$, and therefore $A[o']_p \in \Gamma_t$ as well. 
\item \Cref{cond:fragment:induction:all-lemas} is verified by the order followed in structural induction.
\item \Cref{cond:fragment:induction:deltac} is verified by definition of $\Delta_c$.
\end{itemize}

\end{proof}

\subsubsection{Rewrite termination, confluence, and compatibility}

\begin{lemma}
\label{lemma:rewrite:termination}
The rewrite system $R$ is Church-Rosser.
\end{lemma}

\begin{proof}
We prove that $R$ is Church-Rosser by showing that $R$ is terminating and left-reduced.

\begin{itemize}
	\item In order to show that the system is terminating, we use a total simplification order on all ground \ab-terms and \pred-terms. Let $\gtrdot'$ be an extension on $\gtrdot$ to \pred-terms so that constant $\TOP$ is the smallest element. $\triangleright$ be the lexicographic path order induced by $\gtrdot'$. We have that $\triangleright$ is a simplification order, and therefore, for every $t$, and for any two terms $s_1,s_2$ in the \ab-neighbourhood of $t$, we have that if $s_1 >_t s_2$, then $s_1 \triangleright s_2$. 
In order to show that $R$ is terminating, we show that every rule in $R$ is compatible with $\triangleright$. Let $l \Rightarrow r$ be a rule in $R$. Let $t$ be a term such that $l \Rightarrow r \in R_t$. Let $C^i$ be the clause of $N_t$ such that $L^i = l \equals r$, with $l >_t r$. If $l$ and $r$ are in the \ab-neighbourhood of $t$, we have $l \triangleright r$, since we have already argued that for terms in the \ab-neighbourhood of $t$, $s_1 >_t s_2$ implies $s_1 \triangleright s_2$. If $l$ and $r$ are not in the \ab-neighbourhood of $t$, then $l$ and $r$ must be \pred-terms, and in particular $r$ must be $\TOP$, for no other equalities between \pred-terms are allowed in context clauses. This shows that every rule in $R$ is compatible with a strict, total order $\triangleright$, so the rewrite system must be terminating.
	\item In order to prove that the system is left-reduced, we proceed by contradiction. Assume that there is a rule $l \Rightarrow r$ in $R$ such that $l$ is reducible by $R' = R \backslash \{ l \Rightarrow r \}$. Let $s$ be a term such that $l \Rightarrow r \in R_s$. Let $p$ be the deepest position in $l$ at which $R'$ reduces $l$, so that $l|_p$ is irreducible by $R'$. Let $l' \Rightarrow r'$ be a rule in $R'$ which reduces $l$ at $p$. Let $t$ be a term such that $l' \Rightarrow r' \in R_t$. We have that $t \neq s$, for otherwise we contradict \cref{lemma:fragment:properties:church-rosser}, which guarantees that $R_s$ is Church-Rosser, and therefore left-reduced. 
	
	If $l'$ is a \pred-term, then $l' = l$ and $r' = r$, but then we have $l \Rightarrow r \in R'$, which contradicts our definition of $R'$.
If $l'$ is an \ab-term, then $l'$ is a sub-term of $l$. Observe that $l \Rightarrow r$ can be of one of the following forms: $A \equals \TOP$ with $A$ being ground or containing $s'$ or $s$, or $f(s) \Rightarrow g(s)$, or $f(s) \Rightarrow s$, or $f(s) \Rightarrow o$, or $s \Rightarrow s'$ or $s \Rightarrow o$,, but not $o_1 \Rightarrow o_2$, since this would mean $l' \equals s'$ is of the form $o_3 \Rightarrow o_4$, but by condition \ref{cond:fragment:induction:individuals}, we would have $o_3 \Rightarrow o_4 \in R_s$, and this would contradict \cref{lemma:fragment:properties:church-rosser}, which establishes that $R_s$ must be left-reduced.
 
By contrast, $l' \Rightarrow r'$ can be only of the form $f'(t) \Rightarrow g'(t)$, or $f'(t) \Rightarrow t$, or $f'(t) \Rightarrow t$, or $f'(t) \Rightarrow o$, or $t \Rightarrow t'$ or $t \Rightarrow o$, but not $o_1 \Rightarrow o_2$, since by condition \ref{cond:fragment:induction:individuals}, this would mean $o_1 \Rightarrow o_2 \in R_s$. Thus, $t$ is a subterm of $s', s$, or $f(s)$, and hence in order to define $R_s$, we need to have defined first $R_t$, due to the unfolding strategy described in \cref{sec:completeness:unfolding-strategy}. But if $l' \Rightarrow r'$ is of the form $t \Rightarrow t'$ or $t \Rightarrow o$, then no successors of $t$ are generated, and otherwise, any successor of $t$ is reducible, so $R_{f'(t)}$ is simply $f'(t) \Rightarrow r'$, but this cannot be $l \Rightarrow r$ since then $l \Rightarrow r \in R'$. A contradiction is therefore inescapable, and this completes the proof of the lemma.  
\end{itemize}
\end{proof}

\begin{lemma}
\label{lemma:rewrite:confluence}
For each ground term $t$, each $f \in \Sigma^f$, and each atom $A \in \SUt \cup \textsf{Pr}_{f(t)} \cup \textsf{Ref}_t \cup \textsf{Nom}_t \cup \Omega$
 such that $R^* \models A$ and every \ab-term in $A$ is irreducible by $R$, then $R^*_t \models A$. 
\end{lemma}

\begin{proof}
We consider each possible case:
\begin{itemize}
\item Suppose $A \in \SUt$. $A$ can be of the form $B(t), S(t,t'),S(t',t)$. If $t \in \AllIndiv$, then $A \in \Omega$ and $t'$ does not exist. Observe that if for some $s_0$ we have $\Rmodel{s_0} \models A$, then $A \in \RequiredN$, since otherwise $A \in \ForbiddenN$ and then we could not have $\Rmodel{s_0} \models A$ due to \cref{lemma:fragment:property:compatibility} and  \cref{cond:fragment:induction:consistency}. But then, due to condition \ref{cond:fragment:induction:required}, we have that $\Rmodel{s} \models A$ for any $s$. If $t \notin \AllIndiv$, then it can only occur in $N_t$ or $N_{t'}$. Hence, we have that either $R^*_t \models A$ or $R^*_{t'} \models A$. Now, if $R^*_{t'} \models A$, then $A \in \Gamma_t$, so by \cref{cond:fragment:induction:required} and \cref{lemma:fragment:property:compatibility}, we have $R^*_t \models A$.
\item Suppose $A \in \textsf{Pr}_{f(t)}$. $A$ can be of the form $B(t), S(t,f(t)), S(f(t),t)$. If $t \in \AllIndiv$, by the same argument as above we have $\Rmodel{s} \models A$ for any $s$. If $t \notin \AllIndiv$, then it can only occur in $N_t$ or $N_{f(t)}$. Hence, we have that either $R^*_t \models A$ or $R^*_{f(t)} \models A$. But if  $R^*_{t} \not \models A$, then $R^*_{f(t)} \models A$. However, if  $R^*_{t} \not \models A$, then $A \in \Forbidden{f(t)}$, so by \ref{cond:fragment:induction:compatibility} and \ref{lemma:fragment:property:compatibility}, we have $R^*_{f(t)} \not \models A$, which is a contradiction.
\item Suppose $A \in \textsf{Ref}_t$. If $t \in \AllIndiv$, by the same argument as above we have $\Rmodel{s} \models A$ for any $s$. If $t \notin \AllIndiv$, then it can only occur in $N_t$ and thus $R^*_t \models A$.
\item Suppose $A \in \textsf{Nom}_t$. If $t \in \AllIndiv$, by the same argument as above we have $\Rmodel{s} \models A$ for any $s$. If $t \notin \AllIndiv$, then it can occur only in $N_t$, so we have $R^*_t \models A$. 
\item Suppose $A \in \Omega$ and $R^* \models A$. By \cref{cond:fragment:induction:individuals}, we have $A \in \Gamma_o$, and therefore $\Rmodel{s} \models A$ for any $s$.
\end{itemize}
\end{proof}

\begin{lemma}
\label{lemma:rewrite:compatibility}
Let $s_1$ and $s_2$ and $\tau$ be both DL-\ab-terms or DL-\pred-terms, and let $\tau$ be a substitution where terms are replaced by constants are irreducible by $R$, such that $\tau (x) \neq x$, and such that $s_1 \tau$ and $s_2 \tau$ are ground. Moreover, suppose that for each $z_i$ such that $\tau z_i \neq z_i$ we have that $\tau(z_i)$ is in the \ab-neighbourhood of $\tau (x)$. Then, if $R^*_{\tau(x)} \models s_1 \tau \bowtie s_2 \tau$, we have $R^* \models s_1 \tau \bowtie s_1 \tau \bowtie s_1 \tau$. 
\end{lemma}

\begin{proof}
The case where $\bowtie$ is an equality is trivial, so we consider the case $\bowtie = \nequals$. By \cref{lemma:fragment:properties:irreducibility}, $t$ and $t'$ are irreducible by $R_t$, so they are irreducible by $R$, since they cannot be reduced in any other rewrite system of a successor of $t$, and they are not reduced in a predecessor of $t$ as otherwise $R_t$ would not have been defined. If $s'_2$ or $s'_1$ is of the form $g(t)$, given that $g(t)$ is irreducible by $R_t$, as it is a normal form, and in no other rewrite system we can have $g(t)$ at the left-hand side, since this would have to be in the form of $s=f(t)$ in $R_s$, but $s$ is irreducible by $R_s$. If $s'_2$ or $s'_1$ is of the form $o \in \AllIndiv$, then given that $o$ is irreducible by $R_t$, as it is a normal form, we do not have $\Rmodel{t} \nmodels o \equals o'$ for any $o' < o$, and for the same argument as used above, this means $o \equals o' \in \ForbiddenN$ so in no other rewrite system can we have $o$ at the left-hand side. Thus, $s'_1$ and $s'_2$ are the normal forms of $s_1 \tau$ and $s_2 \tau$ with respect to $R$, so $R^* \models s'_1 \nequals s'_2$, and since $R$ is a congruence this implies $R^*  \models s_1 \tau \nequals s_2 \tau$.
\end{proof}

\subsection{The Completeness claim}

\begin{lemma}
\label{lema:completeness}
For each DL-clause $\Gamma \to \Delta \in \mathcal{O}$, we have $R^* \models \Gamma \rightarrow \Delta$.
\end{lemma}
\begin{proof}
Let $\tau'$ be an arbitrary substitution such that $\Gamma \tau' \rightarrow \Delta \tau'$ is ground, and let $\tau$ be the substitution obtained by replacing the image ground terms in $\tau'$ with their respective normal forms with respect to $R$. Observe that $R^* \models \Gamma \tau' \rightarrow \Delta \tau' $ if and only if $R^* \models \Gamma \tau \rightarrow \Delta \tau $. Thus, suppose $R^* \models \Gamma \tau$, and let us show $R^* \models \Delta \tau$. 
If $\Gamma$ is not empty, consider an arbitrary atom $A_i \in \Gamma$. By definition of DL-clauses, $A_i$ is of the form $B(x), S(x,z_j),S(z_j,x)$. Since image terms in substitution $\tau$ are irreducible by $R$, every $A_i \tau$ is irreducible by $R$. Hence, we have $A_i \tau \rightarrow \TOP \in R$, so it is generated by some generative clause, and therefore, if we let $t = \tau(x)$, we have that $A_i \tau$ is of the form $B(t)$, $S(t,f(t))$, $S(f(t),t)$, $S(t,t')$, $S(t',t)$, $S(t,o)$, $S(o,t)$; or if $t \in \AllIndiv$, then also $S(t,s)$ and $S(s,t)$ for $s \notin \AllIndiv$. We next show that either $A_i \tau \in \SUt \cup \textsf{Pr}_{f(t)} \cup \textsf{Ref}_t \cup \textsf{Nom}_t \cup \Omega$ for every $A_i$, or that  $A_i \tau \in \textsf{Su}_s \cup \textsf{Pr}_{f(s)}  \cup \textsf{Ref}_s \cup \textsf{Nom}_s \cup \Omega$ for every $A_i$. We distinguish two cases: the case where the DL-clause is of the form which can trigger the preconditions of \ruleword{Nom}, and the case where it is not. In the latter, we have the following: 
 \begin{itemize}
\item  $A_i = B(x)$, so $A_i \tau = B(t)$. Then, $B(x) \in \SUtrig$, so $B(t) \in \SUt$.
\item  $A_i = S(x,x)$, so $A_i \tau = S(t,t)$. Then, $S(t,t) \in \textsf{Ref}_t$.
\item  $A_i = S(x,z_j)$, so $A_i \tau = S(t,t')$, $A_i \tau = S(t,f(t))$, or $A_i \tau = S(t,o)$. Since $S(x,y) \in \SUtrig$, in the case $S(t,t')$ we have $S(t,t') \in \SUt$. Also, since $S(y,x) \in \PRtrig$, in the case $S(t,f(t))$ we have $S(t,f(t)) \in \textsf{Pr}_{f(t)}$. In the case $S(t,o)$, we have it in \NOMt. In case $S(t,s)$, with $t \in \AllIndiv$, and $s \notin \AllIndiv$, then $z_j$ is the only neighbour variable in the DL-clause because otherwise it would be a DL-clause of the kind that trigger rule \ruleword{Nom}, so every other $A_i$ is either of the form $S_i(x,z_j)$ or $A_i(x)$, and therefore $A_i \tau$ is of the form $S_i(t,s)$ or $A_i(t)$; observe that $A_i(t) \in \Omega$ because $t \in \AllIndiv$, and $S_i(t,s) \in \textsf{Nom}_s$, which proves the conditions of the lemma.
\item  $A_i = S(z_j,x)$. This case is completely symmetrical to the case discussed above.
\end{itemize}
If the DL-clause is of the form that triggers the preconditions of rule \ruleword{Nom}, we have that if $t \notin \AllIndiv$, or $t \in \AllIndiv$ and every $A_i$ is in $\Omega$, then the same argument as in the previous case applies. If $t \in \AllIndiv$ but $s \notin \AllIndiv$, however, then every $A_i$ is of the form $A_i(x)$ or $S_i(x,z_j)$; and every $A_i \tau$ is of the form $A_i(t)$ or $S_i(t,r)$, where $r$ may not be in $\AllIndiv$. Consider atom $S(t,s)$, and observe that $S(t,s) \in \textsf{Nom}_s$. By \cref{lemma:rewrite:confluence}, we have $\Rmodel{s} \models S(t,s)$, and since this atom is irreducible, it is generated by a generative clause in $N_s$. Consider the non-ground form of the clause $C$ that generates this atom, which must be of the form $C =\Gamma \to \Delta \vee S(o,x)$, with $t=o$ and $\Delta \not \succeq_{X_s} S(o,x)$; observe that by \cref{corollary:fragment:properties:monotonicity} we have $\Rmodel{s} \nmodels \Delta \sigma_s$. Moreover this clause cannot be blocked, since if it were, then we would not have a fragment for $s$: indeed, if $\Gamma' \to \Delta' \vee \Delta''$ is the blocking clause, with $\Delta''$ the part of the head that contains equalities of the form $y \equals o'$, $x \equals o'$, $x \equals y$,  we have $\Gamma' \subseteq \Gamma$, so $\Gamma' \sigma_s \subseteq \Gamma_s$, and the grounding of the blocking clause by $\sigma_s$ would be in $N_s$, and since $\Delta' \subseteq \Delta$, then $\Rmodel{s} \nmodels \Delta' \sigma_t$, so by \cref{corollary:fragment:property:satisfaction}, we would have $\Rmodel{s} \models \Delta'' \sigma_s$ and therefore enforce that $\Rmodel{s} \models s \equals o'$ or $\Rmodel{s}  \models s' \equals o'$, or $\Rmodel{s} \models s \equals s'$, but this contradicts \cref{lemma:fragment:properties:irreducibility}. 

Therefore, since the clause is not blocked, and rule $r$-\ruleword{Succ} is not applicable, we have that $S(o,y) \to S(o,y) \sbin \SC{v_r}$. Since we have that $A_i \to A_i \sbin \SC{v_r}$; and the head cannot be empty since otherwise we violate \cref{lemma:unfolding:noncontradiction}, and rule \ruleword{Nom} is not applicable, we have that clause $C' = \bigwedge{i \in B} A_i \wedge S(o,y) \models \bigvee_{i=1}^K y \equals o_i$ for some $o_i \in \AllIndiv$, where $B$ contains the indices of every unary $A_i$ (and such atoms are in $\Omega$).
  
Moreover, observe that since every unary $A_i$ must be in $\Gamma_o$ and is irreducible, by \cref{cond:fragment:induction:individuals} it must be generated by a generative clause in $N_s$; let $\Gamma_i \to \Delta_i \vee A_i$ be these clauses; observe that $\Rmodel{s} \nmodels \Delta_i$. Consider the non-ground form of these clauses $\Gamma'_i \to \Delta'_i \vee A_i$, which we call $C_i$; recall that $\Delta_i \not \succeq_{X_s} A_i$. We then have that every $C_i$, together with clause $C$ and clause $C'$, satisfy the conditions of rule $r$-\ruleword{Pred}, and we have already argued that $C$ is not blocked, so since $r$-\ruleword{Pred} is not applicable, we have that $$\Gamma \wedge \bigwedge_{i \in B} \Gamma'_i \to \Delta \vee \bigvee_{i \in B} \Delta'_i \vee \bigvee_{i=1}^K x \equals o_i \sbin N_s,$$ and since $\Gamma \sigma_s \subseteq \Gamma_s$ and $\Gamma'_i \subseteq \Gamma_s$ if $i \in B$, by \cref{lemma:fragment:operational:redundancy} we have $$\Gamma \sigma_s \wedge \bigwedge_{i \in B} \Gamma_i \to \Delta \sigma_s \vee \bigvee_{i \in B} \Delta_i \vee \bigvee_{i=1}^K s \equals o_i \sbin N_s.$$ Finally, since we had $\Rmodel{s} \nmodels \Delta \sigma_s$ and $\Rmodel{s} \nmodels \Delta_i$ if $i \in B$, by \cref{corollary:fragment:property:satisfaction} we conclude $\Rmodel{s} \models \bigvee_{i=1}^K s \equals o_i$. However, this contradicts \cref{lemma:fragment:properties:irreducibility}; therefore, we reach a contradiction. Hence, it cannot be the case that $t \in \AllIndiv$ but $s \notin \AllIndiv$, so we must therefore be in any of the cases already discussed, which all verify the result we are trying to prove.

We complete the proof assuming that $A_i \tau \in \SUt \cup \textsf{Pr}_{f(t)} \cup \textsf{Ref}_t \cup \textsf{Nom}_t \cup \Omega$ for every $A_i$, instead of $A_i \tau \in \textsf{Su}_s \cup \textsf{Pr}_{f(s)}  \cup \textsf{Ref}_s \cup \textsf{Nom}_s \cup \Omega$ for every $A_i$, but the argument for the latter case is identical. We have that by \cref{lemma:rewrite:confluence}, $A_i \tau \in R_t$, so $N_t$ contains a generative clause of the form 
$$\Gamma_i \rightarrow \Delta_i \vee A_i \mbox{ with } A_i >_t \Delta_i \mbox{ and } \Gamma_i \subseteq \Gamma_t$$
If $x = X_t$, then $\SC{v}$ contains a clause of the form 
$$\Gamma'_i \rightarrow \Delta'_i \vee A'_i \mbox{ with } \Delta'_i \not \succeq_v A_i' \mbox{ and } \Gamma'_i \Grounding{t} = \Gamma_i, \Delta'_i \Grounding{t} = \Delta_i \mbox{ and } A'_i \Grounding{t} = A_i$$
Since the \ruleword{Hyper} rule is not applicable to the ontology axiom and these clauses (even if there are none), we have that $$\bigwedge_{i=1}^n \Gamma'_i \rightarrow \Delta \sigma \vee \bigvee_{i=1}^n \Delta'_i \sbin \SC{v},$$ with $\sigma$ the substitution where we replace $t$ in the domain by $x$, and by \ref{lemma:fragment:operational:redundancy}, we have $$\bigwedge_{i=1}^n \Gamma_i \rightarrow \Delta \tau \vee \bigvee_{i=1}^n \Delta_i \sbin N_t$$ and by \cref{corollary:fragment:property:satisfaction}, we have $R^*_t \models \Delta \tau \vee \bigvee_{i=1}^n \Delta_i$, but we already had, by \cref{lemma:fragment:property:compatibility}, that $R^*_t \not \models \Delta_i$, which implies $ R^*_t \models \Delta \tau $. Then, by \cref{lemma:rewrite:compatibility}, we conclude that $R \models \Delta \tau$. 
\end{proof}

\begin{lemma}
\label{lemma:absence}
$R^* \not \models \Gamma_Q \rightarrow \Delta_Q$.
\end{lemma}
\begin{proof}
$\Gamma_Q \rightarrow \Delta_Q$ is disproved if there is an element which verifies $\Gamma_Q$ but not $\Delta_Q$. We show that $c$ is this element.
Observe that \cref{lemma:fragment:property:compatibility} implies $R^*_c \not \models \Gamma_Q \rightarrow \Delta_Q$, by definition of $\Gamma_c$ and $\Delta_c$. Thus, by Lemma \cref{lemma:rewrite:compatibility}, if $R^*_c \models \Gamma_Q$, then $R^* \models \Gamma_Q$. 
However, observe also that for each $B(x) \in \Delta_Q$, we have $B(y) \in \PRtrig$, so $B(c) \in \textsf{Pr}(f(c))$, and hence by \cref{lemma:rewrite:confluence}, we have that since $R^*_c \not \models A$, then $R \not \models A$.. Thus, since there are no other types of atoms in $\Delta_Q$, we have $R^* \not \models \Delta_Q$.
\end{proof}
\newpage
\section{Complexity results}

\subsection{Size of extended signature}

For this proof, we assume, for simplicity, that the trivial strategy is being used. Extending it for an expansion strategy introducing at most a finite number of contexts is straightforward.

\subsection{Definitions}
\label{sec:termination:definitions}

We start the proof by introducing a few required definitions:

\begin{definition}[Derivation]\label{def:termination:derivation}
A \emph{derivation} of a clause $C$ from a set of clauses $\SC{}$ is a pair $\nu = (T, \lambda)$ where $T$ is a tree, $\lambda$ is a labelling function that maps each node in $T$ to a context clause and for each $\alpha \in T$ we have:
\begin{enumerate}
\item $\lambda(\alpha) = C$ if $\alpha$ is the root.
\item $\lambda(\alpha) \in \SC{}$ if $\alpha$ is a leaf.
\item if $\alpha$ has children $\beta_1, \dots, \beta_n$, then a rule can be applied to $\lambda(\beta_1), \dots, \lambda(\beta_n)$, and the clause added as a result of this is $\lambda(\alpha)$.
\end{enumerate}
\end{definition}

\begin{definition}[Type]\label{def:termination:type}
A \emph{type} is a conjunction $\Gamma$ where $\Gamma \subseteq \SUtrig$.
\end{definition}

\begin{definition}[Clause blocking]\label{def:termination:definition}
A clause $\Gamma_2 \rightarrow \Delta_2 \vee S(o,x)$ where $\Delta_2 \not \succeq S(o,x)$ is \emph{blocked} by a clause $\Gamma_1 \rightarrow \Delta_1 \vee L_1$, where $\Delta_1 \nsucceq  L_1$ iff $\Gamma_1 \subseteq \Gamma_2$, $\Delta_1 \subseteq \Delta_2$ and $L_1$ is of the form $x \equals y$, $x \equals o'$, or $y \equals o'$. 
\end{definition}

Observe that $\Delta_1 \vee L_1$ contains only equalities and inequalities between \ab-terms, or ground atoms. If, given a clause $C$, there exists a clause $C'$ that blocks $C$, we say that $C$ \emph{is blocked}. 

\begin{definition}[Ground compatibility]\label{def:termination:ground-compatible}
A clause $\Gamma_2 \rightarrow \Delta_2$ is \emph{ground compatible} with a clause $\Gamma_1 \rightarrow \Delta_1$ if and only if every ground atom in $\Gamma_2$ (resp. $\Delta_2$) appears also in $\Gamma_1$ (resp. $\Delta_1$).
\end{definition}

\begin{definition}[Nominal Depth and Root]\label{def:termination:depth}
The depth of an individual $o \in \AllIndiv$, where $\rho$ is the label of $o$, is $|\rho|$.  The root of an individual is the individual $o'$ such that $o= o'_{\rho}$.
\end{definition}

\begin{definition}[$o$-capped clauses]\label{def:termination:capped}
A clause $\Gamma \rightarrow \Delta$ is $o$-capped if and only if the clause does not contain any individual $o'_{\rho \circ \rho'}$, where $o'_\rho = o$.   
\end{definition}
In other words, the clause may not contain a nominal with the same root and an extension of the label. In the remainder of this proof, we fix individuals $o$ and $o'$, and we let $\rho$ and $\rho'$, respectively, to be their labels. Also, we define $o_i$ as $o_{\rho_i}$ where $\rho_i$ is the prefix of $\rho$ of length $i$. We represent $o_i$ as $o_{\rho|_1 \circ \rho|_2 \circ \, \dots \, \circ \rho|_i}$. Given nominals $o_{\rho \cdot S^k}$ and $o_{\rho}$, we say that the latter \emph{precedes} the former.

\subsection{Proof of termination}

The main termination result is a consequence of the following theorem, which relies on \cref{lemma:clausechain}, presented and proved after the theorem. 

From this lemma, we conclude the following:
\begin{theorem}\label{thm:termination}
If $o$ occurs in a derivation, then $|\rho|$ is smaller or equal than the number of possible types in the signature of $\Onto$. 
\end{theorem}

\begin{proof}
Let $Z$ be the number of possible types in the signature of $\Onto$, which is finite since the signature and the number of variables is finite. Suppose that the theorem is false, and consider a derivation which introduces the first clause with a nominal of depth $Z+1$. Then, the last inference step is an application of \ruleword{Nom}, with a premise corresponding to a clause with a nominal $o$ of depth $Z$ in a body atom of the form $S(o,y)$ in the root context. Therefore, there must exist a clause in the derivation with maximal literal $S(o,x)$. Observe that this clause must be $o$-capped as there is no nominal of greater depth than that of $o$. We represent this clause as $C = \Gamma \rightarrow \Delta \vee S(o,x)$ where $\Delta \not \succeq S(o,x)$. If we take the sub-tree corresponding to the derivation of this clause, we have that this clause is not blocked, since if it were blocked, then according to the restrictions of rule $r$-\ruleword{Succ}, we could not use it to generate a clause in the root context with $S(o,y)$ in the body. 

We are in the conditions of \cref{lemma:clausechain}, so let $C_1,\dots,C_{Z}$ be the clauses whose existence is guaranteed by the lemma. Given any $1 \leq i \leq Z$, consider the conjunction $\Gamma'_i$ corresponding to the subset of $\Gamma_i$ with no constants. We have that $\Gamma'_i$ must be a type. Moreover, for each $1 \leq i < j \leq Z$, we have that $\Gamma'_i \not \subseteq \Gamma'_j$. Consider the conjunction $\Gamma'$ obtained from $\Gamma$ in the same manner. Since each $C_i$ is ground-compatible with $C$, we have that either $C_i$ blocks $C$ or $\Gamma'_i \not \subseteq \Gamma'$ for any $i$. But since there are only at most $Z$ different types, by the pigeon-hole principle, there must be some $k \in [1,Z]$ such that $\Gamma'=\Gamma'_k$, and hence $C$ is blocked, which contradicts our initial assumption. Therefore, the theorem must hold. 
\end{proof}

As a corollary from the theorem, we can see that the algorithm is terminating, for it limits the signature of a context structure, and since only a finite amount of clauses can be generated from a finite signature if a finite number of contexts are introduced, the algorithm always terminates. See the next section for more details.

We conclude this section by presenting an proving the auxiliary lemma used in the proof of the theorem.

\begin{lemma}
\label{lemma:clausechain}
Let $o$ be a nominal of depth $n \geq 0$. For any $o$-capped clause of the form $C = \Gamma \rightarrow \Delta \vee S(o,x)$ where $\Delta \not \succeq S(o,x)$, where the derivation of this clause contains only clauses that are also $o$-capped, then either this clause is blocked or there exist $n$ clauses of the form $ C_i =  \Gamma_i \rightarrow \Delta_i  \vee x \equals o_i$ where, for each $1 \leq i \leq n$, we have that:
\begin{enumerateConditions}
\item \label{cond:clausechain:maximal} $\Delta_i \nsucceq x \equals o_i$. 
\item \label{cond:clausechain:capped}$C_i$ is $o_i$-capped.
\item \label{cond:clausechain:gcomp-1}$C_i$ is ground compatible with $C$.
\item \label{cond:clausechain:gcomp-2}$C_i$ is ground compatible with $C_j$ for each $ i < j \leq n-1$, 
\item \label{cond:clausechain:type}$\Gamma'_i \not \subseteq \Gamma'_j$ for each $ i < j \leq n-1$,   
\end{enumerateConditions}
where, for each $i$, let  $\Gamma'_i$ be the part of $\Gamma_i$ that is not ground.
\end{lemma}

\begin{proof}
We proceed by induction on the depth of the nominal. Observe that for the case $n=0$, the result is trivially true, so we move on to the induction step.

Let $o$ be a nominal of depth $n > 0$; let $\rho$ be its label, and suppose we have an $o$-capped clause of the form $C = \Gamma \rightarrow \Delta \vee S(o,x)$ where $\Delta \not \succeq S(o,x)$ which is not blocked, in a derivation where every clause is $o$-capped. We consider the sub-derivation $T_C$ corresponding to generation of this clause. Such derivation $T_C$ must have a path $\mu$ from a subderivation of a clause that introduces $o$ for the first time in the trivial context, such that $o$ occurs in the head of every clause in this path. Let $C^1 = \Gamma^1 \to \Delta^1 \vee S_1(o',x)$ with $\Delta^1 \not \succeq S_1(o',x)$ be the premise of the $r$-\ruleword{Pred} rule that introduces $o$ in $\SC{v}$ for the first time. Observe that $C^1$ is not blocked, and that it is $o'$-capped, with $o'$ the nominal preceding $o$. 
For simplicity, we assume that there is only one premise to this rule in $\SC{v}$. This is without loss of generality, since if there were more than one, they would have the same form and properties, with possibly a different role $S$ in the maximal atom (because they are premises in this application of $r$-\ruleword{Pred}), and the argument relies simply on premises having this property. Observe that we are in the conditions of the lemma for a nominal of smaller depth, so by induction hypothesis, ther exist clauses $C_1,\dots,C_{n-1}$ with the conditions of the lemma. 

Observe now that we have that clause $C^2= \Gamma^1 \to \Delta^1 \vee \Delta^2$ is generated, with $\Delta^2 \in \rootPRtrig \sigma$, where $\sigma$ maps $y$ to $x$, and there is an atom of the form $x \equals o$ in $\Delta^2$, although it is not maximal. We then have that in advancing along $\mu$, the literal is conserved, or it is used as the selected literal in some inference. By choice of $\mu$, since the literal appears in $S(o,x)$, we must have that at some point the literal is maximal and no simplifcations by \ruleword{Join} can be applied. Let this clause be $C^3 = \Gamma^3 \to \Delta^3 \vee x \equals o$, with $\Delta^3 \not \succeq x \equals o$. We choose this clause as the new $C_n$ in the lemma. We show that all the conditions of the lemma are verified. Observe that \cref{cond:clausechain:maximal,cond:clausechain:capped,cond:clausechain:gcomp-1} are satisfied by induction hypothesis for $1 \leq i \leq n-1$, so we prove them for $i=n$.

\begin{itemize}

\item \Cref{cond:clausechain:maximal} is directly verified by choice of $C_n$.
\item \Cref{cond:clausechain:capped} is verified because by induction hypothesis, every clause in the derivation of $C$ is $o$-capped.
\item For \cref{cond:clausechain:gcomp-1}, consider the fragment of $\mu$ from $C^3$ to $C$. If $A \in \Gamma^3$, we need to prove that $A \in \Gamma$. The only reason why this would not be the case, is if $A$ participates in some instance of the \ruleword{Join} rule; however, if this were to be the case, then the other premise used in this application of the \ruleword{Join} rule has a head which is entirely ground, it could have been applied to $C^3$, which contradicts our choice of $C_n$. Furthermore, we have no $A \in \Delta$, due to the ordering of the literals. Therefore, $C^3$ is ground-compatible with $C$.
\end{itemize}

We already have that \cref{cond:clausechain:gcomp-2,cond:clausechain:type} for any $1 \leq i < j \leq n-1$ are verified by induction hypothesis. Thus, let $1 \leq i \leq n-1$, and let us show that $C_i$ is ground compatible with $C$, and $\Gamma'_i \not \subseteq \Gamma'$.

\begin{itemize} 
\item For \cref{cond:clausechain:gcomp-2}, observe that $C^i$ is ground-compatible with $C^{i-1}$, and every ground atom in $C^{i-1}$ is in $C^3$ by construction of $C^3$, therefore $C^i$ is also ground-compatible with $C$.
\item For \cref{cond:clausechain:type}, we proceed using proof by contradiction. Suppose that $\Gamma'_i \to \Gamma'$. 
But then, since $C^i$ is ground-compatible with $C$, and the head of $C^i$ has $x \equals o_i$ as a maximal literal, $C^i$ blocks $C$, which contradicts the assumption that $C$ is not blocked. 
\end{itemize}
This concludes the proof of the lemma.
\end{proof}

Observe that due to \cref{thm:termination}, the signature of the context structure is double exponential in the signature of the original ontology: the length of each $\rho$ is at most equal to the number of types. Since the variables in \SUtrig are only $y$ and $x$, we have that the number of elements in \SUtrig is linear in the size in the ontology, and therefore the number of types is exponential on the size of $\Onto$. Moreover, observe that in each position of the label, we can have any role and any number featuring in a number restriction or a function symbol i.e. we have a quadratic number of possibilities. Thus, we may have a number of different labels which is doubly exponential in the size of $\Onto$. 

Let $k$ be the number of DL-clauses on $\Onto$, and let $m$ the be the larger
of the maximum number of body atoms of a DL-clause in $\Onto$ and the maximum
size of the body of a context clause; $k$ is linear in $\Onto$
since the number of variables in a context clause is fixed, while $m$ can be cubic
in the size of the extended signature. The number $\wp$ of
context clauses that can be constructed using the symbols in $\Onto$ is at most
exponential in the size of these numbers, so it is triple exponential in the size of $\Onto$.

Observe that if the number of contexts is finite, the total number of
context clauses is finite as well. Moreover, once an inference is applied, its
preconditions never become satisfied again. Hence, the number of possible
inferences is finite and each inference is performed just once, so
the algorithm given in Steps 1--3 terminates.

 \newpage
\section{Complexity for $\mathcal{ALCHOIQ}$}
\label{sec:complexity:ALCHOIQ}

\termination*

\begin{proof}
The proof of termination has been given in the previous section. Now, wssume that the strategy can introduce at most $n$
contexts, where $n$ is exponential in the size of $\Onto$. The number of
possible inferences is bounded as follows.
\begin{itemize}
    \item The number of distinct inferences by the \ruleword{Hyper} rule within
    each context is bounded by ${k \cdot \wp^m}$. Hence, the total number of
    inferences is bounded by ${k \cdot \wp^m \cdot n}$, which is exponential in
    the size of $\Onto$.

    \item The number of clauses participating in each distinct inference by a
    rule other than \ruleword{Pred}, \ruleword{Succ}, $r$-\ruleword{Pred}, or $r$-\ruleword{Succ} is constant, 
    so an
    exponential bound on the number of inferences by these rules can be
    obtained analogously to the \ruleword{Hyper} rule.

    \item The \ruleword{Pred} and $r$-\ruleword{Pred} rules are applied to a pair of contexts, and
    each inference involves one clause from one context and at most $m$ clauses
    from the other context. Hence, the number of distinct inferences is bounded
    by ${\wp \cdot \wp^m \cdot n^2}$, which is triple exponential in the size of
    $\Onto$.

    \item The \ruleword{Succ} and $r$-\ruleword{Succ} rules can be applied to any context. Now consider
    an application of the rule to a context $u$, and let clauses ${\Gamma
    \rightarrow \Delta \vee A}$ and ${\Gamma' \rightarrow \Delta' \vee
    A'\sigma}$ play the roles as specified in \cref{tab:newrules}. The
    preconditions of the \ruleword{Succ} rule can become satisfied either when
    a clause ${\Gamma \rightarrow \Delta \vee A}$ is added to $\S{u}$, or when
    a clause ${\Gamma' \rightarrow \Delta' \vee A'\sigma}$ is added to $\S{u}$
    and thus changes the set $K_2$ or $K_3$. Hence, the rule can become applicable at
    most ${\wp^2 \cdot n}$ times, which is triple exponential in the size of $\Onto$.    
    
    \qedhere
\end{itemize}
\end{proof}

\subsection{Pay-as-you-go behaviour}

\payasyougo*

\subsubsection{The $\mathcal{ALCHIQ}$ Description Logic}

If $\Onto$ is in $\mathcal{ALCHIQ}$, the set $\TrueIndiv$ is empty, and therefore so is the set $\AllIndiv$. This produces the following consequences: 
\begin{itemize}
\item In a context structure $\mathcal{D}$ that contains no elements in $\AllIndiv$, no execution of an inference rule will add an element in $\AllIndiv$. Therefore, if we follow the standard application of the calculus and initialise $\ContextStructure$ as usual, it will not contain elements from $\AllIndiv$, and no such element will be introduced by the calculus in the process leading to saturation. 
\item Rules \ruleword{Join}, $r$-\ruleword{Pred}, $r$-\ruleword{Succ}, and \ruleword{Nom} will not be applied; the other rules in the calculus correspond exactly to those in \citeA{DBLP:conf/kr/BateMGSH16}, which is proven to be \textsf{ExpTime}, and therefore worst-case optimal.
\end{itemize}
Observe that according to the definition of $\mathcal{ALCHIQ}$ ontologies given, no ABoxes are included as part of the ontology. If they were to be included, $\Onto$ could also have DL-clauses of the form DL7 (but not DL8). 
\begin{lemma}
Given a sound context structure $\ContextStructure$ such that every $o \in \AllIndiv$ occurs in the root context or in a head ground $\pred$-term, and there are not edges labelled with some $o \in \AllIndiv$, and if every DL-Clause in $\Onto$ is $\mathcal{ALCHIQ}$ or of the form DL7, then applying a rule from \cref{tab:oldrules,tab:newrules} yields a context structure $\ContextStructure'$ where every $o \in \AllIndiv$ occurs in the root context or in a head ground $\pred$-term.
\end{lemma}
\begin{proof}
Observe that $r$-\ruleword{Succ} will not be triggered, since by the conidtions of the lemma we have that no literal of \rootSUtrig can appear in a context clause. This also means that $r$-\ruleword{Pred} and \ruleword{Nom} cannot be applied. Rules \ruleword{Pred} and \ruleword{Succ} simply copy ground atoms, and \ruleword{Eq} and \ruleword{Factor} cannot introduce a new literal with with some $o \in \AllIndiv$  since there are no \ab-equalities or inequalities involving some $o \in \AllIndiv$ in $\ContextStructure$. Therefore, the only rule that can introduce some atom or literal which violates the condition of the lemma is \ruleword{Hyper}. However, the only literals that \ruleword{Hyper} can introduce are ground literals by some DL-clause of the form DL7, which do not violate the conditions of the lemma.
\end{proof}
Therefore, we can see that, in particular, the rule \ruleword{Nom} will never be triggered, and therefore the extended signature corresponds to the original signature. We therefore avoid the double exponential blow-up of $m$ in the proof given in \cref{sec:complexity:ALCHOIQ}. Now, $m$ is linear in the size of $\Onto$ outside the root context. Therefore the proof shows that the total complexity is exponential.

\subsubsection{The $\mathcal{ALCHOI}$ Description Logic}

If $\Onto$ is in $\mathcal{ALCHOI}$, it is easy to see that the rule \ruleword{Nom} is not triggered because there is no axiom that verifies the preconditions that trigger the rule; indeed, the only axioms in any $\mathcal{ALCHOI}$ ontology that may trigger this rule are those which have equalities between neighbour variables in the head, as only those can result in a literal of the form $y \equals y$ or $y \equals f(o)$ after applying a substitution. However, such DL-clauses are absent from $\mathcal{ALCHOI}$ ontologies. Therefore, the extended signature corresponds to the original signature and $m$ in the proof of \cref{sec:complexity:ALCHOIQ} is only cubic in the size of $\Onto$, as we can have any combination of predicate, individual in first position, and individual in second position. The proof therefore entails that the total complexity is \textsf{ExpTime}, and therefore worst-case optimal.

\subsubsection{The $\mathcal{ALCHOQ}$ Description Logic}

If $\Onto$ is in $\mathcal{ALCHOQ}$, we can show that no axiom of the form $S(o,x)$ can be generated, and therefore $r$-\ruleword{Succ} is never triggered, which means that neither $r$-\ruleword{Pred} nor \ruleword{Nom} will be triggered, and once again we have that the extended signature will correspond to the original signature, and the double exponential blow-up in $m$ in the proof in \cref{sec:complexity:ALCHOIQ} is avoided. Once again, $m$ is cubic in the size of $\Onto$ and the total complexity is \textsf{ExpTime}; thus, it is worst-case optimal.

\begin{lemma}
Let $\ContextStructure$ be a context structure where there are no context \pred-terms of the form $S(x,y),S(x,x),S(o,x),S(f(x),x)$ (not even in cores), and no equalities of the form $f(x) \equals y$ or there are no edges labelled with some $o \in \AllIndiv$. Then, applying a rule from \cref{tab:oldrules,tab:newrules} yields a context structure $\ContextStructure'$ where the same conditions hold.
\end{lemma}
\begin{proof}
Observe that there are no axioms of the form DL6, there is no DL-clause that, in combination with an application of \ruleword{Hyper} would result in the generation of a context term of the form $S(x,y),S(o,x),S(f(x),x)$ if no such term exists beforehand. Trivially, \ruleword{Core}, \ruleword{Ineq}, \ruleword{Fact}, \ruleword{Join}, \ruleword{Factor}, and \ruleword{Elim} cannot add any such context term. Rule \ruleword{Pred} can add terms of the form $S(f(x),x)$, but only if one of the premises has a term of the form $S(x,y)$, which is forbidden explicitly by the lemma. Similarly, rule \ruleword{Succ} can add a term of the form $S(x,y)$, but only if the premise contains $S(f(x),x)$, which is also forbidden. Rule $r$-\ruleword{Succ} cannot be applied because we have no term of the form $S(o,x)$, so $\ContextStructure'$ still had no edges labelled with some $o \in \AllIndiv$. As a result of this condition being held, rules $r$-\ruleword{Pred} and \ruleword{Nom} cannot be applied. Finally, rule \ruleword{Eq} may generate elements of the form $S(o,x)$ using an equality only from \pred-terms $S(o',x)$ or $S(f(x),x)$, which are also forbidden. Finally, it may also generate an element of the form $S(f(x),x)$, but this requires, in addition to an equality $f(x) \equals g(x)$, some other element with the form $S(g(x),x)$, which is also forbidden.
\end{proof}

\subsubsection{The Horn-$\mathcal{ALCHOIQ}$ Description Logic}

If $\Onto$ is in Horn-$\mathcal{ALCHOQ}$, and the eager strategy is used, we prove that the following invariant applies to the context structure:

\begin{lemma}
Let $\ContextStructure$ be a context structure such that every clause in a non-root context is of the form $\top \to A$ with $A$ a context literal, as usual, and every clause in a root context is either of the form $\top \to A$ with $A$ a root context literal, or $\bigwedge_{i=1}^n S_i(o,y) \to A$, for $A$ a root context literal. Then, the application of any rule from \cref{tab:oldrules,tab:newrules} yields a context structure $\ContextStructure'$ where the same conditions hold.
\end{lemma}

\begin{proof}
Applications of rules \ruleword{Core}, \ruleword{Ineq}, and \ruleword{Factor} trivially conserve the property described in the lemma. Moreover, \ruleword{Fact} and \ruleword{Join} cannot be applied because of the form of every context clause. We consider in turn each of the remaining rules:
\begin{itemize}
\item If we apply rule \ruleword{Hyper} to a non-root context, the fact that premises have an empty body means that the added clause will have an empty body. Moreover, since the head of DL-clauses in $\Onto$ has either one or zero literals, and $\Delta_i = \bot$ for each premise, we have that the head of the added clause also has one or zero literals. If we apply rule \ruleword{Hyper} to the root context, an analogous argument applies, although the body may now contain a conjunction of atoms of the form $S(o,y)$, but this still is in accordance to the conditions described in the lemma.
\item If we apply rule \ruleword{Eq} to a non-root context, we have  $\Gamma_1=\Gamma_2=\top$ and $\Delta_1 = \Delta_2 = \bot$, and therefore the derived clause verifies the condition described in the lemma. If we apply rule \ruleword{Eq} to a root context, an analogous argument applies, where now $\Gamma_1$ and $\Gamma_2$ may be conjunctions of atoms of the form $S(o,y)$, and therefore the body of the derived clause may also be a conjunction of atoms of such form; this is in accordance to the property described in the lemma.
\item If we apply rule \ruleword{Pred}, we have that $v$ is not the root context, so we always propagate clauses of the form $\top \to A$ in $v$, with $A$ a single literal, to clauses of the form $\top \to A'$ in $u$, with $A' = A \sigma$. 
\item If we apply rule \ruleword{Succ}, use of the eager strategy implies that $\core{v} = K_1$, and the fact that every clause in $u$ is of the form $\top \to A$ means that $K_2=K_1$, so  $K_2 \backslash \core{v} = \emptyset$, and every clause added to $v$ is of the form $\top \to A$ for $A$ a single literal.
\item If we apply rule $r$-\ruleword{Succ}, we add a single literal of the form $S(o,y) \to S(o,y)$ to the root context, and this verifies the conditions of the lemma.
\item If we apply rule $r$-\ruleword{Pred}, given that every literal in the body of the premise in $v_r$ must be of the form $S(o,y)$ and therefore non-ground, we have that the derived consequence is of the form $\top \to L$ with $L$ a single literal, since the premise in $v_r$ has also a single literal in its head.
\item If we apply rule \ruleword{Nom}, given that the axioms in $\Onto$ that trigger this rule can have at most one function symbol in the head, or two neighbour variables, we have that $K=1$, and therefore every clause added by this rule has only one literal in the head. 
\end{itemize}
\end{proof}

Observe that the number of clauses that can be derived in each context is therefore polynomial on the size of the extended signature. Moreover, any clause added by a \ruleword{Pred} or $r$-\ruleword{Pred} rule will not trigger a further application of \ruleword{Succ} or $r$-\ruleword{Succ} since the context $v$ (or $v_r$, respectively) already prevents the preconditions of the rule from being satisfied. Therefore, the core of each context will always be of the form $B(x)$, or $S(y,x)$, or $B(x) \wedge S(y,x)$. This means that a polynomial number of contexts will be derived. As a consequence of this, the number of different possible types in such context structures is polynomial, rather than exponential; therefore, the blow-up induced by the extended signature is exponential, rather than doubly exponential. The proof given in \cref{sec:complexity:ALCHOIQ} can be repeated now with $m$ polynomial in the size of the extended signature, which is exponential, and therefore, the proof shows that the total complexity is exponential. 

\subsubsection{The $\mathcal{ELHO}$ Description Logic}

The proof for the ontologies that are $\mathcal{ELHO}$ is completely analogous to that for Horn-$\mathcal{ALCHOQ}$, since the former is a fragment of the latter. The main difference is that since there are no number restrictions (or inverse roles), rule \ruleword{Nom} is not triggered, and therefore the size of the extended signature corresponds to the original size of the ontology. We saw in the previous section that in this case the the algorithm has a worst-case complexity which is polynomial on the size of the extended signature, and therefore it is polynomial in the size of $\Onto$.

\end{document}